\documentclass[11pt]{article}
\usepackage[a4paper,margin=1in]{geometry}
\usepackage{amsmath,amssymb,amsthm}
\usepackage{graphicx}
\usepackage{booktabs}
\usepackage{enumitem}
\usepackage{tikz}
\usetikzlibrary{positioning,arrows.meta,calc,shapes.geometric}
\usepackage[hidelinks]{hyperref}
\usepackage[numbers,sort&compress]{natbib}
\usepackage{authblk}
\usepackage{microtype}

\newtheorem{theorem}{Theorem}
\newtheorem{proposition}[theorem]{Proposition}
\newtheorem{corollary}[theorem]{Corollary}
\newtheorem{lemma}[theorem]{Lemma}
\newtheorem{definition}[theorem]{Definition}
\newtheorem{remark}[theorem]{Remark}

\title{Bridging Silicon and the Hippocampus: Algebro-Deterministic Memory "VaCoAl" as a Substrate for Vector-HaSH and TEM}

\author[1]{Hiroyuki Chuma\thanks{Corresponding author: \texttt{chuma@iir.hit-u.ac.jp}}}
\author[2]{Kanji Otsuka}
\author[3]{Yoichi Sato}
\affil[1]{Institute of Innovation Research, Hitotsubashi University, \protect\linebreak Kunitachi, Tokyo 186-8603, Japan (Professor Emeritus)}
\affil[2]{Meisei University, Hino, Tokyo, Japan (Professor Emeritus)}
\affil[3]{Shuhari System, Tokyo, Japan}
\date{\today}

\begin{document}
\maketitle

\begin{center}
\small Companion engineering paper:~\cite{chuma2026vacoal}.
\end{center}

\begin{abstract}
\noindent Vector Hippocampal Scaffolded Heteroassociative Memory (Vector-HaSH) and the Tolman--Eichenbaum Machine (TEM) propose that the hippocampal--entorhinal circuit factorizes content storage from a prestructured grid-cell scaffold and supports compositional memory through ripple-mediated replay. Human intracranial electrophysiology, in parallel, has shown that hippocampal sharp-wave ripples (SWRs) gate episodic recall, that ripple-locked cortical reactivation recapitulates encoding-time patterns, and that multi-hop replay fidelity decays approximately multiplicatively along sequence length. The two literatures have advanced in parallel without a shared algebraic object.

\smallskip
\noindent We show that VaCoAl, an algebro-deterministic hyperdimensional memory architecture built from Galois-field linear-feedback shift registers, supplies that object. Specifically: (i)~deterministic Galois-field diffusion is a substrate-level alternative to Vector-HaSH's random scaffold-to-hippocampus projection that satisfies the same quasi-orthogonality requirement, with matched second-moment statistics, strictly stronger worst-case avalanche behavior, and bit-exact reproducibility; (ii)~the path-integral Confidence Ratio CR2, defined as the product of per-step CR1 values along an $n$-hop chain, is the natural functional form for multi-hop replay-fidelity decay under conditional independence of per-step reactivation events, providing the first algebraically tractable model of empirically reported multiplicative decay; (iii)~the Spike-Timing-Dependent-Plasticity-like path selection observed in VaCoAl follows from architectural demands---similarity preservation, compositional reversibility, and bounded-frontier multi-hop search---that also constrain hippocampal computation. We further argue that two qualitatively distinct VaCoAl operating regimes share architectural commitments respectively with the EC--CA3 direct pathway and the EC--DG--CA3 trisynaptic pathway, motivating an energy--capacity--plasticity reading of why both pathways are conserved across $>$520\,Myr of vertebrate evolution and why dentate-gyrus granule cells are elaborated in primates. 
Finally, we show that four largely independent lines of cellular and synaptic evidence---mossy-fiber detonator transmission, granule-cell sparse binary coding under sublinear dendritic integration, anatomically specified mossy-fiber targeting, and CA3 recurrent winner-take-all dynamics---together support a substrate-level reading in which the DG--CA3 pathway implements biophysical homologues of Galois-field arithmetic with approximate reversibility, strengthening the architectural correspondence beyond analogy. 
We additionally connect the framework to Pearl's three-rung Ladder of Causation~\cite{pearl2018book,pearl2009causality}: reversible $\mathrm{GF}(2)$ binding supplies the surgical-modification algebra required by the do-operator (rung~2), and the conserved two-orthogonalizer architecture---Regime~A anchoring the factual world, Regime~B minting counterfactual worlds on demand---supplies the parallel non-interfering representational substrate that rung~3 counterfactual reasoning provably requires, yielding a Pearl-based evolutionary rationale stronger than the energy--capacity--plasticity argument alone.

\smallskip
\noindent 
The framing proceeds in two tiers. At the first tier, the bridge is offered as architectural-commitment correspondence rather than substrate-level identification: VaCoAl is offered as complementary abstraction to Vector-HaSH's scaffold attractor dynamics, not as a drop-in substitute. At the second tier, the cellular and synaptic evidence reviewed in Section~\ref{sec:biophysical-galois} supports a stronger ``biophysical realization with approximate reversibility'' reading, with two reservations retained: $\mathrm{GF}(2)$ reversibility is bit-exact in silicon but only approximate under biological noise, and the biological correlate of the choice of primitive generator $G(x)$ remains unidentified. We state and prove the formal correspondences (with a complete cyclic-group/Hamming-distribution proof of the deterministic quasi-orthogonality property in an appendix), derive testable predictions for human iEEG replay studies and for the encoding/retrieval dissociation, and intend the paper as a bridge between computational neuroscience, hippocampal electrophysiology, and the hyperdimensional-computing engineering tradition.

\end{abstract}

\tableofcontents
\newpage

\section{Introduction}\label{sec:intro}

Three largely independent research strands have converged, over the past five years, on the same architectural picture of hippocampal--entorhinal memory.

\emph{Normative computational neuroscience} has produced two influential models of how the medial temporal lobe stores compositional, sequence-organized memory: Vector Hippocampal Scaffolded Heteroassociative Memory (Vector-HaSH)~\cite{chandra2025episodic} and the Tolman--Eichenbaum Machine (TEM)~\cite{whittington2020tem}. Both posit that a prestructured low-dimensional scaffold—identified with entorhinal grid-cell modules—is mapped into a high-dimensional hippocampal layer through fixed projections, that content is bound to scaffold states by associative learning, and that this factorization underlies generalization across spatial and non-spatial domains. Vector-HaSH in particular requires the scaffold-to-hippocampus projection to be random and fixed, with associative content learning attached to the resulting near-orthogonal codes.

\emph{Empirical electrophysiology} has, in parallel, demonstrated that hippocampal sharp-wave ripples (SWRs) play a causal role in episodic recall in awake humans. Ripple rate rises before vocalized free recall~\cite{norman2019hippocampal,sakon2022hippocampal}; ripple-locked reactivation in cortex recapitulates encoding-time patterns~\cite{vaz2019coupled,norman2021hippocampal}; multi-hop, non-local replay supports inference and decision~\cite{liu2021experience,schwartenbeck2023generative}; and multi-step replay fidelity decays approximately multiplicatively in sequence length~\cite{reithler2025ripple}. A recent line of work links ripple-mediated reactivation to creative recombination, with human iEEG showing that ripples transiently relax medial-prefrontal schema constraints during creative ideation~\cite{He2025human}.

\emph{Algebraic hyperdimensional computing} has independently produced an engineering substrate that, we will show, instantiates these biological proposals at the level of mathematics rather than analogy. The VaCoAl architecture~\cite{chuma2026vacoal} replaces the random projection of Sparse Distributed Memory (SDM)~\cite{kanerva1988sdm,kanerva2009hyperdim} with deterministic Galois-field diffusion implemented by linear-feedback shift registers (LFSRs); replaces exhaustive similarity search with block-wise algebraic majority voting; and supports reversible compositional Binding/Bundling/Unbinding through finite-field bit operations~\cite{kleyko2023survey,plate1995hrr,eliasmith2013how}. On a Wikidata mentor--student directed acyclic graph (DAG) of approximately 470{,}000 records and more than 25.5 million paths, VaCoAl has been shown to produce a path-integral Confidence Ratio CR2 that decays multiplicatively across reasoning hops, yielding a Spike-Timing-Dependent-Plasticity (STDP)-like path-dependent selection function~\cite{chuma2026vacoal}.

The unfilled gap between these strands is precise: Vector-HaSH and TEM are normative claims about what the hippocampal circuit should compute; the iEEG literature establishes when and how the resulting replay events occur; but the field lacks a substrate-level proof that the same architecture can be realized by deterministic algebra, and lacks a tractable mathematical model of multi-hop replay-fidelity decay. We argue that VaCoAl supplies both.

A second conceptual gap motivates the hippocampus-focused narrative below. Recent normative modeling emphasizes that scaffold addresses can be stabilized through grid-driven dynamics that are logically separable from the plastic weights that encode sensory/conceptual contents~\cite{chandra2025episodic}. That separation cleanly distinguishes \emph{dynamical index recovery}—which presupposes a recurrent entorhinal--hippocampal scaffold channel—from \emph{associative content reconstruction}, whose fidelity typically slides with interference and Frontier-bounded overload. Mammalian anatomy nevertheless conserves both a scaffold-friendly EC--CA3 interface and an EC$\rightarrow$DG$\rightarrow$CA3 orthogonalizing detour~\cite{marr1971simple,treves1994computational,rolls2023brain}. VaCoAl gives an algebraically closed account in which deterministic diffusion (QOD) rides alongside path-integrated confidence scorers (CR1/CR2) that resemble STDP-shaped ranking over multi-hop continuations~\cite{chuma2026vacoal}: one substrate where ``hard'' address-like diffusion coexists with ``soft'' branching-aware selection statistics. We articulate the resulting evolutionary question bluntly—why vertebrates conserved both pathways for hundreds of millions of years when a scaffold-only abstraction can look sufficient~\cite{chandra2025episodic}—and propose that, at the silicon level, VaCoAl's explicit toggling between Rescue ($\mathrm{RR}=1$) and Don't Care ($\mathrm{RR}=0$) collision policies~\cite{chuma2026vacoal}, together with hippocampal Regime~A/B routing under acetylcholine and dopamine~\cite{hassel2016neuromodulation}, sketches a plausible \emph{dual-readout-schedule} hypothesis: a single genome encodes wiring for both orthogonalization strategies because behavioral tasks rarely demand only one.

A third, methodological contribution follows from taking the algebraic substrate seriously as a biological hypothesis rather than as an engineering convenience. Cellular and synaptic findings on mossy-fiber detonator transmission~\cite{Henze2002,Vyleta2016,Chamberland2018}, granule-cell sparse binary coding~\cite{leutgeb2007pattern,Diamantaki2016}, sublinear dendritic integration in granule cells~\cite{Krueppel2011}, anatomically specified mossy-fiber targeting~\cite{Acsady1998,Galimberti2006}, and CA3 winner-take-all dynamics~\cite{Guzman2016,rolls2023brain} were accumulated within explanatory frameworks centered on pattern separation, attractor dynamics, and Hebbian binding. None of those frameworks naturally prompted the question whether the assemblage implements deterministic Galois-field arithmetic over $\mathrm{GF}(2)$. Two parallel traditions of dendritic-computation research are particularly worth naming here: Koch and Poggio's~\cite{KochPoggio1983} classical electrical analysis of dendritic spines, which established coincidence detection as a candidate dendritic primitive, and the recent demonstration by Gidon et al.~\cite{Gidon2020} that human layer 2/3 cortical pyramidal neurons exhibit dendritic action potentials implementing XOR-like input--output transformations. Crucially, the layer 2/3 pyramidal cells Gidon et al.\ studied are the same cell class that, in entorhinal cortex, provides the direct perforant-path and temporoammonic input to hippocampal DG, CA3, and CA1, and Benavides-Piccione et al.~\cite{BenavidesPiccione2020} have shown that human hippocampal CA1 pyramidal neurons share the thick-dendritic, electrically compact morphology that supports dendritic-spike-mediated nonlinear integration. Gidon et al.'s XOR-like dendritic computation is therefore not a cortical phenomenon disjoint from the hippocampal substrate; it is a property of the pyramidal cell family that includes EC L2/3 input cells and hippocampal CA1/CA3 pyramidal cells. Both observations are entirely consistent with our reading but were developed without converging on the Galois-field framework. We argue that the integrative question is generated by VaCoAl's specific algebraic commitments—deterministic diffusion, block-wise voting, reversible XOR-and-shift binding—which prior frameworks centered on circular convolution~\cite{plate1995hrr}, random projection~\cite{kanerva1988sdm}, or learned readouts~\cite{whittington2020tem} did not produce. The substantive contribution is thus not the discovery of XOR-like operations in neural tissue, which Koch \& Poggio and Gidon et al.\ have independently established, but the proposal that DG--CA3 implements the specific algebraic structure VaCoAl exposes.

\subsection{Contributions}\label{sec:contributions}

This paper makes six intersecting contributions, all framed as architectural-commitment correspondences rather than substrate-level identifications; the fifth contribution further argues that the cellular and synaptic substrate is consistent with a biophysical-realization reading, and the sixth connects the framework to Pearl's Ladder of Causation. 

\begin{enumerate}[leftmargin=2em,itemsep=0.5em]
\item \textbf{Galois-field diffusion is a deterministic realization of Vector-HaSH's random scaffold projection (Section~\ref{sec:correspondence-I}).} Under standard primitivity assumptions on the LFSR generator polynomial, deterministic Galois-field diffusion satisfies the Johnson--Lindenstrauss-style quasi-orthogonality requirement of Vector-HaSH with matched second-moment statistics relative to a random sparse binary projection, strictly stronger worst-case avalanche behavior at minimum-weight perturbations, and full reproducibility. The ``random'' in random projection is therefore not essential to the Vector-HaSH proposal: what matters is diffusive quasi-orthogonality, and finite-field algebra delivers it deterministically.

\item \textbf{The path-integral Confidence Ratio CR2 is the correct functional form for multi-hop SWR replay fidelity (Section~\ref{sec:correspondence-II}).} Under conditional independence of per-step reactivation events—a standard first-order assumption when ripples are well separated—the probability of successful $n$-hop replay equals $\prod_{i=1}^{n} p_i$, the same form as VaCoAl's $\mathrm{CR2}(n)$. We propose CR2 as the first algebraically tractable model of empirically reported replay-fidelity decay~\cite{sakon2022hippocampal,reithler2025ripple} and derive a testable prediction for iEEG multi-hop replay studies.

\item \textbf{STDP-like path selection follows from architectural demands (Section~\ref{sec:synthesis}).} The combination of (1) and (2), together with a bounded Frontier Size as an architectural commitment, gives rise to the ``strengthen the near, weaken the far'' selection rule. Biological STDP at the cellular level~\cite{bi1998synaptic} and asymmetric place-field expansion at the systems level~\cite{mehta2000experience} can be read as one realization of the same architectural constraint that VaCoAl exposes algebraically; the present account makes the architectural constraint precise but does not claim biology implements VaCoAl operations at the substrate level.

\item \textbf{Why two orthogonalizers? An energy--capacity--plasticity reading of the trisynaptic-plus-direct architecture (Section~\ref{sec:two-orthogonalizers}).} Vector-HaSH establishes that an EC$\leftrightarrow$CA3 loop alone suffices for episodic memory in abstraction, raising the question why evolution conserves a second EC$\rightarrow$DG$\rightarrow$CA3 orthogonalizer alongside it. We adopt the layered vocabulary: \emph{Regime~A} is the EC$\leftrightarrow$CA3 scaffold-vector channel; \emph{Regime~B} is the EC$\rightarrow$DG$\rightarrow$CA3 mossy-fiber relational-write channel (vector-index versus ontology-like semantics). In parallel, VaCoAl's collision policy switch Rescue versus Don't Care ($\mathrm{RR}\in\{1,0\}$; Appendix~\ref{app:vector-hash}) is recruited as a complementary readout-regime analogy: exact index repair versus abstaining reads that leave multiplicative CR2 path scores informative under branching. Separately, VaCoAl supplies a formal parable in terms of two LFSR scheduling extremes (Schedules~U/G: unified diffusion versus gated per-block diffusion); those schedules are neither Regime~A/B nor RR flags. We derive a trade-off proposition, relate it to pathway dissociations~\cite{lee2004differential,hassel2016neuromodulation}, and read primate DG elaborations~\cite{seress1992postnatal,seress2007comparative} as sharpening Regime~B's effective relational capacity.

\item \textbf{Biophysical realization of Galois-field arithmetic in DG--CA3 (Section~\ref{sec:biophysical-galois}).} Four largely independent lines of cellular and synaptic evidence---mossy-fiber detonator transmission~\cite{Henze2002,Vyleta2016,Chamberland2018}, granule-cell sparse binary coding under sublinear dendritic integration~\cite{leutgeb2007pattern,Diamantaki2016,Krueppel2011}, anatomically specified mossy-fiber targeting~\cite{Acsady1998,Galimberti2006,Rollenhagen2010}, and CA3 recurrent winner-take-all dynamics~\cite{Guzman2016,Bezaire2016,rolls2023brain}---together support a reading in which the trisynaptic pathway implements biophysical homologues of Galois-field arithmetic. The reading is stronger than architectural analogy but retains two reservations: $\mathrm{GF}(2)$ reversibility is only approximate under biological noise, and the biological correlate of the choice of primitive generator $G(x)$ remains unidentified. This integrative reading is generated by VaCoAl's specific algebraic commitments and would not naturally have followed from prior frameworks centered on random projection, circular convolution, or learned readouts.

\item \textbf{Connection to Pearl's Ladder of Causation: a substrate-level account of why two orthogonalizers are necessary, not merely efficient (Section~\ref{sec:pearl-ladder}).} We argue that VaCoAl's reversible $\mathrm{GF}(2)$ binding supplies the surgical-modification algebra required by Pearl's do-operator at rung~2~\cite{pearl2009causality,pearl2018book}, and that the conserved two-orthogonalizer hippocampal architecture supplies the parallel non-interfering representational substrate that rung~3 counterfactual reasoning provably requires. Regime~A (EC$\leftrightarrow$CA3 stable scaffold) anchors the factual world; Regime~B (EC$\rightarrow$DG$\rightarrow$CA3 on-demand orthogonalization) mints counterfactual worlds; CR2 path traces over the two scaffolds provide the comparison machinery for $P(Y_x \mid X', Y')$ estimation. This yields an evolutionary rationale for the two-orthogonalizer architecture stronger than the energy--capacity--plasticity argument of Section~\ref{sec:two-orthogonalizers}: parallel non-interfering representation of factual and counterfactual worlds is \emph{necessary} for rung~3 capability, not merely efficient under mixed task statistics. Three reservations are retained: VaCoAl's relational-DAG benchmarks are not yet established as causal DAGs in Pearl's structural sense; the GF(2)-binding/do-operator correspondence is structural rather than proven by an equivalence theorem; and the DG-versus-EC dissociation in counterfactual reasoning is testable but not directly established empirically. The integrative insight is that this Pearl-based reading emerges only from taking $\mathrm{GF}(2)$ algebra seriously as a biological hypothesis.
\end{enumerate}

The remainder of this paper is organized as follows. Section~\ref{sec:background} provides self-contained background on VaCoAl, Vector-HaSH/TEM, and the SWR-replay literature. Section~\ref{sec:correspondence-I} states the random-projection correspondence. Section~\ref{sec:correspondence-II} states the replay-fidelity correspondence. Section~\ref{sec:synthesis} synthesizes the two correspondences into a circuit-level mapping onto DG--CA3--CA1--EC, and Section~\ref{sec:biophysical-galois} develops the biophysical-realization reading. 
Section~\ref{sec:two-orthogonalizers} extends the synthesis with the energy--capacity--plasticity reading and connects it to primate-specific specializations. Section~\ref{sec:predictions} derives empirical and engineering predictions. Section~\ref{sec:discussion} discusses limitations and broader implications, and develops the connection to Pearl's Ladder of Causation in Section~\ref{sec:pearl-ladder}. 
Section~\ref{sec:conclusion} concludes. Appendix~\ref{app:formal} states formal definitions and equivalence; Appendix~\ref{app:proof} gives the proof of Theorem~\ref{thm:det-qod}; Appendix~\ref{app:vector-hash} continues Section~\ref{sec:synthesis} computationally—Rescue versus Don't Care, attractor-complete versus graded CR2 readout, effective branching, and Regime~A/B hypotheses.

\section{Background}\label{sec:background}

This paper joins three literatures whose readers may not all share the same background. The neuroscience reader may know less about hyperdimensional computing; the engineering reader may know less about the trisynaptic loop. We therefore review the three at the level of detail needed for the present synthesis. Readers familiar with any subsection may skim it.

\subsection{Hippocampal Anatomy and the Two Pathways}\label{sec:bg-anatomy}

The classical trisynaptic loop, first described by Cajal~\cite{cajal1911histologie}, comprises three sequential synaptic relays. The perforant path conveys input from layer-II cells of entorhinal cortex (EC) to dentate gyrus (DG) granule cells. The mossy-fiber projection from DG granule cells to CA3 pyramidal cells follows. Schaffer collaterals from CA3 to CA1, with subsequent CA1 projection to subiculum and to deep layers of EC, complete the loop back to neocortex.

Three properties of the trisynaptic loop are computationally salient. First, the perforant-path-to-DG projection is divergent: roughly $2\times 10^5$ EC layer-II cells project to approximately $10^6$ DG granule cells in the rat~\cite{seress2007comparative}. Second, granule cells fire sparsely during behavior, with typical activation rates well below 1\%~\cite{leutgeb2007pattern}. Third, the mossy-fiber projection is itself sparse and large-amplitude: each granule cell makes a small number ($\sim 15$ in rodent) of strong synaptic contacts with specific CA3 pyramidal cells, with synaptic dynamics that are unusually effective at driving CA3 spikes~\cite{rolls2023brain}. Together these features have motivated the long-standing computational view of DG as a pattern separator~\cite{marr1971simple,treves1994computational,mchugh2007dentate}.

In parallel with the trisynaptic loop, EC layer-II cells project directly to CA3 pyramidal cells via collateral branches of the perforant path, bypassing DG. EC layer-III cells project directly to CA1 via the temporoammonic pathway. The direct EC--CA3 projection is more diffuse and less amplified than the mossy-fiber input but is fast and always active. As we discuss below, Vector-HaSH primarily models this direct projection plus CA3 recurrent dynamics; the trisynaptic loop is not part of the model.

A robust double dissociation between the two pathways with respect to encoding and retrieval has been established by selective lesion and genetic-knockout studies~\cite{mchugh2007dentate,nakazawa2002requirement,lee2004differential,hunsaker2007dissociations,hunsaker2008evaluating}. The DG path is selectively gated by acetylcholine and dopamine during novelty~\cite{hassel2016neuromodulation,lisman2005hippocampal,duzel2010novelty}. The DG is the principal site of conserved adult neurogenesis~\cite{kempermann2018human,akers2014hippocampal}. Each of these properties points to a specific computational role that the direct pathway does not play (Figure~\ref{fig:two-pathways}); we develop the implications in Section~\ref{sec:two-orthogonalizers}.

\begin{figure}[t]
\centering
\begin{tikzpicture}[node distance=1.6cm and 2.0cm, font=\small]
\node[draw,rounded corners,fill=blue!10,minimum width=2.6cm,minimum height=0.9cm] (EC) {Entorhinal Cortex (EC)};
\node[draw,rounded corners,fill=orange!15,minimum width=2.0cm,minimum height=0.9cm,below right=1.0cm and 0.8cm of EC] (DG) {Dentate Gyrus (DG)};
\node[draw,rounded corners,fill=blue!10,minimum width=1.6cm,minimum height=0.9cm,right=2.4cm of EC] (CA3) {CA3};
\node[draw,rounded corners,fill=blue!10,minimum width=1.6cm,minimum height=0.9cm,right=1.4cm of CA3] (CA1) {CA1};

\draw[{Stealth}-{Stealth},thick,blue!70!black] (EC) -- node[above,font=\scriptsize]{Regime~A} node[below,font=\scriptsize]{(direct)} (CA3);
\draw[-{Stealth},thick,orange!80!black] (EC) -- node[left,font=\scriptsize,xshift=-2pt]{perforant} (DG);
\draw[-{Stealth},thick,orange!80!black] (DG) -- node[right,font=\scriptsize,xshift=2pt]{mossy fiber} (CA3);
\draw[-{Stealth},thick] (CA3) -- node[above,font=\scriptsize]{Schaffer} (CA1);
\draw[-{Stealth},thick,dashed] (CA1.south) .. controls +(0,-1.0) and +(0,-1.2) .. node[below,font=\scriptsize,yshift=-2pt]{return to EC} (EC.south);

\node[draw,rounded corners,fill=yellow!20,below=1.4cm of DG,font=\scriptsize,align=center] (NM) {ACh / DA\\novelty gate};
\draw[-{Stealth},thick,dashed,red!70!black] (NM) -- (DG);

\end{tikzpicture}
\caption{Hippocampal Regime~A (blue) and Regime~B (orange) in the paper's vocabulary: Regime~A is the EC$\leftrightarrow$CA3 scaffold-vector route emphasized by Vector-HaSH; Regime~B is the EC$\rightarrow$DG$\rightarrow$mossy-fiber$\rightarrow$CA3 channel that supports sparse directed relational writes~\cite{kempermann2018human,hassel2016neuromodulation}. Both routes are recruited across encoding and retrieval with state-dependent gain. Do not confuse these with VaCoAl's LFSR scheduling extremes (Schedules U/G; Section~\ref{sec:two-orthogonalizers}): those engineering schedules illustrate an abstract energy--capacity trade-off but are not called Regime~A/B here. Likewise, the silicon flag $\mathrm{RR}\in\{1,0\}$ is a readout-statistics analogue, not a circuit identification.}
\label{fig:two-pathways}
\end{figure}
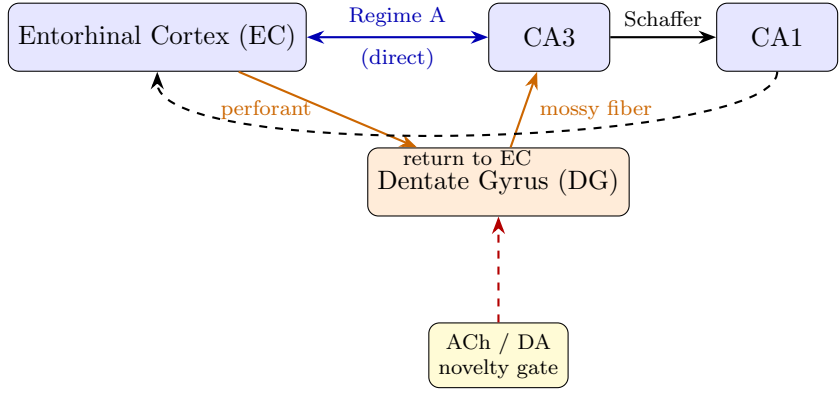

\subsection{VaCoAl in One Page}\label{sec:bg-vacoal}

VaCoAl~\cite{chuma2026vacoal} represents data as binary polynomials over $\mathrm{GF}(2)$. Let $P(x)$ be an input hypervector and let $G(x)$ be a primitive feedback polynomial of degree $m$. The diffusion operation is
\begin{equation}\label{eq:diffusion}
\Psi_{\mathrm{VaCoAl}}\!\bigl(P(x)\bigr) = x^{m}\,P(x) \pmod{G(x)},
\end{equation}
implementable as a single LFSR pass. Equation~\eqref{eq:diffusion} replaces the random projection that standard SDM/HDC architectures use to obtain quasi-orthogonal addresses.

The input hypervector of length $L$ is partitioned into $N$ blocks of length $q = L/N$. Each block applies its own diffusion with a distinct generator polynomial $G_b(x)$ and seed, producing an $m$-bit address into a dedicated SRAM/DRAM array. In the write phase, an Entry Address (EA) is written to the addressed location in every block. In the read phase, each block votes for the EA at its computed address, and the winner is taken by majority:
\begin{equation}\label{eq:vote}
W_{\mathrm{VaCoAl}} = \arg\max_{v} \sum_{i=1}^{N}\delta(v_i,v).
\end{equation}
Erroneous votes generated by the avalanche effect of LFSR diffusion scatter as a flat field across the address space, while correct votes concentrate on a single peak.

Each retrieval produces a per-step Confidence Ratio
\begin{equation}\label{eq:cr1}
\mathrm{CR1} = \frac{\text{votes for winner}}{N},
\end{equation}
and a path-integral Confidence Ratio along an $n$-hop reasoning chain:
\begin{equation}\label{eq:cr2}
\mathrm{CR2}(n) = \prod_{i=1}^{n}\mathrm{CR1}(i).
\end{equation}

Two operating modes are distinguished by a flag we will call $\mathrm{RR}$. In \emph{Rescue mode} ($\mathrm{RR}=1$), residual block collisions are resolved by exact-match lookup against compact auxiliary arrays, pinning $\mathrm{CR1}=\mathrm{CR2}=1.0$ and reproducing a Python \texttt{dict} baseline bitwise on the 25.5M-record genealogy benchmark. In \emph{Don't Care mode} ($\mathrm{RR}=0$), collisions cause affected blocks to abstain from voting; CR1 remains close to but slightly below $1.0$, and CR2 decays multiplicatively along multi-hop reasoning chains. The decay produces an emergent path-dependent selection function whose structure resembles biological STDP: short, direct paths retain high weight, while long, circuitous paths are pruned. Section~\ref{sec:two-orthogonalizers} contrasts idealized \emph{Schedules~U} and \emph{G} for LFSR diffusion (unified-block versus gated per-block re-initialization). Regime~A/B name hippocampal semantics (Figure~\ref{fig:two-pathways}); the silicon RR flag is a readout-statistics analogue; Schedules~U/G describe diffusion-scheduling extremes. These three abstraction layers must be kept distinct.

VaCoAl additionally implements reversible compositional operations. With $\otimes$ denoting Binding (XOR-and-shift over $\mathrm{GF}(2)$, $O(N)$) and $+$ denoting Bundling, a structured representation
\begin{equation}\label{eq:binding}
\mathrm{Repr} = (\mathrm{Color}\otimes\mathrm{Red}) + (\mathrm{Shape}\otimes\mathrm{Apple})
\end{equation}
can be queried by role through Unbinding, recovering constituents exactly in the absence of noise. This algebraic reversibility distinguishes VaCoAl from circular-convolution Holographic Reduced Representations, which require $O(N\log N)$ operations~\cite{plate1995hrr}, and from deep embeddings, which mix constituents irreversibly.

\subsection{Vector-HaSH and the Tolman--Eichenbaum Machine}\label{sec:bg-vectorhash}

Vector-HaSH~\cite{chandra2025episodic} models the entorhinal--hippocampal circuit as a factorization of \emph{scaffold} (entorhinal grid-cell modules with predefined modular connectivity and a low-dimensional velocity-driven shift mechanism) and \emph{content} (sensory, conceptual, or episodic information from neocortex). Crucially, the scaffold projects into the hippocampus through random fixed projections; the inverse projection from hippocampus back to entorhinal grid cells is once-learned and fixed; and content--scaffold bindings are plastic and modified by associative learning. The model exhibits a graceful tradeoff between the number of stored items and recall detail, in contrast to the catastrophic ``memory cliff'' of conventional Hopfield networks~\cite{hopfield1982neural}.

Two clarifications sharpen how we cite that fact. First, avoiding the cliff hinges on \emph{separating} the scaffold (content-independent quasi-orthogonal hash-like indices supported by structured grid coupling) from the associative storage of particulars: interference primarily erodes reconstructed detail rather than wiping the scaffold index manifold wholesale. Second, the attractor dynamics that enact error correction within Vector-HaSH are naturally read as \emph{retrieval-time computation on a recurrent entorhinal--hippocampal channel}; feedforward relays that traverse cortex only once can denoise cues partially but cannot replicate the iterated contraction toward a scaffold fixed point that the minimal published circuit highlights~\cite{chandra2025episodic}. These observations set up our later claim that VaCoAl's CR2-ranked readout occupies a logically distinct niche—grading multi-hop branching continuations—from attractor-complete index cleanup.

The technical core of Vector-HaSH that concerns us is its scaffold-to-hippocampus projection. Let $g\in\mathbb{R}^{N_g}$ be a grid-scaffold state, where $N_g$ is the total number of grid cells across modules. Let $h\in\{0,1\}^{N_h}$ be a hippocampal state. The projection is
\begin{equation}\label{eq:vechash}
h = \sigma\bigl(W_{gh}\,g\bigr),
\end{equation}
where $W_{gh}\in\mathbb{R}^{N_h\times N_g}$ is a sparse binary matrix sampled once at initialization, and $\sigma(\cdot)$ is a thresholding nonlinearity producing sparse binary hippocampal codes. The model's capacity and generalization properties depend on $W_{gh}$ acting as a near-isometry: for distinct grid states $g_1\neq g_2$, the resulting hippocampal codes $h_1, h_2$ should be approximately orthogonal in Hamming distance.

The Tolman--Eichenbaum Machine~\cite{whittington2020tem} arrives at a closely related formalism via a different route: it casts spatial and relational memory as instances of structural abstraction, where structure (graph relations, learned across many tasks) and content (specific stimuli) are factorized in entorhinal and hippocampal codes respectively. In TEM, the binding between structure and content is implemented as a tensor product followed by a learned readout. For the present paper, what matters is that both Vector-HaSH and TEM share the architectural commitment to (a)~prestructured scaffolding, (b)~quasi-orthogonal expansion into the hippocampus, and (c)~factorized binding of content to scaffold states.

\subsection{Sharp-Wave Ripples and Multi-Hop Replay Fidelity}\label{sec:bg-swr}

Hippocampal sharp-wave ripples (SWRs) are transient bursts of synchronized neural activity in the 80--200~Hz band, originating in CA3 and propagating to CA1 through the Schaffer collaterals~\cite{buzsaki2015hippocampal}. In rodents, SWR-locked replay traverses both forward and reverse spatial trajectories~\cite{diba2007forward} and depicts future paths to remembered goals~\cite{pfeiffer2013hippocampal}. Disruption of SWRs impairs spatial learning and consolidation~\cite{girardeau2009selective}.

In humans, intracranial recordings have established several findings directly relevant to multi-hop replay. First, hippocampal ripple rate rises sharply in the seconds before vocalized free recall~\cite{norman2019hippocampal,sakon2022hippocampal}. Second, ripple-locked reactivation in higher visual cortex recapitulates encoding-time activation patterns, with content selectivity preserved across the ripple~\cite{norman2019hippocampal,vaz2019coupled}. Third, ripples coordinate with the default mode network during recent and remote autobiographical recollection~\cite{norman2021hippocampal}. Fourth, replay sequences support relational and compositional inference: Liu et al.~\cite{liu2021experience} demonstrate non-local replay during decision making, and Schwartenbeck et al.~\cite{schwartenbeck2023generative} show that generative replay underlies compositional inference in the hippocampal--prefrontal circuit. Fifth, ripple-mediated reactivation participates in creative ideation by transiently relaxing medial-prefrontal schema constraints~\cite{He2025human}.

Crucially for the present paper, the joint probability of successful multi-step replay decays roughly multiplicatively in sequence length~\cite{reithler2025ripple}. We will return to this empirical fact in Section~\ref{sec:correspondence-II} and propose VaCoAl's CR2 as its formal model.

\subsection{Two Senses of ``Time'' in This Paper}
\label{sec:two-senses-of-time}

Before turning to the formal correspondences, we flag a terminological distinction that the three literatures handle differently and that will 
matter for Proposition~1 and for the database analogy of Section~\ref{sec:two-orthogonalizers}.

We explicitly distinguish the logical generation depth modeled by $\mathrm{CR2}$ 
from the real-time continuous intervals coded by time cells. CA1 time cells 
\cite{Pastalkova2008},\cite{Macdonald2011} tile metric time within a behavioral 
epoch with overlapping receptive fields; their substrate is a continuous 
one-dimensional manifold parameterized by elapsed seconds, and Howard's 
temporal-context framework formalizes this as a Laplace-transform 
representation of recent history. The hop index $n$ of $\mathrm{CR2}(n)$, 
by contrast, counts discrete recall events along a DAG with no commitment 
to the wall-clock interval between successive ripples. Empirically, a single 
hop in a transitive-inference chain may correspond to ripple events 
separated by hundreds of milliseconds or several seconds; 
$\mathrm{CR2}$'s per-hop independence is a statement about event count, 
not duration.

The two coding schemes are complementary rather than competing. Time 
cells provide metric within-episode structure that supports interval 
estimation, sequence ordering within a single trajectory, and behavioral 
timing. DG-mediated event tagging---which we will identify with hippocampal 
Regime~B in Section~\ref{sec:two-orthogonalizers}---provides the 
cross-episode discrete index that $\mathrm{CR2}$ consumes when reasoning 
traverses a combinatorial DAG. The present paper concerns the latter; 
the former is orthogonal to our claims and we do not model it.

\section{Correspondence I: Why Algebro-Deterministic Diffusion Matters}\label{sec:correspondence-I}

This section makes the case that VaCoAl's algebro-deterministic diffusion is not an engineering convenience or a metaphor for biological randomness, but a substantive substrate-level alternative to the random scaffold-to-hippocampus projection of Vector-HaSH. The argument has three parts: (i)~what Vector-HaSH actually requires of its projection; (ii)~why deterministic Galois-field algebra delivers exactly that requirement; (iii)~why this matters for both biology and engineering. We aim to make the case in plain prose. The formal definitions, the precise statement of the equivalence theorem, and its proof are deferred to Appendices~\ref{app:formal} and~\ref{app:proof}.

\paragraph{Vector-HaSH versus VaCoAl error-correction modes (forward reference).} Corollary~\ref{cor:equivalence} below licenses substituting Galois diffusion for $W_{gh}$ whenever a claim depends only on QOD; it does not identify Vector-HaSH with VaCoAl on \emph{readout policy}. Appendix~\ref{app:vector-hash} states, in compressed form—and extends Section~\ref{sec:synthesis} computationally—how Vector-HaSH-style attractor cleanup relates to VaCoAl's Rescue mode, why Don't Care is the regime in which multiplicative CR2 shapes candidate rankings, and how that interacts with DAG branching.

\subsection{What Vector-HaSH Actually Requires}\label{sec:requires}

Vector-HaSH's grid-scaffold-to-hippocampus map is a sparse binary projection $W_{gh}$, sampled at random once at construction time and then held fixed. The model's storage capacity, error correction, and graceful degradation all rest on a single property of $W_{gh}$: that distinct grid-cell states $g_1\neq g_2$ produce hippocampal codes $h_1,h_2$ whose Hamming distance is concentrated near $m/2$, where $m$ is the hippocampal layer's dimension. This is the Hamming-distance analogue of the Johnson--Lindenstrauss isometry property: distinct inputs become near-orthogonal in the high-dimensional output space, allowing many items to be stored without interference.

We will call this the \emph{quasi-orthogonal diffusion property} (QOD); the precise definition is given in Appendix~\ref{app:formal}. The crucial observation, which Vector-HaSH itself notes, is that QOD is what the model needs from $W_{gh}$. The randomness of $W_{gh}$ is not separately important; it is a means by which a generic projection achieves QOD with high probability. If a different mechanism achieves QOD without recourse to randomness, the model's claims still hold.

This is not a hairsplitting distinction. Random projection is, in the biological setting, a hypothesis about biological wiring—that the synaptic connectivity from EC layer-II cells to DG granule cells is sufficiently irregular to approximate random sampling. The hypothesis is plausible but not directly verifiable: the brain's actual wiring is not literally random in a probabilistically meaningful sense, but rather the result of developmental processes that produce irregularity. Whether the resulting irregularity suffices for QOD is, ultimately, an empirical claim about the brain's developmental output. If we can identify a non-random mechanism that achieves QOD provably and exactly, then the question of whether the brain is ``random enough'' becomes secondary: what matters is whether the brain implements \emph{some} mechanism producing QOD, and any mechanism will do.

\subsection{Why Galois-Field Diffusion Delivers QOD Without Randomness}\label{sec:delivery}

VaCoAl's diffusion operation, repeated from Section~\ref{sec:bg-vacoal} for convenience, is
\begin{equation}\label{eq:diffusion-2}
\Psi(P) = x^{m}\,P(x) \pmod{G(x)},
\end{equation}
where $G(x)$ is a primitive polynomial of degree $m$ over $\mathrm{GF}(2)$, and $P(x)$ is the input polynomial. This is a single LFSR pass: a sequence of XOR operations on bits, deterministic and reproducible.

The key fact, established formally in Appendix~\ref{app:formal} as Theorem~\ref{thm:det-qod} and proved in detail in Appendix~\ref{app:proof}, is that $\Psi$ achieves QOD provably—without recourse to randomness, with reproducibility bit-for-bit, and with statistical properties on output Hamming weights that match those of a comparable random sparse projection. Specifically, the average Hamming distance between distinct outputs concentrates near $m/2$ (the same as random projection), the variance scales as $m/4$ (the same as random projection), and the concentration probability around $m/2$ is exponentially tight (the same as random projection). On the second-moment statistics that determine QOD, deterministic Galois-field diffusion and random sparse projection are statistically indistinguishable.

The intuition for why this works is worth stating in prose, even though the formal proof requires the cyclic-group structure of finite fields. Multiplication by $x^m$ in the residue ring $\mathrm{GF}(2)[x]/G(x)$ is, when $G(x)$ is primitive, a permutation that traces a Hamilton path through all $2^m-1$ nonzero residues. Different inputs land at different points along this path, and the points are spread evenly enough across the $\mathrm{GF}(2)^m$ vector space that the average Hamming weight is $m/2$ with the right variance. The avalanche property of LFSRs—a single-bit input flip produces an output with $\sim m/2$ flipped bits—is the geometric expression of this: small perturbations in the input land far apart on the Hamilton path, which means far apart in Hamming distance. Random projection achieves the same diffusion through statistical irregularity; Galois-field algebra achieves it through structural rigour. Different routes to the same destination.

\subsection{Three Practical Consequences}\label{sec:consequences}

The equivalence between random and algebraic diffusion has three consequences that matter for the present paper.

\emph{First: reproducibility.} A random projection is a hypothetical entity. To use it, one must either commit to a specific random instance (and then it is no longer random) or argue that ``almost any'' instance suffices. The first option requires choosing one of the infinitely many possible random matrices, each of which has slightly different empirical behavior on any specific input. The second option is asymptotic and probabilistic; for any finite input ensemble, some random instances are unlucky and deviate noticeably from the ideal QOD level. By contrast, $\Psi$ for a fixed primitive polynomial $G(x)$ is a single, fully specified map. Every input has one, exactly determined output. If two laboratories compute $\Psi$ on the same input, they get the same output bit-for-bit. There is no statistical variance over the construction of $\Psi$ to worry about.

\emph{Second: worst-case behavior on small perturbations.} Random sparse projection with $k$ ones per row produces, for a single-bit input flip, an output Hamming distance bounded above by $k$. If $k\ll m$, this is far from the QOD ideal of $m/2$. Random sparse projection is therefore weakest precisely where QOD is most needed: discrimination of inputs that differ in just one or a few bits. By contrast, the LFSR avalanche property guarantees that single-bit flips at $\Psi$'s input produce outputs whose Hamming weight concentrates near $m/2$, regardless of which input bit is flipped. The deterministic construction is in this sense \emph{stronger} than typical random sparse constructions, even though they share the same average behavior.

\emph{Third: auditability.} For an engineering deployment, one must be able to verify that a memory operation produced the right answer. With random projection, this requires storing the random matrix and replaying it. With $\Psi$, it requires only the seed and the polynomial $G(x)$—a few hundred bits, regardless of the size of the operation. For a biological reading, auditability matters less, but the architectural commitment to determinism does have a biological correlate: the dentate-gyrus mossy-fiber projection, while neither literally random nor literally LFSR-based, has specific structural properties (sparseness, large-amplitude synapses, particular target cell selection) that are themselves determined by genetic and developmental programs, not by stochastic choice. The brain's biology, like our algebra, achieves QOD through structure rather than through randomness in a strong sense.

\subsection{What This Means for the Vector-HaSH Proposal}\label{sec:means-bio}

The implication is that random projection is not the essence of the Vector-HaSH proposal. What matters is QOD-achieving diffusion, and finite-field algebra delivers it deterministically. As a corollary (formally stated as Corollary~\ref{cor:equivalence} in Appendix~\ref{app:formal}): any property of Vector-HaSH that depends on the QOD of its scaffold projection survives the substitution of random projection by Galois-field diffusion. Conversely, any property of VaCoAl that depends on its diffusion's QOD survives the substitution of finite-field algebra by a typical random sparse projection. The two architectures are interchangeable on the property they share—QOD—and differ only on properties beyond it: reproducibility, worst-case avalanche, auditability, and ease of hardware implementation.

This has a substantive consequence for normative neuroscience. The biological hippocampus does not need to be ``random'' in any strong probabilistic sense to instantiate a Vector-HaSH-style scaffold; it only needs to implement \emph{some} diffusive mechanism that achieves QOD. A dentate-gyrus mossy-fiber projection that is structured but high-entropy will perform indistinguishably from one that is genuinely random, provided the diffusion property holds. The architectural-level claim of Vector-HaSH is thus more robust than its random-projection framing might suggest, because the framing is one of multiple possible implementations of the same architectural commitment.

\subsection{What This Means for Engineering}\label{sec:means-eng}

For neuro-symbolic AI and hyperdimensional-computing engineering, the consequence is operational. VaCoAl-style diffusion can be substituted for random projection in standard HDC benchmarks~\cite{kleyko2023survey,raviv2024linear} without functional loss, gaining reproducibility, auditability, and—on hardware that exploits LFSR's $O(1)$ shift cost—energy efficiency. The companion paper~\cite{chuma2026vacoal} develops these engineering implications in detail. For the present paper, the relevant point is that the biological substrate (mossy fibers, granule cells, sparse-dense connectivity) and the engineering substrate (LFSRs, primitive polynomials, $\mathrm{GF}(2)$-XOR) can be read as two implementations of the same architectural commitment.

Readers wanting the formal definition of QOD, the precise statement of the equivalence theorem, the comparison theorem with random sparse projection, and the 9-lemma proof through cyclic-group structure are referred to Appendices~\ref{app:formal} and~\ref{app:proof}. The remainder of the main text proceeds with the prose-level understanding developed above.

\section{Correspondence II: CR2 as a Formal Model of Multi-Hop Replay Fidelity}\label{sec:correspondence-II}

This section formalizes the second bridge: VaCoAl's path-integral Confidence Ratio CR2 is the correct functional form for multi-hop SWR replay fidelity. We state the claim, justify it under standard assumptions on ripple-event independence, and derive a testable empirical prediction.

\subsection{Empirical Decay of Replay Fidelity}

Consider a memory consisting of an $n$-step sequence of states $s_1\to s_2\to\dots\to s_n$ encoded during awake experience. During subsequent ripple-mediated replay, the joint event ``the entire $n$-step sequence is correctly reactivated'' decomposes into per-step reactivation events $E_i$, each of which has some probability $p_i$ of success.

Empirical measurements of replay fidelity in human iEEG~\cite{norman2019hippocampal,sakon2022hippocampal,vaz2019coupled,liu2021experience} use cortical reactivation strength—the cosine similarity between the encoding-time activation pattern for state $s_i$ and the ripple-locked reactivation pattern at the corresponding replay time—as a proxy for $p_i$. The aggregate empirical observation across studies~\cite{reithler2025ripple} is that $P(\text{full sequence})$ decays approximately multiplicatively in $n$.

\subsection{The Multiplicative Form of CR2}

\begin{proposition}[CR2 as replay-fidelity expectation]\label{prop:cr2}
Suppose that per-step replay events $E_1,\dots,E_n$ are conditionally independent given the encoding context, with success probabilities $p_1,\dots,p_n$. Then the probability of full $n$-step sequence replay equals
\begin{equation}\label{eq:joint}
P\!\left(\bigcap_{i=1}^{n} E_i\right) = \prod_{i=1}^{n} p_i.
\end{equation}
If $p_i$ is identified with VaCoAl's per-step Confidence Ratio $\mathrm{CR1}(i)$—interpreted as the fraction of independent block-vote channels that correctly reactivate at hop $i$—then Eq.~\eqref{eq:joint} is exactly $\mathrm{CR2}(n)$ as defined in Eq.~\eqref{eq:cr2}.
\end{proposition}

The conditional-independence assumption requires comment. Within a single ripple, neural activity is highly synchronous and far from independent. However, the relevant decomposition is across ripple events, not within them. In awake recall and during replay across non-adjacent steps in a multi-hop sequence, ripple events are typically separated by hundreds of milliseconds to seconds and are gated by independent thalamo-cortical state changes~\cite{norman2021hippocampal,buzsaki2015hippocampal}. Under these conditions, conditional independence of per-step reactivation success is a reasonable first-order approximation, and is the standard assumption in the iEEG-replay literature when computing joint sequence-replay probabilities.

\begin{remark}[VaCoAl as the deterministic envelope of stochastic replay]\label{rem:envelope}
For biological replay, $p_i$ is a stochastic quantity averaged over many ripple events. For VaCoAl, $\mathrm{CR1}(i)$ is deterministic. Proposition~\ref{prop:cr2} should therefore be read as: VaCoAl's $\mathrm{CR2}(n)$ predicts the \emph{expected} sequence-replay fidelity over many trials. This makes CR2 a falsifiable prediction for averaged iEEG measurements (Section~\ref{sec:predictions}).
\end{remark}

\begin{remark}[Hop count versus wall-clock interval]\label{rem:hop-vs-walltime}
The index $i$ in Proposition~\ref{prop:cr2} enumerates discrete 
recall events; it does not parameterize wall-clock time. Two replay 
sequences of the same hop count $n$ but different inter-ripple 
intervals are predicted to share the same expected $\mathrm{CR2}(n)$, 
holding per-step reactivation probability fixed. Whether the per-step 
probability $p_i$ itself depends on the interval to the previous 
ripple---through Laplace-style temporal context drift, for 
example---is a separate question we do not address here 
(see Section~\ref{sec:discussion}, sixth limitation).
\end{remark}

\subsection{Why STDP-Like Selection Falls Out}

The path-dependent selection observed in VaCoAl—short, direct paths preserved, long, circuitous paths pruned~\cite{chuma2026vacoal}—is now seen to be a direct consequence of Eqs.~\eqref{eq:cr2} and~\eqref{eq:joint}. Because $\mathrm{CR1}(i) < 1$ in Don't Care mode for at least some hops, $\mathrm{CR2}(n)$ decays strictly with $n$, with the decay rate set by the geometric mean of per-hop confidence values. This is the algebraic analogue of biological STDP at the systems level: place-field expansion through repeated experience~\cite{mehta2000experience} and ripple-mediated synaptic potentiation~\cite{bi1998synaptic} both produce a ``strengthen the near, weaken the far'' selection rule, but the rule itself follows from any architecture that integrates per-step confidence multiplicatively along a sequence.

\begin{corollary}[Convergence of biological and algebraic STDP-like selection]\label{cor:stdp}
The path-dependent selection function exhibited by both biological hippocampal replay and by VaCoAl in Don't Care mode is mathematically forced by (i)~compositional path integration, (ii)~per-step confidence below unity, and (iii)~bounded frontier search. Cellular and synaptic mechanisms are one realization of this selection rule; finite-field algebra is another.
\end{corollary}

\section{Synthesis: VaCoAl as a Normative Substrate for the\texorpdfstring{\\}{ }Hippocampal-Entorhinal Circuit}\label{sec:synthesis}

Combining Theorem~\ref{thm:det-qod} and Proposition~\ref{prop:cr2}, we obtain a circuit-level mapping between VaCoAl and the hippocampal--entorhinal architecture. Figure~\ref{fig:correspondence} summarizes the mapping, and Table~\ref{tab:correspondence} states it explicitly.

\paragraph{A complementarity thesis (not amalgamation).} Table~\ref{tab:correspondence} should not be mistaken for a naive identification ``VaCoAl $\equiv$ Vector-HaSH.'' Vector-HaSH foregrounds scaffold attractors whose cleanup dynamics recover discrete indices regardless of associative crowding~\cite{chandra2025episodic}; sensory decoders downstream may remain approximate as load rises. VaCoAl foregrounds Galois-field quasi-orthogonal addressing together with Confidence Ratios whose product contracts along long branching paths—explicit path-quality scores that emerge when abstentions or collisions keep $\mathrm{CR1}<1$ (Section~\ref{sec:bg-vacoal}). Keeping the table's three columns distinct matters because the bidirectional scaffold route and the feedforward trisynaptic detour support different correction semantics in normative accounts: recurrent entorhinal--hippocampal scaffolding provides the dynamical contraction toward grid-consistent indices, whereas EC$\rightarrow$DG$\rightarrow$mossy-fiber CA3 ingress primarily enforces pattern separation and asymmetric relational writes rather than serving, by itself, as a drop-in replacement for that loop dynamics~\cite{chandra2025episodic,marr1971simple,rolls2023brain}. Galois diffusion is our constructive proof that deterministic structure can supply the QOD statistics Vector-HaSH demands of $W_{gh}$ (Theorem~\ref{thm:det-qod}); CR2 supplies a tractable, falsifiable analogue of ripple-gated sequential evidence accumulation (Proposition~\ref{prop:cr2}). Full Vector-HaSH dynamics are richer than silicon majority voting—we argue they are compatible algebraic realizations layered on orthogonalizing pathways that biology already split. About this point, four points deserve emphasis.
\vspace{1\baselineskip}

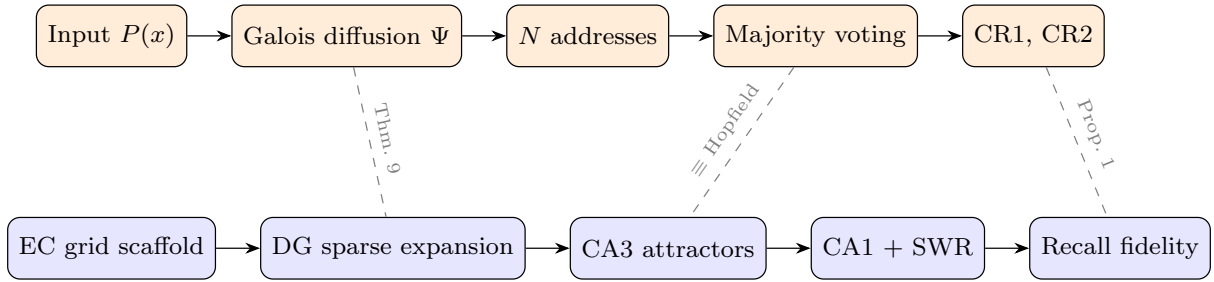
\begin{figure}[h]
\centering
\resizebox{\textwidth}{!}{%
\begin{tikzpicture}[node distance=0.7cm and 0.5cm, font=\scriptsize]
\node[draw,rounded corners,fill=orange!15,minimum width=1.7cm,minimum height=0.7cm] (P) {Input $P(x)$};
\node[draw,rounded corners,fill=orange!15,minimum width=2.0cm,minimum height=0.7cm,right=of P] (D) {Galois diffusion $\Psi$};
\node[draw,rounded corners,fill=orange!15,minimum width=1.7cm,minimum height=0.7cm,right=of D] (B) {$N$ addresses};
\node[draw,rounded corners,fill=orange!15,minimum width=1.8cm,minimum height=0.7cm,right=of B] (V) {Majority voting};
\node[draw,rounded corners,fill=orange!15,minimum width=1.6cm,minimum height=0.7cm,right=of V] (CR) {CR1, CR2};

\draw[-{Stealth}] (P) -- (D);
\draw[-{Stealth}] (D) -- (B);
\draw[-{Stealth}] (B) -- (V);
\draw[-{Stealth}] (V) -- (CR);

\node[draw,rounded corners,fill=blue!10,minimum width=1.7cm,minimum height=0.7cm,below=1.7cm of P] (EC) {EC grid scaffold};
\node[draw,rounded corners,fill=blue!10,minimum width=2.0cm,minimum height=0.7cm,right=of EC] (DG) {DG sparse expansion};
\node[draw,rounded corners,fill=blue!10,minimum width=1.7cm,minimum height=0.7cm,right=of DG] (CA3) {CA3 attractors};
\node[draw,rounded corners,fill=blue!10,minimum width=1.8cm,minimum height=0.7cm,right=of CA3] (CA1) {CA1 + SWR};
\node[draw,rounded corners,fill=blue!10,minimum width=1.6cm,minimum height=0.7cm,right=of CA1] (R) {Recall fidelity};

\draw[-{Stealth}] (EC) -- (DG);
\draw[-{Stealth}] (DG) -- (CA3);
\draw[-{Stealth}] (CA3) -- (CA1);
\draw[-{Stealth}] (CA1) -- (R);

\draw[dashed,gray] (D) -- node[pos=0.5,sloped,above,font=\tiny]{Thm.~\ref{thm:det-qod}} (DG);
\draw[dashed,gray] (V) -- node[pos=0.5,sloped,above,font=\tiny]{$\equiv$ Hopfield} (CA3);
\draw[dashed,gray] (CR) -- node[pos=0.5,sloped,above,font=\tiny]{Prop.~\ref{prop:cr2}} (R);
\end{tikzpicture}%
}
\caption{Architectural correspondence between VaCoAl (top, orange) and the hippocampal--entorhinal circuit (bottom, blue). Galois-field diffusion is a deterministic realization of the random scaffold-to-hippocampus projection (Theorem~\ref{thm:det-qod}); block-wise majority voting is the discrete analogue of CA3 attractor dynamics; the path-integral CR2 is a formal model of multi-hop SWR replay fidelity (Proposition~\ref{prop:cr2}). The dashed lines denote architectural-commitment correspondences, not substrate identities.}
\label{fig:correspondence}
\end{figure}

\begin{table}[t]
\centering
\caption{Circuit-level correspondence between VaCoAl and the hippocampal--entorhinal architecture.} \label{tab:correspondence}
\vspace{1\baselineskip}
\small
\begin{tabular}{p{4.2cm} p{5.4cm} p{4.0cm}}
\toprule
\textbf{VaCoAl Mechanism} & \textbf{Hippocampal Counterpart} & \textbf{Formal Bridge} \\
\midrule
Galois-field diffusion $\Psi$, Eq.~\eqref{eq:diffusion-2} & Random scaffold projection $W_{gh}$ in Vector-HaSH; DG mossy-fiber expansion of EC layer-II input & Theorem~\ref{thm:det-qod}; Cor.~\ref{cor:equivalence} \\
Block partitioning, $N$ independent diffusions & Sparse mossy-fiber connectivity from DG to CA3, with locality of noise & QOD per block, Prop.~\ref{prop:rsbp} \\
Block-wise majority voting, Eq.~\eqref{eq:vote} & CA3 recurrent attractor dynamics; pattern completion via Hebbian-tagged synapses & Both implement noise-tolerant integration of partial cues~\cite{treves1994computational,nakazawa2002requirement} \\
Don't Care abstention ($\mathrm{RR}=0$) & Silenced or weakly-tuned cells excluded from recall pattern & Reduces effective per-step confidence below unity \\
Per-step Confidence Ratio CR1 & Per-ripple reactivation strength in CA1 / cortex & $\mathrm{CR1}(i)\leftrightarrow p_i$ in Prop.~\ref{prop:cr2} \\
Path-integral CR2, Eq.~\eqref{eq:cr2} & Multi-hop sequence replay fidelity; SWR-mediated multi-step recall & Proposition~\ref{prop:cr2} \\
Bounded Frontier Size, FS & Working-memory / DMN search bound during creative ideation & Bounded rationality as architecture~\cite{He2025human} \\
Reversible Binding/Unbinding & Relational compositionality of declarative memory~\cite{cohen1993memory} & Both factorize structure and content; cf.\ TEM~\cite{whittington2020tem} \\
Rescue mode ($\mathrm{RR}=1$, $\mathrm{CR1}=\mathrm{CR2}=1.0$) & Pristine recall of strongly consolidated memory; full ripple coupling & Boundary case $p_i=1$ in Prop.~\ref{prop:cr2} \\
\bottomrule
\end{tabular}
\end{table}

\subsection{Why DG--CA3 Is the Right Anchor}

The DG--CA3 subcircuit has been the focus of computational hippocampal theory since Marr~\cite{marr1971simple} and Treves \& Rolls~\cite{treves1994computational}. The DG performs pattern separation by sparse expansion recoding of EC input, and the CA3 recurrent network performs pattern completion via attractor dynamics. Genetic dissociations have established that NMDA receptors in DG mediate rapid pattern separation~\cite{mchugh2007dentate}, while NMDA receptors in CA3 mediate pattern completion from partial cues~\cite{nakazawa2002requirement}. VaCoAl's diffusion-plus-voting structure mirrors this division of labor: diffusion is pattern separation, and majority voting is pattern completion. Theorem~\ref{thm:det-qod} then says that DG's biological random projection and VaCoAl's algebraic deterministic projection are computationally equivalent for the QOD property that both rely on. We emphasize, however, that majority voting is a discrete analogue of attractor dynamics rather than a literal implementation of continuous-time recurrent contraction; Appendix~\ref{app:vector-hash} returns to this distinction.

\subsection{Why CA1 Readout and SWR Are the Right Anchor for CR2}

CA1 integrates direct EC layer-III input (via the temporoammonic pathway) with CA3 Schaffer-collateral input, and serves as the principal output stage of the hippocampus to subiculum and deep entorhinal layers, and thence to neocortex~\cite{rolls2023brain}. SWRs initiated in CA3 propagate through CA1 to cortex during awake quiescence and slow-wave sleep, where ripple-locked reactivation drives memory consolidation~\cite{buzsaki2015hippocampal,girardeau2009selective}. The empirical decay of multi-hop replay fidelity is measured at this output stage. VaCoAl's CR1/CR2 are likewise readout-stage measurements: they aggregate the votes of the diffusion-and-voting subcircuit and produce the path-quality trace that downstream reasoning consumes.

\subsection{Why Compositionality Connects to TEM}

The Tolman--Eichenbaum Machine~\cite{whittington2020tem} and the relational-memory theory of Cohen \& Eichenbaum~\cite{cohen1993memory} identify compositional reversibility as a defining feature of hippocampal function: relational binding allows constituents to be unbound and recombined for relational inference, episodic future thinking, and constructive imagination~\cite{schwartenbeck2023generative,schacter2007remembering,spens2024generative}. VaCoAl's Binding/Unbinding over $\mathrm{GF}(2)$ is the algebraic substrate of this property. Combined with Theorem~9, this means that VaCoAl provides not only a substrate for the scaffolded-memory aspect of hippocampal function but also for its compositional-relational aspect.

\subsection{Why Continuous Scaffolds Alone Do Not Suffice for Deep DAGs}
\label{sec:why-not-continuous-scaffolds}

The correspondences in 
Sections~3 and 4 motivate VaCoAl as a 
substrate-level alternative for the scaffold projection, the attractor 
readout, and compositional binding. A reader familiar with 
continuous-scaffold accounts of hippocampal coding may nonetheless ask 
why an additional orthogonalizing pathway is needed at all: place cells 
tile space, time cells tile within-episode duration, and grid modules 
supply velocity-driven structure. Why not let the DAG live on this 
continuous manifold?

While continuous grid/time scaffolds excel at metric spaces, representing deep, combinatorial directed acyclic graphs requires the non-commutative, orthogonalizing transformations we map to the 
DG--CA3 pathway. Three properties make the case.

\paragraph{Capacity scaling.} A DAG of depth $n$ and branching factor $b$ admits $b^n$ distinct paths. A place$\times$time scaffold of $N_{\mathrm{cells}}$ place-coding units and duration $T$ provides 
$O(N_{\mathrm{cells}} \cdot T)$ jointly addressable states---polynomial in its physical extent. By contrast, $k$-sparse DG coding over $N_{\mathrm{DG}}$ granule cells admits $\binom{N_{\mathrm{DG}}}{k}$ (= binomial coefficient ${}_{N_{DG}}C_k$) codes that scale exponentially in $N_{\mathrm{DG}}$. Distinguishing all paths in a deep combinatorial DAG demands the latter regime; the former saturates. 

\paragraph{Similarity preservation cuts the wrong way.} Continuous 
scaffolds map nearby inputs to nearby codes by design---this is what 
enables generalization across spatial and temporal neighborhoods. But 
two DAG paths that share intermediate states (for instance 
$A \to B \to C$ and $A' \to B \to C'$) require codes that 
\emph{separate} despite the shared midpoint. DG pattern separation 
supplies precisely this anti-similarity-preserving expansion. A 
continuous scaffold cannot, because shared $B$ forces shared scaffold 
position.

\paragraph{Non-commutativity of role binding.} Continuous-scaffold 
bundling is order-independent: 
$v_{\mathrm{subject}} + v_{\mathrm{verb}} + v_{\mathrm{object}}$ 
cannot distinguish ``Dog bites Man'' from ``Man bites Dog'' 
(see Figure~4). Representing role-asymmetric 
assertions requires non-commutative binding. VaCoAl realizes this 
through XOR-and-shift, where the shift introduces order dependence; 
DG mossy fibers realize it through sparse, directed contacts onto 
specific CA3 pyramidal targets, with postsynaptic identity carrying 
the role tag. Time cells encode within-episode metric position but 
do not carry the role-asymmetry information that compositional 
binding requires; sequence order on a single trajectory is not the 
same as role assignment within a relational tuple.

These three properties together motivate why a hippocampus conserving 
only the EC$\leftrightarrow$CA3 continuous-scaffold loop (Regime~A) 
would face structural difficulty on deep relational tasks even granted 
an ample velocity controller, and why the trisynaptic detour 
(Regime~B) plausibly performs work that continuous scaffolds cannot 
subsume. The argument is at the level of representational commitments, 
not at the level of claiming biology computes Galois XOR: substrate 
independence of the commitments is precisely what 
Corollary~\ref{cor:equivalence} establishes for the QOD layer and 
what we conjecture extends to the binding layer. 
Section~\ref{sec:two-orthogonalizers} builds on this capacity argument 
with an energy--efficiency argument for why both orthogonalizers are 
conserved together.

\subsection{Biophysical Realization of Galois-Field Arithmetic in DG--CA3}\label{sec:biophysical-galois}

The architectural-commitment framing of Sections~\ref{sec:correspondence-I}--\ref{sec:synthesis} deliberately stops short of claiming that the DG--CA3 pathway implements Galois-field arithmetic at the substrate level. This subsection examines how far that reservation can be relaxed in light of cellular and synaptic evidence accumulated over the past two decades. We argue that four largely independent lines of work, when read together, support a stronger reading: the trisynaptic pathway is not merely \emph{compatible with} the algebraic operations VaCoAl performs in silicon, but appears to instantiate biophysical homologues of those operations with quantitative correspondences that are unlikely to be coincidental.

\paragraph{The reframing this subsection performs.}
The cellular and synaptic findings reviewed below were not, individually, motivated by considerations of Galois-field arithmetic. Mossy-fiber detonator transmission was characterized to understand information transfer in CA3~\cite{Henze2002,Vyleta2016}; granule-cell sparse coding was measured to test pattern-separation theories~\cite{leutgeb2007pattern,Diamantaki2016}; mossy-fiber targeting specificity was established as part of hippocampal anatomy~\cite{Acsady1998,Galimberti2006}; CA3 winner-take-all dynamics were modeled to explain attractor-based pattern completion~\cite{Guzman2016,rolls2023brain}. Each finding has been read within its original explanatory frame. What VaCoAl's algebraic architecture supplies, and what prior frameworks centered on circular convolution~\cite{plate1995hrr}, random projection~\cite{kanerva1988sdm}, or learned readouts~\cite{whittington2020tem} did not naturally generate, is a single substrate-level question that unifies these findings: \emph{is the trisynaptic pathway implementing deterministic Galois-field arithmetic over $\mathrm{GF}(2)$?} The integrative reading developed below is a consequence of asking that question; the question itself follows from taking VaCoAl's deterministic algebraic substrate seriously as a candidate biological architecture rather than as an engineering convenience.

Two prior lines of dendritic-computation research deserve named acknowledgment, since both anticipate XOR-like operations in neural tissue from directions independent of the Galois-field framework. Koch and Poggio~\cite{KochPoggio1983}, in a classical theoretical analysis of dendritic spine electrical properties, established coincidence detection as a candidate dendritic primitive and argued that the spine geometry supports nonlinear input combination. More recently, Gidon et al.~\cite{Gidon2020} demonstrated experimentally that human layer 2/3 cortical pyramidal neurons generate dendritic action potentials implementing XOR-like input--output transformations, providing direct \emph{in vitro} evidence that single neurons can compute XOR at the dendritic level. The relevance of Gidon et al.'s finding to the hippocampal architecture is direct rather than analogical: layer 2/3 pyramidal cells are the same cell class that, in medial entorhinal cortex, provides the perforant-path input to DG and the temporoammonic input to CA1~\cite{rolls2023brain}, and Benavides-Piccione et al.~\cite{BenavidesPiccione2020} have shown that human hippocampal CA1 pyramidal neurons share with cortical L2/3 pyramidal cells the thick-dendritic, electrically compact morphology that supports dendritic-spike-mediated nonlinear integration. Dendritic XOR-like computation is therefore a property of the broader pyramidal cell family that includes both the EC L2/3 input cells and the hippocampal CA1/CA3 pyramidal cells. Neither Koch \& Poggio's theoretical analysis nor Gidon et al.'s experimental demonstration converged on the Galois-field reading; both are consistent with it. The reframing in this subsection is therefore not the discovery of XOR-like computation in neural tissue, which these two lines of work have independently established, but the proposal that DG--CA3 implements the specific algebraic structure---deterministic diffusion plus block-wise voting plus reversible binding---that VaCoAl exposes.

\paragraph{Mossy-fiber detonator synapses as deterministic avalanche.}
The avalanche property of Theorem~\ref{thm:det-qod}(e)---that a single-bit input perturbation produces an output whose Hamming weight concentrates near $m/2$---requires that single granule-cell activations reliably reshape downstream CA3 ensembles. The biophysics of mossy-fiber--CA3 transmission supplies exactly this reliability. Henze, Wittner, and Buzs\'aki~\cite{Henze2002} demonstrated \emph{in vivo} that single mossy-fiber EPSPs are sufficient to drive postsynaptic CA3 pyramidal cells to spike threshold, earning these synapses the ``conditional detonator'' designation. Vyleta, Borges-Merjane, and Jonas~\cite{Vyleta2016} and Chamberland et al.~\cite{Chamberland2018} quantified the underlying mechanism: mossy-fiber boutons combine high initial release probability with pronounced short-term facilitation, producing transmission that is near-deterministic on the timescale of a single granule-cell burst. This is the biological counterpart of the deterministic worst-case behavior that distinguishes Galois-field diffusion from random sparse projection (Remark~\ref{rem:comparison}): where a random projection with row-sparsity $k$ admits output Hamming distances bounded above by $k$ for single-bit input flips, mossy-fiber transmission ensures that a single granule-cell activation reliably propagates to CA3, eliminating the analogous failure mode.

\paragraph{Granule-cell sparse binary coding as $\mathrm{GF}(2)$ representation.}
VaCoAl's diffusion operates over $\mathrm{GF}(2)$ with addresses represented as sparse binary vectors. Multiple lines of \emph{in vivo} recording support a similar representational commitment in DG. Leutgeb et al.~\cite{leutgeb2007pattern} measured active granule-cell fractions below 1--2\% during spatial behavior, and Diamantaki et al.~\cite{Diamantaki2016} confirmed extreme sparsity with chronic imaging of identified granule cells. Perni\'a-Andrade and Jonas~\cite{PerniaAndrade2014} further showed that granule-cell firing is locked to theta--gamma oscillations, providing a clocked temporal structure within which discrete binary-like activations can be read out---analogous to the clocked shift operation that drives LFSR diffusion. The XOR-and-shift binding of Eq.~\eqref{eq:binding} requires a substrate that performs sublinear coincidence integration; Krueppel, Remy, and Beck~\cite{Krueppel2011} demonstrated that granule-cell dendrites in fact integrate inputs sublinearly, in contrast to the supralinear integration characteristic of CA1 pyramidal cells. The match between algebraic and biological substrate is therefore not only at the level of sparse binary code, but extends to the dendritic integration regime that makes XOR-like coincidence operations cheap to compute.

\paragraph{Specific mossy-fiber targeting as block partition.}
VaCoAl partitions the hypervector into $N$ blocks, each with an independent generator polynomial $G_b(x)$ driving an independent diffusion (Section~\ref{sec:bg-vacoal}). The classical anatomical work of Acs\'ady et al.~\cite{Acsady1998} established that each granule cell makes large mossy boutons onto a small set of CA3 pyramidal cells and filopodial contacts onto roughly an order of magnitude more interneurons, with postsynaptic identities determined by developmental programs rather than by stochastic targeting~\cite{Galimberti2006}. EM reconstructions by Rollenhagen and L\"ubke~\cite{Rollenhagen2010} and Wilke et al.~\cite{Wilke2013} confirm that individual mossy boutons contain multiple active zones contacting structurally distinct postsynaptic spine heads. This sparse, directed, anatomically specified connectivity is the biological analogue of independent block-wise diffusion in VaCoAl: each granule cell defines a specific small set of CA3 targets, just as each $G_b(x)$ defines a specific small region of the address space. The structural specificity is incompatible with a literal random-projection reading; it is compatible with deterministic block partition.

\paragraph{CA3 recurrent winner-take-all as continuous-time majority voting.}
Equation~\eqref{eq:vote} expresses pattern completion as block-wise majority voting, an argmax over candidate addresses. Guzm\'an et al.~\cite{Guzman2016} directly measured CA3 recurrent connectivity at approximately 1\% connection probability, sufficient to support attractor dynamics under threshold-linear readout, and demonstrated experimentally that the resulting network performs pattern completion. Bezaire et al.~\cite{Bezaire2016} reconstructed the full inhibitory microcircuit in the adjacent CA1 region at scale, illustrating how perisomatic inhibition implements effective winner-take-all dynamics on the timescale of theta cycles; analogous inhibitory motifs operate in CA3. Rolls~\cite{rolls2023brain} synthesizes the resulting picture: CA3 acts as an auto-associative network whose attractor states function as the continuous-time analogue of discrete majority readout. The relationship between Eq.~\eqref{eq:vote} and CA3 dynamics is therefore not metaphorical: the discrete-time argmax that VaCoAl computes is the limiting case of the recurrent contraction that CA3 performs continuously.

\paragraph{Consonance with the consensus computational theory of DG--CA3.}
The four correspondences above align quantitatively with the theory of DG--CA3 developed by Treves and Rolls and elaborated over three decades~\cite{TrevesRolls1992,treves1994computational,KesnerRolls2015,rolls2023brain}. That theory identifies five mechanisms by which the dentate gyrus and mossy-fiber projection achieve sparse, decorrelated representations in CA3: (i) sparse mossy-fiber connectivity ($\sim 46$ inputs per CA3 cell in rodents) that randomizes the set of active CA3 cells per memory~\cite{TrevesRolls1992}; (ii) competitive Hebbian learning at the perforant path--DG synapses; (iii) expansion recoding by the numerous granule cells ($\sim 10^6$ in rat) with strong feedback inhibition; (iv) non-associative plasticity of mossy-fiber synapses that nonlinearly amplifies repeatedly firing inputs~\cite{Henze2002}; and (v) adult neurogenesis providing fresh granule cells with new connections~\cite{kempermann2018human}. Critically, the consensus framework explicitly describes the mossy-fiber effect as ``randomizing,'' but Treves and Rolls~\cite{TrevesRolls1992,rolls2023brain} are careful to add that the actual connectivity is ``specified genetically onto each CA3 neuron'' with some learning reducing the accuracy required: the apparent ``randomness'' is, at the substrate level, structural determinism with learning-mediated tuning, exactly the substrate-level commitment our framework identifies. A further point of consonance: Rolls~\cite{rolls2023brain} explicitly notes that ``some of this architecture may be special to the hippocampus, and not found in the neocortex, because of the importance of storing and retrieving large numbers of (episodic) memories in the hippocampus,'' supporting our reading that DG--CA3 is a specialized algebraic substrate distinct from generic competitive-learning circuits elsewhere in cortex.

\paragraph{Quantitative correspondence between Rolls' Poisson framework and Theorem~\ref{thm:det-qod}.}
To establish the proposed architecture as a rigorous biological model, we must demonstrate that the hippocampal DG--CA3 pathway functions as a biophysical instantiation of Galois-field arithmetic. This requires a formal unification of the statistical measures of biological sparseness with the algebraic requirements of quasi-orthogonal diffusion, linked through the metric of normalized Hamming distance.
\subparagraph{The Biophysical Anchor: Treves--Rolls's Definition of Sparseness}
The fundamental measure of sparseness in the hippocampal circuit is defined by \cite{TrevesRolls1992} as the population sparseness $a$:
\begin{equation}
    a = \frac{\langle r \rangle^2}{\langle r^2 \rangle} 
\end{equation}
where $r$ represents the firing rates across the neuronal population. This ratio extracts the degree of signal overlap; as $a \to 0$, the representation becomes increasingly sparse, a condition necessary to minimize interference between episodic memory patterns stored in the CA3 network \cite{TrevesRolls1992,rolls2023brain}.

The Rolls/Treves-Rolls framework and the present account compute pattern separation through related but distinct statistics, reflecting different code densities. Rolls~\cite{rolls2023brain,TrevesRolls1992} models the number of active mossy-fiber inputs received by a CA3 cell as a Poisson random variable with mean
\begin{equation}\label{eq:rolls-poisson-mean}
\lambda = C_{MF} \cdot a_{DG},
\end{equation}
where $C_{MF} \approx 46$ is the number of mossy-fiber synapses onto each CA3 neuron and $a_{DG} \approx 0.01$ is the DG firing sparseness. Using the population sparseness measure $a = \langle r \rangle^2 / \langle r^2 \rangle$ for a population of firing rates~\cite{rolls2023brain}, and observing that a Poisson random variable with mean $\lambda$ has $\langle r \rangle = \lambda$ and $\langle r^2 \rangle = \lambda(1+\lambda)$, the population sparseness of the mossy-fiber drive to CA3 satisfies the Treves-Rolls identity

\begin{equation}\label{eq:rolls-sparseness}
a_{{CA3}\_{drive}} = \frac{\lambda^2}{\lambda(1+\lambda)} = \frac{x}{1+x}, \qquad x = C_{MF}\cdot a_{DG}.
\end{equation}

For two distinct memories, the expected normalized overlap between the resulting sparse CA3 patterns equals approximately $(a_{{CA3}\_{drive}})^2$, which approaches zero as the representation becomes sparser. This is the biological measure of pattern separation: small overlaps between small active subsets.

\subparagraph{Algebraic Counterpart: Dimensionality-Independent Diffusion}
In our PyVaCoAl/VaCoal architecture, pattern separation is achieved through algebro-deterministic Galois-field diffusion $\Psi$. As detailed in Appendix A, we assess this process using the normalized Hamming distance ($d_H/m$). Dividing the Hamming distance by the vector length $m$ is essential for ensuring a dimensionality-independent assessment of quasi-orthogonality; it allows the definition of a consistent ``QOD level $\epsilon$'' that remains invariant whether the system operates in thousands or millions of dimensions (see Appendix A.2).

Theorem~\ref{thm:det-qod}(c)--(d) gives the analogous decorrelation statistic for VaCoAl's deterministic Galois-field diffusion. For distinct inputs $P_1 \neq P_2$ with difference $\Delta = P_1 \oplus P_2 \notin K$ (recall $K$ is the collision kernel of Lemma~\ref{lem:residue}), the normalized output Hamming distance is concentrated as
\begin{equation}\label{eq:vacoal-distance}
\frac{\mathbb{E}\!\left[d_H(\Psi(P_1), \Psi(P_2))\right]}{m} = \frac{1}{2} + O(2^{-m}), \qquad
\mathrm{Var}\!\left[\frac{d_H(\Psi(P_1), \Psi(P_2))}{m}\right] = \frac{1}{4m} + O\!\left(\frac{2^{-m}}{m}\right).
\end{equation}
This expresses the algebraic measure of pattern separation: fractional Hamming distance near $1/2$ between dense binary codes, equivalent to Johnson--Lindenstrauss near-orthogonality.

\subparagraph{The Bridging Identity: Unification through Mutually Exclusive Events}
The deep connection between biological sparseness and algebraic diffusion is revealed by the following bridging identity. For any two independent binary m--vectors with sparsenesses $a_1$ and $a_2$ (equivalently, two binary codewords with bit-wise activation probabilities $a_1$ and $a_2$), the expected normalized Hamming distance corresponds to the sum of the probabilities of the two mutually exclusive events where the bits differ ($1$ and $0$, or $0$ and $1$):
\begin{equation}\label{eq:vacoal-distance}
    E\left[\frac{d_H}{m}\right] = a_1(1 - a_2) + a_2(1 - a_1) = a_1 + a_2 - 2a_1a_2 
\end{equation}

\begin{itemize}
    \item \textbf{In VaCoAl:} Setting $a_1 = a_2 = 0.5$ (the dense coding limit for optimization in silicon) results in $E[d_H/m] = 0.5$, maximizing information density \cite{rolls2023brain,TrevesRolls1992}.
    \item \textbf{In the Hippocampus:} Substituting the sparse biological values ($a \approx a_{CA3\_drive}$) results in the very low overlaps calculated by \cite{TrevesRolls1992}, which are necessary for high-capacity episodic storage under metabolic ATP constraints ~\cite{attwell2001energy,Lennie2003}.
\end{itemize}

This unification supports a further architectural reading: the choice between $a = 1/2$ (dense codes; VaCoAl, engineering) and $a \ll 1/2$ (sparse codes; biological hippocampus) is a substrate-level optimization under different constraints. Sparse codes minimize metabolic cost per memory at the price of greater susceptibility to drop-out failures; dense codes maximize per-vector information at the price of higher per-operation energy. Biological hippocampus operates under hard ATP constraints~\cite{attwell2001energy,Lennie2003} that favor sparse coding; silicon VaCoAl operates under thermal-density constraints~\cite{kleyko2023survey} that favor dense binary codes with parallel XOR throughput. Both implement the same QOD-statistics principle expressed through Eq.~\eqref{eq:vacoal-distance}, with code density chosen for substrate-specific reasons. The consensus computational theory and the present account are thus two implementations of one mathematical commitment, not two competing readings; the algebraic level makes precise the substrate-level commitments that the cellular level identifies operationally.

\paragraph{What this strengthens, and what remains a reservation.}
Taken together, the four correspondences above support a reading in which the DG--CA3 pathway implements biophysical homologues of Galois-field arithmetic: deterministic single-cell-to-ensemble transmission (Lemma~\ref{lem:avalanche}-like behavior in tissue), sparse binary representation over a clocked basis (the granule-cell coding regime), anatomically specified block partition (mossy-fiber targeting), and contractive readout (CA3 attractor dynamics). The architectural correspondence developed in Sections~\ref{sec:correspondence-I}--\ref{sec:synthesis} thereby gains a substrate-level reading that is stronger than analogy. Two reservations should nonetheless be retained. First, the reversibility of $\mathrm{GF}(2)$ binding is bit-exact in silicon but only approximate in biology, where LTP/LTD is stochastic and recall under interference is graded; this is the same limitation that Plate's HRR~\cite{plate1995hrr} faces relative to its idealized algebraic form. Second, the biological correlate of the choice of primitive polynomial $G(x)$ remains unidentified: mossy-fiber target specificity is developmentally programmed~\cite{Galimberti2006}, but whether this developmental program is functionally equivalent to selecting a primitive generator over $\mathrm{GF}(2)$ is an open question. We therefore characterize the present reading as ``biophysical realization with approximate reversibility,'' stronger than the architectural-commitment framing of Section~\ref{sec:synthesis} but weaker than full substrate identification.

\section{Why Two Orthogonalizers? An Energy--Capacity--Plasticity Reading}\label{sec:two-orthogonalizers}

Vector-HaSH~\cite{chandra2025episodic} establishes that an EC$\leftrightarrow$CA3 loop with random scaffold-to-hippocampus projection alone suffices to support compositional, episodic, and sequential memory in abstraction. Interpreting that result mechanistically underscores an asymmetry: the sufficiency thesis is anchored on \emph{bidirectional} scaffold-mediated dynamics capable of iterated cleanup, whereas the classical EC$\rightarrow$DG$\rightarrow$CA3 relay is overwhelmingly feedforward once mossy terminals contact CA3. Feedforward preprocessing can orthogonalize collisions and reshape CA3 eligibility before recurrence engages, yet it lacks, by itself, the closed-loop relaxation that retracts ambiguous CA3 states to grid-consistent scaffold attractors~\cite{chandra2025episodic,rolls2023brain}. Seen this way, the ``two orthogonalizers'' motif is stronger than redundancy for capacity alone: Regime~A supplies the contraction channel for index-aligned states while Regime~B supplies gated pattern separation tuned to novelty and asymmetric relational tagging~\cite{marr1971simple,kempermann2018human,hassel2016neuromodulation}.

If the EC$\leftrightarrow$CA3 path occupies that privileged dynamical niche, why does every mammalian (and indeed every vertebrate) hippocampus also conserve a second orthogonalizing path EC$\rightarrow$DG$\rightarrow$CA3 via the perforant path and mossy fibers? The trisynaptic loop has been the canonical hippocampal anatomy since Cajal~\cite{cajal1911histologie}, and the dentate gyrus is the unique site of conserved adult neurogenesis~\cite{kempermann2018human}; an architecture this expensive to maintain across $>$520\,Myr of vertebrate evolution~\cite{murray2017evolution} is unlikely to be redundant.

In this section we argue that VaCoAl gives a principled answer while keeping mechanistic distinctions in view~\cite{chandra2025episodic,rolls2023brain}: even if scaffold-centric EC$\leftrightarrow$CA3 dynamics suffice in abstraction, conserved EC$\rightarrow$DG$\rightarrow$CA3 circuitry may perform orthogonalization and relational tagging that attractor relaxation alone does not subsume. The deterministic-vs-random equivalence proved as Corollary~\ref{cor:equivalence} does not say that one orthogonalizer suffices; it says that the QOD property can be implemented in multiple ways.

\paragraph{VaCoAl RR regimes as readout analogue (not circuitry identity).} On the silicon side, flipping PyVaCoAl's rescue flag $\mathrm{RR}=1$ versus $\mathrm{RR}=0$ toggles collision resolution and hence whether CR1 is pinned at unity (Rescue; dict-matched exact reads on the genealogy benchmark~\cite{chuma2026vacoal}) or can fall below unity (Don't Care; multiplicative CR2 summarizes branch quality~\cite{chuma2026vacoal}); Appendix~\ref{app:vector-hash} records the branching bookkeeping. \emph{Anatomy does not deprive an animal's brain of mossy fibers between trials}; rather, state-dependent mixtures across Regime~A versus Regime~B—mediated partly by cholinergic/dopaminergic gain on perforant-to-DG synapses~\cite{lee2004differential,hassel2016neuromodulation}—offer a plausible biological analogue of which readout statistic dominates: consolidation-weighted scaffold closure (conceptually nearer Rescue / attractor-complete index repair~\cite{chandra2025episodic}) versus branching-heavy cognition where leaky per-hop evidence accumulates (conceptually nearer Don't Care; Section~\ref{sec:correspondence-II}). This dual-readout-schedule analogy, together with phylogenetic arguments that hippocampal homology is ancient yet adaptable~\cite{murray2017evolution}, rationalizes conserved dual routing when lifetime task statistics continually revisit both regimes faster than morphology could prune supposedly ``redundant'' trisynaptic hardware~\cite{kempermann2018human}.

The present section first isolates two idealized LFSR schedules (Schedules U and G) as engineering bookends for orthogonalization scheduling in silicon, derives a trade-off proposition that justifies implementing both roles in a single engineered memory, and then re-anchors the biology using hippocampal Regime~A/B vocabulary (Figure~\ref{fig:two-pathways}): Vector-HaSH's EC--CA3 scaffold loop is the shared backbone for encoding and retrieval, with EC$\rightarrow$DG$\rightarrow$CA3 as a second, novelty-gated relational writer—not as an exclusively ``encoding-only'' bundle. We connect the proposition to the empirical encoding/retrieval dissociation (read as state-dependent mixing rather than strict division of labor) and to primate-specific elaboration of dentate-gyrus granule cells.

\subsection{Idealized LFSR Schedules (Schedules U and G)}\label{sec:schedules}

VaCoAl as defined in Section~\ref{sec:bg-vacoal} admits a continuum of parameter choices, of which two extreme schedules are illuminating. We call them \textbf{Schedule~U} (unified) and \textbf{Schedule~G} (gated, per-block).

\paragraph{Schedule~U (unified LFSR; single $G$, single seed).} A single LFSR with one primitive polynomial $G(x)$ and one seed produces all addresses; all $N$ blocks share the same diffusion. The scaffold is fixed at construction time: distinct inputs map to distinct deterministic addresses, and the QOD level $\varepsilon=O(\sqrt{(\log 1/\delta)/m})$ is achieved over the input ensemble (Theorem~\ref{thm:det-qod}). Energy cost per memory operation is minimal: one shift-register pass, $O(N)$ XOR operations. The architectural commitment is that the relationship between input and address never changes; familiar inputs are retrieved through the same fixed map that wrote them. Capacity is bounded by the address-space size $2^m$ and the per-address store size.

\paragraph{Schedule~G (gated multi-LFSR; novelty re-seed).} Each of the $N$ blocks holds an independent primitive polynomial $G_b(x)$, and the per-event seed for each block is sampled from a novelty signal at encoding time. New inputs receive new orthogonal codes that are uncorrelated with codes previously assigned, even if the inputs themselves are highly similar. Energy cost per memory operation is substantially higher: $N$ independent diffusions, plus the bookkeeping to record per-block seeds. The architectural commitment contrasts with Schedule~U: orthogonal coding is generated on demand for novel events, at the cost of additional energy and additional plasticity.

These schedules are not mere parameter tweaks; in an engineered VaCoAl stack they correspond to qualitatively different roles (Figure~\ref{fig:schedules}). Schedule~U is cheap map re-use: it commits durably to a fixed input-to-address map. Schedule~G is expensive fresh tagging: it pays higher energy per event to mint codes that are deliberately decorrelated from older codes even when inputs are nearly identical. Figure~\ref{fig:schedules} labels the panels by silicon scheduling roles only; it does \emph{not} assert a one-to-one map onto hippocampal Regime~A versus Regime~B (cf.\ Section~\ref{sec:mapping-AB}).

\begin{figure}[t]
\centering
\resizebox{\textwidth}{!}{%
\begin{tikzpicture}[font=\scriptsize, node distance=0.4cm]
\node[draw,rounded corners,fill=blue!10,minimum width=4.5cm,minimum height=0.5cm] (Utitle) {\textbf{Sched.~U: unified LFSR (single $G$, single seed)}};
\node[draw,rounded corners,fill=blue!5,minimum width=2.0cm,below=of Utitle] (UI) {Input $P(x)$};
\node[draw,rounded corners,fill=blue!5,minimum width=2.0cm,below=of UI] (UG) {single $G(x)$, fixed seed};
\node[draw,rounded corners,fill=blue!5,minimum width=2.0cm,below=of UG] (UN) {$N$ blocks share $\Psi$};
\node[draw,rounded corners,fill=blue!5,minimum width=2.0cm,below=of UN] (UV) {majority vote};
\node[draw,rounded corners,fill=blue!5,minimum width=2.0cm,below=of UV] (UA) {address (fixed map)};
\node[below=0.1cm of UA] (Ucost) {cost: $E_A$ (low) ~ capacity: $C_A$};

\draw[-{Stealth}] (UI) -- (UG);
\draw[-{Stealth}] (UG) -- (UN);
\draw[-{Stealth}] (UN) -- (UV);
\draw[-{Stealth}] (UV) -- (UA);

\node[draw,rounded corners,fill=orange!15,minimum width=4.5cm,minimum height=0.5cm,right=2.5cm of Utitle] (Gtitle) {\textbf{Sched.~G: gated multi-LFSR ($G_b$, novelty re-seed)}};
\node[draw,rounded corners,fill=orange!10,minimum width=2.0cm,below=of Gtitle] (GI) {Input $P(x)$};
\node[draw,rounded corners,fill=orange!10,minimum width=2.4cm,below=of GI] (GG) {$N$ polys $G_b(x)$, novelty re-seed};
\node[draw,rounded corners,fill=orange!10,minimum width=2.4cm,below=of GG] (GN) {$N$ independent diffusions};
\node[draw,rounded corners,fill=orange!10,minimum width=2.4cm,below=of GN] (GV) {role-bound storage};
\node[draw,rounded corners,fill=orange!10,minimum width=2.4cm,below=of GV] (GA) {fresh orthogonal code per event};
\node[below=0.1cm of GA,align=center] (Gcost) {cost: $E_B \gg E_A$\\ capacity grows with novelty rate};

\draw[-{Stealth}] (GI) -- (GG);
\draw[-{Stealth}] (GG) -- (GN);
\draw[-{Stealth}] (GN) -- (GV);
\draw[-{Stealth}] (GV) -- (GA);

\node[draw,rounded corners,fill=yellow!20,left=0.5cm of GG,align=center] (NS) {novelty signal\\ (ACh, DA)};
\draw[-{Stealth},dashed,red!70!black] (NS) -- (GG);
\end{tikzpicture}%
}
\caption{Two VaCoAl engineering schedules for exposition (Figure~\ref{fig:two-pathways} defines hippocampal Regime~A/B). Schedule~U (left, blue): one primitive polynomial and one fixed seed shared by all $N$ blocks. Schedule~G (right, orange): independent polynomials and novelty-gated per-block re-seeding. A complete stack can mix them; cholinergic/dopaminergic gating of DG~\cite{hassel2016neuromodulation} rhymes with Schedule~G at the systems level but is not a literal statement that CA3 diffusion is LFSR-identical.}
\label{fig:schedules}
\end{figure}
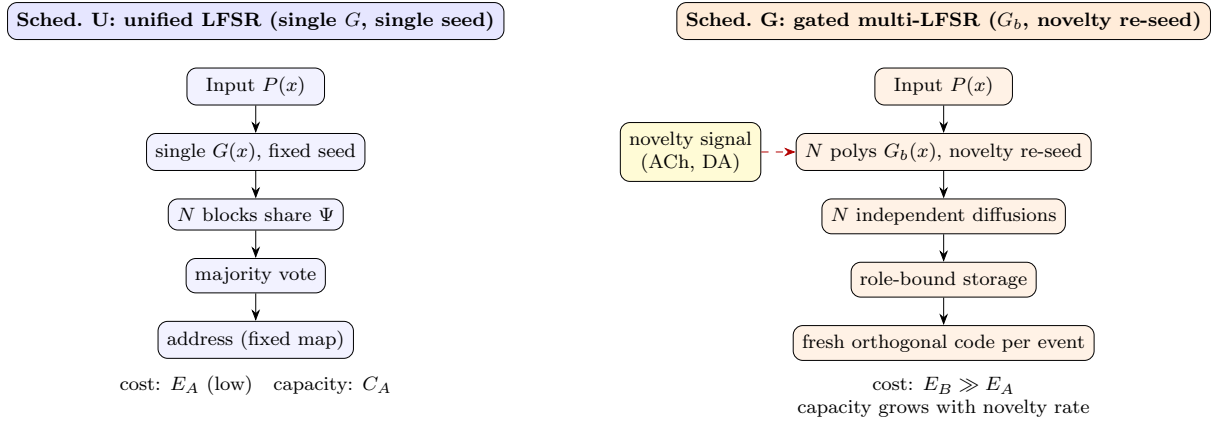

\subsection{Mapping Hippocampal Regime~A/B onto Pathways}\label{sec:mapping-AB}

We adopt the layered thesis: \textbf{Regime~A} is the entorhinal--CA3 scaffolded vector-index substrate —Vector-HaSH's natural home —where similarity and attractor completion dominate; \textbf{Regime~B} is the EC$\rightarrow$DG$\rightarrow$mossy-fiber$\rightarrow$CA3 route that supports sparse directed relational writes (edge-list-/ontology-like commitments) gated by neuromodulation and extended by adult neurogenesis~\cite{kempermann2018human,hassel2016neuromodulation}. Regime~A participates in both encoding and retrieval whenever grid-compatible structure must land in CA3; Regime~B is recruited when near-duplicate inputs require pattern separation and when role-structured events need asymmetric CA3 wiring.

\paragraph{Regime~A (EC$\leftrightarrow$CA3).} Layer-III (and related) perforant input plus CA3 recurrence is the Vector-HaSH-sufficient loop~\cite{chandra2025episodic}: content rides on a continuous scaffold; relations may remain implicit in manifold coordinates unless additional mechanisms break permutation symmetry.

\paragraph{Regime~B (EC$\rightarrow$DG$\rightarrow$CA3).} The layer-II trisynaptic detour supplies sparse granule activity and powerful mossy terminals—computationally aligned with writing a small labeled cut into CA3 rather than uniformly nudging a dense field.

The VaCoAl Schedule~U/Schedule~G discussion (Figure~\ref{fig:schedules}) is a separate, silicon-native parable for why a memory system might pair a cheap fixed-map mixer with a costly gated allocator. It illustrates commitments that rhyme with vector-db versus ontology-db economics; \emph{it is not the definition of hippocampal Regime~A/B, and it is not the same abstraction layer as the RR flag}.

The database analogy in Figure~\ref{fig:db-analogy} tracks representation class. Panel~(a) illustrates vector-index semantics (Regime~A-like); panel~(b) illustrates typed directed assertions (Regime~B-like). The mossy-fiber projection is sparse and directed; neurogenesis appends new outputs~\cite{akers2014hippocampal}. We do not claim literal Plate-style binding in axons~\cite{kleyko2023survey,plate1995hrr}.

\begin{figure}[t]
\centering
\begin{tikzpicture}[font=\scriptsize]
\node[font=\small\bfseries] at (0,3.6) {(a) Vector store (Regime~A-like)};
\filldraw[blue!10] (-2.2,-0.2) rectangle (2.2,3.0);
\node[circle,draw,fill=white,inner sep=2pt] (a1) at (-1.0,2.2) {Dog};
\node[circle,draw,fill=white,inner sep=2pt] (a2) at (1.2,2.0) {Bite};
\node[circle,draw,fill=white,inner sep=2pt] (a3) at (0.0,0.8) {Man};
\node[circle,draw,fill=white,inner sep=2pt] (a4) at (-1.4,0.4) {Cat};
\node[circle,draw,fill=white,inner sep=2pt] (a5) at (1.6,0.6) {Run};
\node[align=center,font=\tiny] at (0,-0.6) {Items as embeddings; similarity by distance.\\ Order-invariant: $v(\mathrm{Dog}) + v(\mathrm{Bite}) + v(\mathrm{Man})$\\ $= v(\mathrm{Man}) + v(\mathrm{Bite}) + v(\mathrm{Dog})$};

\node[font=\small\bfseries] at (8,3.6) {(b) Ontology store (Regime~B-like)};
\node[circle,draw,fill=orange!15,inner sep=2pt] (b1) at (6.2,2.6) {Dog};
\node[circle,draw,fill=orange!15,inner sep=2pt] (b2) at (8.0,2.6) {Bite};
\node[circle,draw,fill=orange!15,inner sep=2pt] (b3) at (9.8,2.6) {Man};
\draw[-{Stealth},thick] (b1) -- node[above,font=\tiny]{Subject} (b2);
\draw[-{Stealth},thick] (b2) -- node[above,font=\tiny]{Object} (b3);

\node[circle,draw,fill=orange!15,inner sep=2pt] (c1) at (6.2,1.0) {Dog};
\node[circle,draw,fill=orange!15,inner sep=2pt] (c2) at (8.0,1.0) {Bite};
\node[circle,draw,fill=orange!15,inner sep=2pt] (c3) at (9.8,1.0) {Man};
\draw[{Stealth}-,thick] (c2) -- node[above,font=\tiny]{Subject} (c3);
\draw[-{Stealth},thick] (c2) -- node[above,font=\tiny]{Object} (c1);

\node[align=center,font=\tiny] at (8.0,-0.2) {Directed edges with role labels.\\ Top: ``the dog bit the man''; bottom: ``the man bit the dog''.\\ Distinct ontologies $\Rightarrow$ distinct propositions.};
\end{tikzpicture}
\caption{Database-architectural analogy aligned with hippocampal Regime~A/B (not with VaCoAl Sched.\ U/G). (a)~Vector-index semantics—similarity geometry without native permutation sensitivity (Regime~A-like). (b)~Typed, directed assertions—role asymmetry between ``the dog bit the man'' and ``the man bit the dog'' (Regime~B-like). Mossy fibers provide sparse directed contacts; neurogenesis adds outputs~\cite{akers2014hippocampal}. Real circuits mix both semantics with task-dependent gain; the split is commitments, not exclusivity.}
\label{fig:db-analogy}
\end{figure}

\paragraph{A note on the database analogy.}
The vector-store-versus-ontology-store contrast in 
Figure~\ref{fig:db-analogy} serves to illustrate an algorithmic 
Pareto frontier regarding energy and capacity (formalized as 
Proposition~\ref{prop:tradeoff} below), rather than imposing strict 
anatomical exclusivity. The mammalian hippocampus does not implement 
a clean partition in which Regime~A handles only similarity-geometry 
queries and Regime~B handles only typed directed assertions; both 
regimes contribute to both query classes with task-dependent gain 
modulated by acetylcholine and 
dopamine~\cite{hassel2016neuromodulation,lisman2005hippocampal,duzel2010novelty}. 
The figure isolates \emph{representational commitments}---permutation-symmetric 
similarity geometry versus role-asymmetric directed binding---not 
exclusive anatomical assignments. Equivalently, the analogy describes 
which representational class each pathway is comparatively better 
suited to express, not which class it exclusively expresses.

\begin{proposition}[Energy--capacity--plasticity trade-off]\label{prop:tradeoff}
Consider an idealized memory system that must support both retrieval of familiar items and orthogonal encoding of novel items, under a fixed bound on bit-flip energy per memory operation. Let $E_A$ be the per-operation energy of a fixed-map orthogonalizer of capacity $C_A$, and $E_B$ the per-operation energy of a gated on-demand orthogonalizer that adds a fresh orthogonal code of capacity $C_B$ per activation. The subscripts $A,B$ index these abstract engineering components only; they are not hippocampal Regime labels and they are not Sched.~U/G labels. Suppose $E_A \ll E_B$ and $C_B \gg C_A$. If novel events occur at a rate $\nu$ and familiar retrievals at a rate $\phi$ with $\phi \gg \nu$, then for any total energy budget $E_{\mathrm{tot}}$ the energy-optimal architecture allocates:
\begin{itemize}
\item the always-on fixed-map orthogonalizer, paying $\phi E_A$ per unit time, and
\item a gated on-demand orthogonalizer engaged only on novelty, paying $\nu E_B$ per unit time,
\end{itemize}
yielding total energy $\phi E_A + \nu E_B$ and total capacity $C_A + \nu C_B$ (per unit time, integrated over the system's operating lifetime). Neither orthogonalizer alone can achieve this Pareto front: the fixed-map component alone caps capacity at $C_A$, and the gated component alone pays $\phi E_B$ for every retrieval, exhausting the energy budget when $\phi \gg \nu$.
\end{proposition}

\begin{proof}[Sketch]
The argument is a direct consequence of partitioning two qualitatively distinct workloads (frequent retrieval of fixed items, rare encoding of novel items) across two qualitatively distinct mechanisms with complementary energy/capacity profiles. Formally, this is the standard ``hierarchy of cheap-and-fast versus expensive-and-flexible'' trade-off that recurs across computer architecture (cache/main-memory hierarchy), network design (always-on links versus gated bursts), and biological systems (constitutive versus inducible enzyme expression). The proposition does not require the orthogonalizers to be biologically distinct in identity, only that the system can route an event to the appropriate orthogonalizer based on a novelty signal. Section~\ref{sec:dissociation} shows that the hippocampus has exactly such a routing mechanism, mediated by ACh and DA gating of the DG path.
\end{proof}

Proposition~\ref{prop:tradeoff} is informal in the sense that the precise capacity and energy parameters $C_A, C_B, E_A, E_B$ depend on architectural choices we have not pinned down; but it has the standard form of a hierarchical-resource trade-off. We will see that the empirical encoding/retrieval dissociation in the hippocampus matches this prediction qualitatively when read as state-dependent mixing of two operator classes rather than as strict division of labor.

While the transition between exact collision resolution ($RR=1$) and abstaining reads ($RR=0$) is a binary flag in VaCoAl, we hypothesize that neuromodulation enacts a continuous analog of this switch in biology. During familiar, low-interference retrieval, strong CA3 recurrent dynamics emulate Rescue mode by forcing convergence to a complete scaffold address. However, during novelty, increased cholinergic and dopaminergic tone  effectively raises the threshold for CA3 attractor-based pattern completion. This forces the system into a biological 'Don't Care' regime, where the network relies on the graded, branch-aware evidence accumulation (CR2) supported by the EC→DG→CA3 pathway. In this view, neuromodulators dynamically tune the system's tolerance for partial matches, sliding the effective architecture between Vector-HaSH-style index repair and VaCoAl-style branching inference.

\subsection{Empirical Dissociation: Encoding versus Retrieval}\label{sec:dissociation}

The strongest empirical support for Proposition~\ref{prop:tradeoff} comes from a now well-replicated dissociation between the two hippocampal pathways with respect to encoding and retrieval.

Lee \& Kesner~\cite{lee2004differential} reported that selective lesions of the direct EC$\to$CA3 pathway impair retrieval of previously learned spatial locations but spare new encoding, whereas Hunsaker, Mooy, Swift \& Kesner~\cite{hunsaker2007dissociations} and follow-up work~\cite{hunsaker2008evaluating} reported the converse: lesions or genetic disruption of the DG$\to$CA3 pathway impair encoding of novel spatial discriminations but spare retrieval of previously consolidated memories. McHugh et al.~\cite{mchugh2007dentate} showed that DG-specific NMDAR knockout selectively impairs rapid pattern separation during encoding; Nakazawa et al.~\cite{nakazawa2002requirement} showed that CA3-specific NMDAR knockout selectively impairs retrieval from partial cues. Hassel et al.~\cite{hassel2016neuromodulation} review the neuromodulatory gating of the DG path: cholinergic and dopaminergic inputs increase the gain of the perforant-path-to-DG synapse during novelty, selectively engaging the trisynaptic loop while leaving the direct path approximately constant.

This pattern matches the coarse prediction of Proposition~\ref{prop:tradeoff} once commitments are mapped to pathways: a low-cost Regime~A (EC$\leftrightarrow$CA3 vector scaffold, sufficient in principle in Vector-HaSH) paired with a higher-gain Regime~B (DG--mossy relational writer) that disproportionately affects encoding when interference is high. \emph{Lesions dissociate because the two operator classes can be weighted differently by task and neuromodulation, not because either pathway is silent during ``the other'' phase in intact animals.} This reading is conservative and consistent with Vector-HaSH's emphasis on the sufficiency in principle of the EC$\leftrightarrow$CA3 loop: the trisynaptic loop is not a retrieval-passive bystander, but its gain is preferentially raised by novelty.

\subsection{Primate Specialization of Dentate-Gyrus Granule Cells}\label{sec:primate-DG}

If the DG route is the principal anatomical substrate of on-demand pattern separation and ``new tag'' allocation, then phylogenetic elaborations of DG should be readable as expansions of that auxiliary orthogonalizer's effective $C_B$. The comparative anatomy of the dentate gyrus indeed shows two such elaborations specific to anthropoid primates and especially to humans.

\paragraph{Basal dendrites on granule cells.} Seress and colleagues~\cite{seress1992postnatal,seress2007comparative,seress1987basal} reported that primate dentate granule cells, unlike rodent granule cells, possess basal dendrites in addition to the canonical apical dendritic tree. Approximately half of human granule cells have basal dendrites, twice the rate observed in rhesus macaques and absent in standard rodent models. Some basal dendrites curve up into the molecular layer and receive perforant-path input alongside the apical tree; others extend into the hilus and receive associational input from mossy cells and hilar interneurons. The net effect is increased per-cell input integration capacity and the introduction of a second, partially independent dendritic compartment per granule cell.

\paragraph{Mossy-cell projection refinement.} Saito et al.~\cite{saito2024comparative} showed that primate mossy cells have associational projection patterns distinct from rodent mossy cells: in monkeys, no mossy-cell subpopulation makes commissural projections to the contralateral DG, and septal versus temporal mossy cells exhibit refined laminar specificity in their ipsilateral projections. The mossy-cell loop, which provides recurrent excitatory feedback to granule cells and also drives sparse-coding via indirect feedback inhibition~\cite{galloni2022recurrent}, is therefore organized differently in primates than in rodents.

In the VaCoAl framework these specializations are readable as increases in the effective parameters of the gated on-demand orthogonalizer (the engineering component indexed $C_B$ in Proposition~\ref{prop:tradeoff}, conceptually aligned with hippocampal Regime~B rather than with Sched.~G as a literal LFSR). Basal dendrites add a partially independent input compartment per granule cell, equivalent in our framework to splitting a block into sub-blocks with their own diffusion—increasing the effective number of independent blocks $N$ at fixed total cell count, which by Theorem~\ref{thm:det-qod}(d) tightens the QOD level achievable. Mossy-cell laminar refinement adds finer-grained intra-DG associational structure, which corresponds to better-controlled inter-block independence in per-block re-seeding. We propose, formally:

\begin{proposition}[Primate DG elaboration as refinement of the on-demand orthogonalizer]\label{prop:primate}
The primate-specific elaborations of dentate-gyrus granule cells (basal dendrites, refined mossy-cell laminar projections) correspond, within VaCoAl, to (i)~an increase in the effective number of independent blocks $N$ per fixed total granule-cell count, and (ii)~a decrease in residual inter-block correlation. By Theorem~\ref{thm:det-qod}(d), both refinements tighten the QOD concentration level achievable per unit DG volume, increasing the capacity $C_B$ for on-demand orthogonalization without proportionally increasing energy cost.
\end{proposition}

The proposition is testable: it predicts that primate DG should achieve sharper pattern separation per granule cell than rodent DG, and that the per-cell improvement should scale with the number of basal dendrites and the refinement of mossy-cell laminar specificity. We are not aware of direct comparative pattern-separation measurements at the per-cell level across species; if such measurements become possible, Proposition~\ref{prop:primate} provides a quantitative target.

\paragraph{Neurogenic temporal tagging as on-demand code expansion.}
A related line of work, the neurogenic temporal-tagging 
hypothesis~\cite{aimone2006potential,aimone2011resolving}, proposes 
that adult-born granule cells supply temporally distinguishable codes 
to events occurring within their integration window, providing an 
event-tag rather than a metric-time signal. This is consistent with 
the present framing: neurogenesis can be read as on-demand expansion 
of the gated orthogonalizer's effective code pool (VaCoAl's 
Schedule~G in Figure~\ref{fig:schedules}), with novelty gating 
mediated by ACh/DA tone supplying the analogue of per-event reseeding. 
Whether neurogenesis-derived codes additionally encode metric time, 
or only serve as cross-episode discriminators, is an empirical 
question beyond the scope of the present paper; either reading is 
compatible with our use of $\mathrm{CR2}$ as a discrete-event index.

\paragraph{Reduced human CA3 commissural connectivity as capacity doubling.}
A further primate-specific elaboration concerns CA3 itself. Rolls~\cite{rolls2023brain} highlights that, while rodent CA3 forms a single bilateral autoassociation network through extensive cross-midline commissural connections, the human hippocampal commissure is weak, leaving left and right CA3 effectively as two independent attractor networks. Because the dominant capacity term of an autoassociator scales with the number of recurrent collateral synapses per neuron (not with neuron count per se), splitting one bilateral network into two effectively independent networks of equal per-neuron connectivity \emph{doubles} the total number of storable memories rather than simply doubling neuron count. In the VaCoAl framework, this is the dual on the CA3 side of the on-demand orthogonalizer refinements described above on the DG side: doubling the number of effectively independent attractor banks at fixed per-bank capacity. Together, the DG-side elaborations (basal dendrites, mossy-cell laminar refinement, sustained neurogenesis) and the CA3-side architectural change (bilateral independence) converge on a single primate adaptation pattern: increased on-demand orthogonalization capacity and increased storage capacity, both achieved through architectural rather than gross-anatomical scaling.

\subsection{XOR as Energy-Efficient Binding: Convergent Optimization}\label{sec:xor}

A further architectural correspondence concerns the binding operation itself. VaCoAl implements compositional binding as XOR-and-shift over $\mathrm{GF}(2)$, an $O(N)$ operation that costs roughly 1--10~femtojoules per binary gate on modern CMOS~\cite{horowitz2014computing}. Holographic Reduced Representations using circular convolution cost $O(N\log N)$ multiplications per binding~\cite{plate1995hrr}; deep-learning embeddings perform binding by general matrix multiplication, far more expensive still per operation.

The biological binding operation in CA3 is qualitatively similar in its energy profile. CA3 recurrent collaterals are sparse: each pyramidal cell synapses onto roughly 2\%--5\% of other CA3 pyramidal cells~\cite{rolls2023brain}. Mossy-fiber inputs to CA3 are sparser still, with each granule cell contacting only 10--50 CA3 cells via large but few synaptic contacts. The Hebbian potentiation rule that binds co-active CA3 inputs is in essence a coincidence detector that activates only when both the pre- and post-synaptic neurons fire above threshold—formally analogous to an AND gate with delayed plasticity, or, when run in the time-reversal direction characteristic of LTD, to an XOR-like coincidence operation. Both VaCoAl's $\mathrm{GF}(2)$ binding and CA3's coincidence-based binding therefore implement sparse, two-input, locally computed binding rules, in stark contrast to the dense matrix multiplications of artificial deep networks.

This is not a coincidence in the loose sense. Both biological and silicon substrates face severe energy constraints: the brain operates at $\sim 20$~W total~\cite{attwell2001energy}, of which the hippocampus consumes roughly 1--2~W, and modern HDC accelerators face thermal density limits at $\sim 1$~W/cm$^2$~\cite{kleyko2023survey}. Both have settled on sparse, locally computed XOR-like binding, plausibly because such operations are an energy-efficient way to implement reversible compositional binding while preserving similarity geometry. We do not claim that XOR-like binding is the unique Pareto-optimal solution for biology; we note only that the engineering substrate VaCoAl and the biological substrate CA3 share an architectural commitment to sparse, locally computed coincidence-based binding under joint reversibility, similarity-preservation, and energy constraints.

The energy argument can be sharpened by reference to the broader constraint that ATP availability imposes on neural computation. Attwell and Laughlin~\cite{attwell2001energy} estimated that signaling consumes the dominant fraction of the brain's $\sim 20$~W metabolic budget, and Lennie~\cite{Lennie2003} argued that this budget forces cortical computation toward sparse codes and local operations. Backpropagation as implemented in artificial deep networks requires dense matrix multiplication and bidirectional propagation of error signals, neither of which is consistent with the ATP-bounded, locally constrained regime in which biological neurons operate. Lillicrap et al.~\cite{Lillicrap2020} review the evidence that exact backpropagation is not implemented in biological circuits and that approximate, local Hebbian variants are the realistic candidate; Whittington and Bogacz~\cite{Whittington2019} survey predictive-coding schemes that can approximate gradient signals through local computations. The point relevant here is comparative rather than dismissive: TEM and related models that require trained readouts remain biologically meaningful as normative accounts, but their training-phase computational demands are precisely the regime in which VaCoAl's fixed-$G(x)$ deterministic diffusion plus local XOR-based binding occupies a more energy-efficient frontier. Given that XOR-and-shift costs $\sim 1$--$10$~fJ per gate on CMOS~\cite{horowitz2014computing} and that biological coincidence detection achieves comparable per-operation efficiency through sparse Hebbian summation, the absence of XOR-like compositional binding from $\sim 520$~Myr of vertebrate hippocampal evolution would be surprising; its presence in the form of mossy-fiber-driven sparse coincidence binding in CA3 (Section~\ref{sec:biophysical-galois}) is what one would predict from energy considerations alone.

\subsection{Summary}\label{sec:two-orth-summary}

Sections~\ref{sec:schedules}--\ref{sec:xor} answer the section title. Vector-HaSH establishes Regime~A (EC$\leftrightarrow$CA3) suffices in principle; evolution nonetheless conserves Regime~B (DG--mossy) for non-substitutable relational-write statistics. The dual-readout-schedule analogy proposes that phylogenetic coexistence parallels rapid switching among operator demands modeled in silicon by $\mathrm{RR}=1$ versus $\mathrm{RR}=0$ readout regimes~\cite{chuma2026vacoal}. Proposition~\ref{prop:tradeoff} formalizes the engineering energy--capacity split between fixed-map and gated components; VaCoAl Sched.~U/G (Figure~\ref{fig:schedules}) is a silicon parable for that split—\emph{not} the definition of Regime~A/B and \emph{not} the RR flag. Lesion dissociations~\cite{mchugh2007dentate,nakazawa2002requirement,lee2004differential,hunsaker2007dissociations,hassel2016neuromodulation} are consistent with complementary Regime~A/Regime~B operator classes under differing gains. Primate DG elaborations tighten Regime~B's effective $C_B$ (Proposition~\ref{prop:primate}). XOR-like binding in VaCoAl and CA3 remains an architectural-commitment parallel, not a substrate identity claim.

The overall picture is that VaCoAl's architectural choices intersect with the trade-off frontier on which the hippocampal--entorhinal circuit appears to sit. We frame this as an architectural-commitment correspondence. Whether the parallel reflects shared selection pressures (in biology) and engineering pressures (in silicon) is a question the present paper does not resolve; we note only that the algebraic and biological architectures land on similar points in design space when similar constraints are imposed.

\section{Empirical and Engineering Predictions}\label{sec:predictions}

The two formal correspondences yield specific predictions that distinguish the present account from previous proposals.

\subsection{Prediction 1: Multiplicative Decay Curve in iEEG Multi-Hop Replay}

If Proposition~\ref{prop:cr2} is correct, then in human iEEG studies of multi-step episodic replay, the average decoded replay fidelity as a function of sequence length $n$ should follow
\begin{equation}\label{eq:pred1}
\bar{F}(n) \approx \bar{p}^{\,n},
\end{equation}
where $\bar{p}\in(0,1)$ is the trial-averaged per-step reactivation probability. Equivalently, $\log \bar{F}(n)$ should be approximately linear in $n$ with negative slope $\log \bar{p}$. The current literature reports multiplicative decay qualitatively~\cite{reithler2025ripple}; the prediction here is that quantitative fits to Eq.~\eqref{eq:pred1} should outperform additive or power-law alternatives across studies. Existing datasets from Sakon \& Kahana~\cite{sakon2022hippocampal} and Norman et al.~\cite{norman2019hippocampal,norman2021hippocampal} would suffice to test this directly.

A stronger version of the prediction concerns subject-by-subject variability. The geometric mean of per-step ripple-content reactivation strength, measured across single-step recall trials in a given subject, should predict the slope $\log \bar{p}$ of that subject's multi-step decay curve. We are not aware of any current model that makes this quantitative prediction.

\subsection{Prediction 2: Frontier-Bounded Creative Ideation}

In VaCoAl, the Frontier Size (FS) bounds the number of concurrent paths tracked during Don't Care retrieval, preventing combinatorial explosion during branching reasoning. We posit that medial-prefrontal (mPFC) schema control provides the biological implementation of this algorithmic bound. The recent finding that hippocampal ripples relax mPFC constraints during creative ideation  corresponds formally to an algorithmic increase in FS.

We predict that for a creativity task with controlled difficulty, the inter-trial variance of ripple-locked replay trajectories will increase proportionally to the reduction in mPFC inhibitory tone. When mPFC control is high (rigid schema), FS is small, and only paths with near-perfect CR2 scores survive; when mPFC control is relaxed, FS increases, allowing structurally divergent paths with lower intermediate CR2 scores to enter the recall frontier. Direct stimulation studies in mPFC could quantify this exact relationship between schema rigidity and replay-trajectory variance.

\subsection{Prediction 3: Engineering Equivalence Under Substitution}

Corollary~\ref{cor:equivalence} predicts that any HDC system whose computational properties depend on the QOD of its scaffold projection can be modified to use Galois-field diffusion in place of random projection without functional loss, and with strict gains in determinism, auditability, and worst-case avalanche behavior at minimum-weight perturbations (Remark~\ref{rem:comparison}). This is a directly testable claim for the HDC engineering literature: we predict that VaCoAl-style diffusion can be substituted for random projection in standard HDC benchmarks (e.g., classification tasks in~\cite{kleyko2023survey,raviv2024linear}) with equal or better accuracy and improved energy efficiency on hardware that exploits LFSR's $O(1)$ shift cost.

\subsection{Prediction 4: Quantitative Encoding/Retrieval Dissociation Curve}

Proposition~\ref{prop:tradeoff} predicts that selectively disabling each pathway has asymmetric marginal costs. Let $r_{\mathrm{enc}}$ and $r_{\mathrm{ret}}$ denote behavioral encoding and retrieval performance. Under selective disruption of Regime~B (EC$\to$DG$\to$CA3; e.g., DG NMDAR knockout~\cite{mchugh2007dentate}, focal DG lesion, or pharmacological blockade of cholinergic gating), we predict $\Delta r_{\mathrm{enc}}/\Delta r_{\mathrm{ret}} \approx C_B/(C_A+C_B) \gg 1$ when $C_B \gg C_A$ (engineering subscripts as in the proposition—not Sched.~U/G labels). Under selective disruption of Regime~A's EC$\to$CA3 direct drive~\cite{lee2004differential}, retrieval-leaning paradigms that depend on scaffold-consistent CA3 completion should show disproportionate impairment because Regime~A carries the Vector-HaSH suffices-for-episodic skeleton; this is not a claim that intact Regime~A is retrieval-only. Ratio predictions are workload-specific; joint lesion-meta-analysis remains an open target~\cite{mchugh2007dentate,nakazawa2002requirement,lee2004differential,hunsaker2007dissociations}.

\subsection{Prediction 5: Primate versus Rodent Per-Cell Pattern Separation}

Proposition~\ref{prop:primate} predicts that the per-granule-cell pattern-separation efficiency in primate DG should exceed that of rodent DG by a factor that scales with the increase in effective per-cell block count (basal-dendrite frequency $\times$ mossy-cell laminar refinement). For human DG, where basal-dendrite frequency is roughly twice that of macaque~\cite{seress1987basal}, we predict a corresponding factor-of-two improvement in per-cell QOD level, holding total cell count constant. This prediction would be tested by single-cell-resolution electrophysiology (or two-photon imaging) of pattern-separation-induced activity in DG across species, normalizing for total granule cell count and input rate.

\section{Discussion}\label{sec:discussion}

\subsection{What the Bridge Buys}

Computational neuroscience has long faced a tension between two views of the hippocampus: a normative-mathematical view (Marr, Treves \& Rolls, Vector-HaSH, TEM) and a phenomenological-electrophysiological view (Buzs\'aki, Wilson, Foster, Norman, Kahana). The two have been productive in parallel but rarely produce the same kind of statement. The first produces equations; the second produces empirical decay curves. Our two formal correspondences—Theorem~\ref{thm:det-qod} and Proposition~\ref{prop:cr2}—show that VaCoAl can produce both kinds of statement simultaneously, in a single algebraic framework. This is the principal value of the bridge.

For artificial intelligence, the bridge identifies VaCoAl as an auditable substrate for memory-based reasoning whose properties have biological warrant. Unlike deep-learning embeddings that mix constituents irreversibly, and unlike hash tables that destroy similarity intentionally, VaCoAl provides reversible compositional binding and a path-quality trace that survives multi-hop reasoning. The CR1/CR2 scores function as built-in explainability outputs whose semantics are mathematically defined rather than post-hoc constructed.

For computational neuroscience, the bridge offers an algebraically tractable model of multi-hop replay-fidelity decay. Vector-HaSH and TEM are silent on replay dynamics; the SWR-replay literature is rich in dynamics but sparse on tractable models. CR2 fills exactly this gap. The expected decay form $\bar{F}(n) \approx \bar{p}^{\,n}$ from Eq.~\eqref{eq:pred1} is, to our knowledge, the first quantitative prediction connecting algebraic memory architectures to empirical replay-fidelity measurements.

Readers focused on hippocampal error correction should keep the scaffold--content decomposition in view~\cite{chandra2025episodic}. VaCoAl does not negate the relevance of recurrent grid--CA3 stabilization; nor does Galois substitution for random $W_{gh}$ magically supply attractor contraction without recurrence. Rather, deterministic QOD strengthens the substrate story for the scaffold map, CR2 strengthens the branching-replay narrative on the CA1/cortex readout stages, and the Regime~A/Regime~B split suggests how feedforward orthogonalization and scaffold recurrence could trade labor under neuromodulatory routing~\cite{lee2004differential,hassel2016neuromodulation}. Future normative extensions that unify explicit DG circuitry with scaffold attractors can treat Theorem~\ref{thm:det-qod} as a permissible projection-level choice and Appendix~\ref{app:vector-hash} as bookkeeping for Rescue versus Don't Care regimes.

A third value of the bridge, distinct from the engineering and neuroscience contributions above, is conceptual. The cellular and synaptic evidence reviewed in Section~\ref{sec:biophysical-galois}---mossy-fiber detonator transmission, granule-cell sparse binary coding, anatomically specified targeting, and CA3 recurrent winner-take-all dynamics---was accumulated over more than two decades within explanatory frameworks centered on pattern separation, attractor dynamics, and Hebbian binding. None of these frameworks naturally prompted the question whether the assemblage implements deterministic Galois-field arithmetic over $\mathrm{GF}(2)$. That question is generated by taking VaCoAl's algebraic substrate seriously as a biological hypothesis rather than as engineering convenience. Independent of whether the strong reading of Section~\ref{sec:biophysical-galois} survives further scrutiny, the reframing illustrates a methodological point: engineering substrates designed under hard constraints (energy, auditability, reversibility) can supply integrative questions that purely biological framings do not generate on their own. The same constraints that drove VaCoAl toward Galois-field algebra in silicon plausibly drove vertebrate hippocampal evolution toward functionally analogous structure in tissue, and the algebraic substrate makes the analogy precise enough to test.

\subsection{Relation to Earlier Work}

Several earlier proposals partially anticipate the bridge. Kanerva's SDM~\cite{kanerva1988sdm,kanerva2009hyperdim} was already inspired by the cerebellum and hippocampus; our Theorem~\ref{thm:det-qod} can be read as a deterministic strengthening of the SDM address-decoder analysis. Raviv's linear codes for HDC~\cite{raviv2024linear} use error-correcting codes for HDC operations but treat them as a forced-convergence tool rather than a quasi-orthogonal scoreboard. Babadi \& Sompolinsky~\cite{babadi2014sparseness} and Litwin-Kumar et al.~\cite{litwinkumar2017optimal} analyze the role of sparse expansion in pattern separation but assume randomness; our work shows the randomness assumption is unnecessary. Cayco-Gajic \& Silver~\cite{caycogajic2019reevaluating} review pattern separation across cerebellum, hippocampus, and mushroom body but do not propose a deterministic substitute for random expansion. The Spens \& Burgess generative model~\cite{spens2024generative} provides a Bayesian framework for memory construction and consolidation that complements ours but operates at a different level of abstraction.

\subsection{Limitations}

Several limitations should be made explicit.

First, Proposition~\ref{prop:cr2} assumes conditional independence of per-step ripple events. This is a first-order approximation; in tasks with strong sequential structure, ripples likely exhibit positive correlations through shared cortical state. Empirical fits to Eq.~\eqref{eq:pred1} should include a correction term for such correlations, and the deviation from pure multiplicative decay is itself a measurable quantity.

Second, the QOD property in Theorem~\ref{thm:det-qod} is a first-order property of LFSR diffusion; biological DG--CA3 connectivity may exhibit higher-order structure (e.g., place-cell modular organization) that our analysis does not capture. The correspondence in Section~\ref{sec:synthesis} is therefore at the level of computational principles, not at the level of full biological detail.

Third, our predictions assume that current iEEG decoding methods can resolve per-step replay fidelity reliably. The state of the art~\cite{liu2021experience,schwartenbeck2023generative} suggests this is feasible, but the prediction in Eq.~\eqref{eq:pred1} is best tested on tasks designed for multi-hop replay measurement rather than retrospectively on existing datasets.

Fourth, while Section~\ref{sec:two-orthogonalizers} engages the comparative anatomy of hippocampal pathways and primate-specific DG specialization, the present paper does not engage the developmental biology of hippocampal memory (e.g., infantile amnesia and dentate-gyrus neurogenesis~\cite{josselyn2012infantile,akers2014hippocampal}) or the prefrontal-cortical schema systems implicated in creativity. Each of these is a productive extension; we focus here on the substrate-level architectural bridge and on the evolutionary trade-off that justifies the trisynaptic-plus-direct architecture.

Fifth, the relationship between VaCoAl's silicon majority voting and Vector-HaSH's continuous-time scaffold attractors is one of \emph{architectural commitment}, not implementation. Silicon majority voting is a discrete analogue of CA3 attractor dynamics; the continuous-time recurrent contraction modeled in Vector-HaSH is a richer object than the present paper claims to instantiate. Appendix~\ref{app:vector-hash} discusses this carefully and identifies regimes where the two models diverge as well as those where they converge.

Sixth, the present paper does not engage the time-cell literature~\cite{Pastalkova2008},\cite{Macdonald2011} or 
Laplace-transform temporal-context models (Howard and Eichenbaum's 
TCM family) in mechanistic detail. We have argued in 
Section~\ref{sec:two-senses-of-time} that $\mathrm{CR2}$'s logical 
hop index is orthogonal to the metric within-episode coding that 
time cells provide, and the two should be read as complementary 
substrates rather than competing ones. A full reconciliation---in 
particular, one that models how continuous temporal scaffolds and 
discrete event indices interact during ripple-mediated replay, and 
how Laplace-context drift might modulate per-step $\mathrm{CR1}$ 
values across hops separated by long versus short intervals---remains 
a productive extension. The present paper's empirical predictions 
(Section~\ref{sec:predictions}) concern hop-count dependence under 
controlled inter-hop intervals, leaving the orthogonal 
interval-dependence question to follow-up work.

\subsection{Broader Implications}

The architectural principles identified by Vector-HaSH and TEM can be realized either by random sparse biological wiring or by deterministic algebraic diffusion (Corollary~\ref{cor:equivalence}). This shifts the question of why these principles arose: the relevant constraints appear to be similarity preservation, compositional reversibility, and bounded multi-hop search, rather than the specific implementation choices (random versus deterministic) of any particular substrate. Murray, Wise \& Graham~\cite{murray2017evolution} have argued that the vertebrate hippocampus has been a stable computational substrate for over 520 million years, with primate elaborations layered on top (Figure~\ref{fig:evolution}). The architectural-commitment correspondence we develop here is consistent with their thesis: the same constraints can plausibly account for the persistence of the DG--CA3--CA1--EC architecture across vertebrate evolution.

A key advantage of deterministic Galois-field diffusion over random sparse projection is the worst-case avalanche behavior at minimum-weight perturbations (Theorem 9(e)). While the biological hippocampus does not literally shift bits through a register, the dentate gyrus exhibits structural properties that approximate this deterministic avalanche. A minimal perturbation in EC input is massively amplified by the divergent perforant path and the extreme sparsity of granule cell firing. Furthermore, the powerful, 'detonator' synapses of the mossy fibers  ensure that even a single-neuron difference in the DG code can reliably flip the downstream CA3 ensemble into an entirely distinct state. This circuit motif achieves the biological equivalent of our minimal-perturbation bounds, providing structural determinism where random projection would suffer from bounded failure modes.

A practical corollary for neuro-symbolic AI is that auditable memory-based reasoning need not sacrifice the computational benefits of distributed representations. VaCoAl shows that determinism, reversibility, and explainability are compatible with high-dimensional similarity preservation, provided the projection is replaced by an appropriate algebraic diffusion.

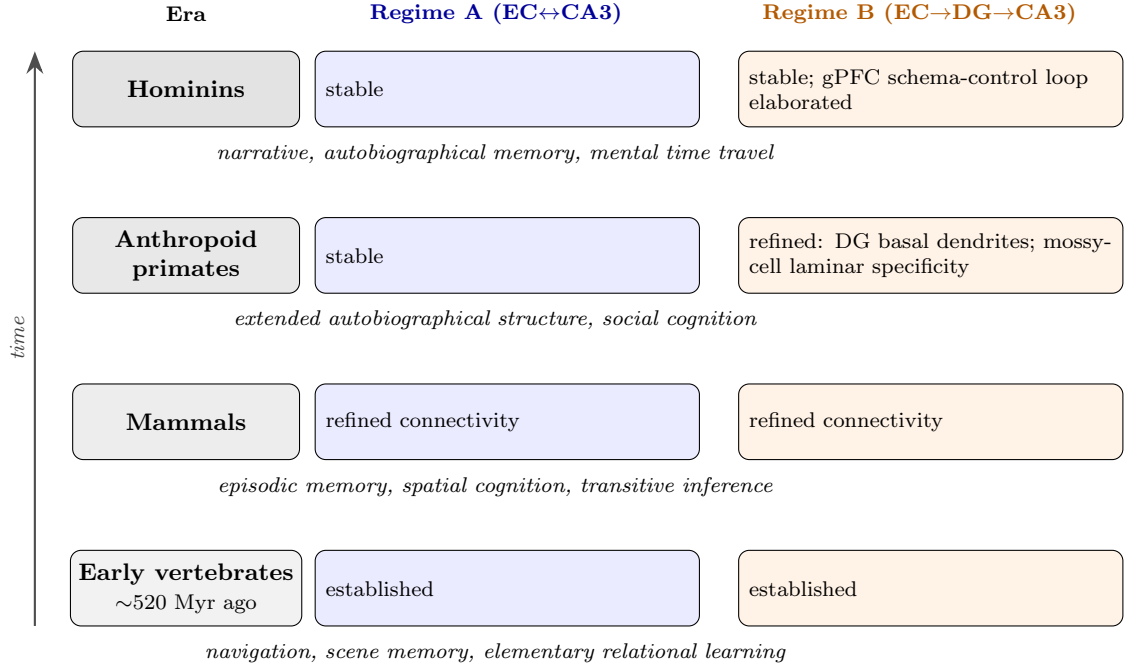
\begin{figure}[h]
\centering
\begin{tikzpicture}[font=\footnotesize,
  era/.style={draw,rounded corners,fill=gray!15,
    minimum width=3.0cm,minimum height=1.0cm,
    align=center,inner sep=3pt,font=\footnotesize\bfseries},
  Acol/.style={draw,rounded corners,fill=blue!8,
    text width=4.8cm,minimum height=1.0cm,
    align=left,inner sep=4pt,font=\scriptsize},
  Bcol/.style={draw,rounded corners,fill=orange!10,
    text width=4.8cm,minimum height=1.0cm,
    align=left,inner sep=4pt,font=\scriptsize},
  cog/.style={text width=14.5cm,align=center,
    font=\scriptsize\itshape,inner sep=2pt}
]

\node[font=\scriptsize\bfseries] at (0,0)    {Era};
\node[font=\scriptsize\bfseries,text=blue!60!black]   at (4.1,0)  {Regime A (EC$\leftrightarrow$CA3)};
\node[font=\scriptsize\bfseries,text=orange!70!black] at (9.7,0)  {Regime B (EC$\to$DG$\to$CA3)};

\node[era,fill=gray!22] (E1) at (0,-1.0)    {Hominins};
\node[Acol,anchor=west] (A1) at (1.7,-1.0)  {stable};
\node[Bcol,anchor=west] (B1) at (7.3,-1.0)  {stable; gPFC schema-control loop elaborated};
\node[cog,anchor=north]       at (4.1,-1.65) {narrative, autobiographical memory, mental time travel};

\node[era,fill=gray!18] (E2) at (0,-3.2)    {Anthropoid\\ primates};
\node[Acol,anchor=west] (A2) at (1.7,-3.2)  {stable};
\node[Bcol,anchor=west] (B2) at (7.3,-3.2)  {refined: DG basal dendrites; mossy-cell laminar specificity};
\node[cog,anchor=north]       at (4.1,-3.85) {extended autobiographical structure, social cognition};

\node[era,fill=gray!14] (E3) at (0,-5.4)    {Mammals};
\node[Acol,anchor=west] (A3) at (1.7,-5.4)  {refined connectivity};
\node[Bcol,anchor=west] (B3) at (7.3,-5.4)  {refined connectivity};
\node[cog,anchor=north]       at (4.1,-6.05) {episodic memory, spatial cognition, transitive inference};

\node[era,fill=gray!10] (E4) at (0,-7.6)    {Early vertebrates\\ \scriptsize\mdseries$\sim$520 Myr ago};
\node[Acol,anchor=west] (A4) at (1.7,-7.6)  {established};
\node[Bcol,anchor=west] (B4) at (7.3,-7.6)  {established};
\node[cog,anchor=north]       at (4.1,-8.25) {navigation, scene memory, elementary relational learning};

\draw[-{Stealth[length=3mm]},thick,gray!60!black] 
  (-2.0,-8.1) -- (-2.0,-0.5) node[midway,sloped,above,font=\scriptsize\itshape,gray!50!black]{time};

\end{tikzpicture}
\caption{Evolutionary trajectory of the dual-channel hippocampal--entorhinal interface mapped onto the Murray--Wise--Graham framework~\cite{murray2017evolution}. Each row gives the status of \textbf{\textcolor{blue!60!black}{Regime~A}} (EC$\leftrightarrow$CA3 scaffold-vector channel; the Vector-HaSH substrate) and \textbf{\textcolor{orange!70!black}{Regime~B}} (EC$\to$DG$\to$CA3 trisynaptic channel) at the corresponding epoch, with representative cognitive capacities listed in italics below each row. Both regimes are inherited from early vertebrates and elaborated through mammalian, anthropoid-primate, and hominin radiations. Primate-specific refinements appear on the Regime~B side (basal dendrites on dentate-gyrus granule cells, mossy-cell laminar specificity); the elaborated granular prefrontal cortex (gPFC) schema-control loop arises late. The figure is an organizing visualization of the evolutionary frame, not a derivation; the cognitive capacities listed at each level are illustrative and not in one-to-one correspondence with the anatomical changes.}
\label{fig:evolution}
\end{figure}

\subsection{Connection to Pearl's Ladder of Causation}\label{sec:pearl-ladder}

The substrate-level framework developed in the preceding sections connects directly to one of the most influential formal theories of causal reasoning, Pearl's three-rung Ladder of Causation~\cite{pearl2018book,pearl2009causality}. Pearl distinguishes three progressively more powerful classes of causal queries, each requiring computational primitives that the rungs below do not supply: \emph{rung~1, Association}, captured by conditional probabilities $P(Y \mid X)$ and answerable from passive observation alone; \emph{rung~2, Intervention}, captured by the do-operator $P(Y \mid \text{do}(X))$, which requires surgical alteration of one variable holding others fixed and is provably unanswerable from observational data alone~\cite{pearl2009causality}; and \emph{rung~3, Counterfactuals}, captured by $P(Y_x \mid X', Y')$, which requires holding both the factual world and a contrary-to-fact world simultaneously and comparing them~\cite{pearl2018book}. Pearl establishes formally that each rung is strictly more expressive than the rungs below~\cite{pearl2009causality}, and identifies counterfactual reasoning as the capability most distinctive of human cognition and most clearly absent from current deep-learning systems~\cite{pearl2018book}. The framework developed in the present paper, we suggest, supplies a candidate substrate-level account of how the hippocampal--entorhinal circuit climbs all three rungs, and provides a principled answer to the question why evolution conserved \emph{both} orthogonalization pathways rather than only one.

The mapping is direct. Rung~1 (association) corresponds to associative retrieval from a stable scaffold: a partial cue indexes a hippocampal address via the EC$\leftrightarrow$CA3 loop, attractor contraction restores the scaffold state, and the associated content is read out. Vector-HaSH's scaffold-attractor dynamics and VaCoAl's CR1 retrieval (Sections~\ref{sec:correspondence-I}--\ref{sec:synthesis}) both implement this rung; Regime~A (EC$\leftrightarrow$CA3) supplies the necessary architectural primitives. Rung~2 (intervention) requires more. The defining algebraic property of the do-operator is \emph{surgical modification}: do$(X{=}x')$ replaces the value of variable $X$ with $x'$ while leaving every other structural relation in the causal model untouched~\cite{pearl2009causality}. Implementing surgical modification on a compositional representation $\mathrm{Repr} = \sum_i (\mathrm{Role}_i \otimes \mathrm{Filler}_i)$ requires that one specific role binding be unbound, replaced, and re-bound \emph{exactly}---without disturbing the other bindings and without losing information through approximate decoding. This is precisely the property that distinguishes VaCoAl's reversible $\mathrm{GF}(2)$ XOR-and-shift binding (Section~\ref{sec:bg-vacoal}, Eq.~\eqref{eq:binding}) from approximate compositional schemes such as circular convolution~\cite{plate1995hrr}, where unbinding incurs noise that accumulates across operations, and from learned embeddings that fold constituents irreversibly into dense vectors~\cite{kleyko2023survey}. Reversible $\mathrm{GF}(2)$ binding gives the do-operator its precise algebraic implementation on a high-dimensional substrate; without it, rung~2 queries cannot be answered exactly under repeated intervention. We do not claim VaCoAl is the unique algebraic substrate supporting rung~2, but among candidate hyperdimensional substrates it is the only one whose binding operations are simultaneously reversible, local, and energy-efficient (Section~\ref{sec:two-orthogonalizers}).

Rung~3 (counterfactuals) imposes a structurally new requirement that, we argue, makes the conserved two-orthogonalizer architecture not merely useful but \emph{necessary}. A counterfactual query $P(Y_x \mid X', Y')$ asks the probability that $Y$ would have taken value $Y_x$ under a hypothetical $\text{do}(X{=}x)$, given that the actual observed values were $X'$ and $Y'$. Pearl's structural causal model evaluates this by holding the factual world fixed (to determine the values of exogenous variables consistent with the observed evidence) while simultaneously instantiating the counterfactual world (to compute what $Y$ would have been under the alternative intervention)~\cite{pearl2009causality,pearl2018book}. Implementing this on a finite memory substrate requires \emph{two non-interfering orthogonal representations} of the same underlying variables: one anchoring the factual world, one minting the counterfactual world, both queryable in parallel. A single-scaffold architecture cannot support this without interference: writing the counterfactual on top of the factual scaffold would overwrite the very information that defines the contrast, and the multi-hop CR2 traces required to evaluate $Y_x$ and $Y'$ would mutually contaminate. The conserved two-orthogonalizer hippocampal architecture supplies exactly the structure rung~3 demands. \emph{Regime~A} (EC$\leftrightarrow$CA3, stable scaffold) anchors the factual world: the scaffold attractor stabilized over a lifetime of associative writes provides the durable representation against which counterfactuals are evaluated. \emph{Regime~B} (EC$\rightarrow$DG$\rightarrow$MF$\rightarrow$CA3, on-demand orthogonalization) mints the counterfactual world: novelty-gated mossy-fiber writes onto sparse, anatomically specified CA3 targets generate fresh orthogonal codes that do not collide with the factual scaffold, and adult neurogenesis~\cite{kempermann2018human} maintains lifelong capacity for this counterfactual minting. The CR2 path-integral statistic (Section~\ref{sec:correspondence-II}) provides the comparison machinery: a counterfactual probability $P(Y_x \mid X', Y')$ admits an algorithmic estimator as a ratio of CR2 traces computed over the counterfactual and factual scaffolds respectively. The Frontier Size (FS; Section~\ref{sec:predictions}) bounds the number of counterfactual worlds that can be entertained simultaneously, supplying the working-memory constraint that humans empirically display in counterfactual reasoning tasks.

\paragraph{An explicit algorithmic estimator for $P(Y_x \mid X', Y')$.}
The Regime~A/B mapping permits a more precise statement. Pearl's structural causal model evaluates a counterfactual query in three steps~\cite{pearl2009causality}: (i) \emph{abduction}---infer the values of exogenous (background) variables $U$ from the factual evidence $\{X', Y'\}$; (ii) \emph{action}---perform $\text{do}(X{=}x)$ by surgical modification, replacing the structural equation for $X$ with the constant assignment $X{=}x$ while leaving all other structural relations untouched; (iii) \emph{prediction}---compute the value of $Y$ in the modified model under the abducted background values $U$. On the VaCoAl substrate, these three steps map directly onto the two-scaffold architecture as follows. Let $\Pi_A$ denote the factual scaffold maintained by Regime~A, with stored memories indexed by addresses produced via the diffusion $\Psi$ (Eq.~\eqref{eq:diffusion}). Let $\Pi_B$ denote a counterfactual scaffold instantiated transiently by Regime~B through novelty-gated re-seeding of the block-wise diffusion (Section~\ref{sec:schedules}, Schedule~G). Then:
\begin{align}
\text{abduction:} \quad & \hat U \;=\; \mathrm{Retrieve}_{\Pi_A}(X', Y'), \label{eq:cf-abduction}\\
\text{action:}    \quad & \Pi_B \;\leftarrow\; \mathrm{Bind}\!\left(\hat U,\; \mathrm{Unbind}(X) \cdot x\right), \label{eq:cf-action}\\
\text{prediction:}\quad & \hat Y_x \;=\; \mathrm{Retrieve}_{\Pi_B}\!\left(\hat U,\, X{=}x\right). \label{eq:cf-prediction}
\end{align}
In Eq.~\eqref{eq:cf-action}, the surgical-modification semantics of $\text{do}(X{=}x)$ are realized by VaCoAl's reversible $\mathrm{GF}(2)$ operations: $\mathrm{Unbind}(X)$ extracts and removes the factual binding of variable $X$, leaving every other role--filler binding in the compositional representation intact, and the new binding $X{=}x$ is then composed in via $\mathrm{Bind}$. This is the substrate-level instantiation of Pearl's prescription to ``remove all arrows that point to the intervened variable'' on the causal DAG~\cite{pearl2018book}: at the algebraic level, the role's binding is severed and replaced atomically, with no spillover into the abducted background. The counterfactual probability estimator becomes
\begin{equation}\label{eq:cf-estimator}
\hat P(Y_x \mid X', Y') \;=\; \frac{\mathrm{CR2}_{\Pi_B}\!\left(\hat U \to X{=}x \to Y_x\right)}{\mathrm{CR2}_{\Pi_A}\!\left(\hat U \to X' \to Y'\right)},
\end{equation}
where the numerator is the multi-hop confidence trace through the counterfactual scaffold for the chain $\hat U \to X{=}x \to Y_x$, and the denominator is the analogous trace through the factual scaffold for the actually observed chain $\hat U \to X' \to Y'$ (recall Eq.~\eqref{eq:cr2}). The ratio normalizes for path-length-dependent CR2 decay, so chains of comparable depth yield well-conditioned estimates; the Frontier Size $\mathrm{FS}$ controls how many such pairs may be evaluated in parallel. Eqs.~\eqref{eq:cf-abduction}--\eqref{eq:cf-estimator} make explicit the claim that the two-orthogonalizer architecture is necessary: $\Pi_A$ and $\Pi_B$ must be simultaneously addressable, mutually non-interfering, and joined only at the abducted background $\hat U$ that propagates from the former to the latter through the action step. A single-scaffold architecture cannot realize this three-step procedure without destroying either the factual reference or the surgical character of the action step.

This reading yields an evolutionary rationale stronger than the energy--capacity--plasticity argument of Section~\ref{sec:two-orthogonalizers}. The earlier argument established that two orthogonalizers are \emph{efficient} given mixed task statistics. The Pearl-based argument establishes that two orthogonalizers are \emph{necessary} for any organism that performs rung-3 counterfactual reasoning, because parallel non-interfering representation of factual and counterfactual worlds requires the architectural separation that Regime~A and Regime~B together provide. If counterfactual reasoning is a meaningful capability under selection pressure---and human cognitive ecology suggests it is, supporting planning, regret, learning from absent alternatives, and social cognition involving others' possible actions~\cite{schacter2007remembering,schwartenbeck2023generative}---then the two-orthogonalizer architecture is selected for, and the conservation across $>520$\,Myr of vertebrate evolution~\cite{murray2017evolution} reflects the universal utility of counterfactual capacity rather than mere redundancy.

The same line of argument articulates with Harari's account of the Cognitive Revolution~\cite{harari2015sapiens}, which Pearl himself invokes as the evolutionary moment when humans acquired the capacity to ``imagine things that have never existed''~\cite{pearl2018book}. Pearl identifies the Lion Man sculpture of Stadel Cave---a chimera roughly 40{,}000 years old---as the earliest material evidence of counterfactual capacity, the ``precursor of every philosophical theory, scientific discovery, and technological innovation''~\cite{pearl2018book}. On the present account, the Cognitive Revolution is not best read as the appearance of an entirely new neural mechanism but as the cognitive-ecological exploitation of an already-conserved two-orthogonalizer architecture whose rung-3 capability had long been latent. Figure~\ref{fig:evolution} aligns with this reading: Regime~A and Regime~B are inherited from early vertebrates ($\sim 520$\,Myr), refined through mammalian and anthropoid radiations, and elaborated in primates through DG basal dendrites, mossy-cell laminar specificity, and (in humans) bilateral CA3 independence (Section~\ref{sec:primate-DG}). What distinguishes the Cognitive Revolution in this picture is not the emergence of new substrate but the elaborated granular prefrontal cortex schema-control loop that supplies the working-memory machinery---a sufficiently large and flexibly modulated Frontier Size (Section~\ref{sec:predictions})---to operate the two-orthogonalizer substrate at the depth required for sustained counterfactual narratives, mental time travel, and chimeric imagination. The architecture, in other words, was already in place; what changed was the cortical infrastructure that could exploit it.

Three reservations must be retained. First, Pearl's structural causal model formalism requires DAGs with explicit causal semantics; VaCoAl's reasoning benchmarks (e.g., the Wikidata mentor--student DAG of $\sim 470{,}000$ records, Section~\ref{sec:bg-vacoal}) are relational rather than strictly causal. Demonstrating that VaCoAl's algebraic operations implement SCM intervention semantics on \emph{causal} DAGs is a non-trivial extension that we have not formally established here. Second, the claim that reversible $\mathrm{GF}(2)$ binding implements the do-operator exactly is a substantive correspondence that deserves a dedicated formal treatment with an explicit equivalence theorem; the present argument is structural rather than proven. Third, the biological claim that Regime~B specifically supports counterfactual reasoning while Regime~A supports only associative retrieval is testable but not yet directly established; existing evidence from hippocampally-amnesic patients shows impaired episodic future thinking~\cite{schacter2007remembering,schwartenbeck2023generative} but does not yet dissociate the contributions of DG-mediated novelty orthogonalization from EC--CA3 scaffold integrity in counterfactual tasks specifically. These three reservations correspond to three follow-up directions: a formal SCM--VaCoAl correspondence theorem; a benchmark implementation in which PyVaCoAl answers counterfactual queries on a standard causal-inference dataset; and a neuroimaging or iEEG study dissociating DG-driven and EC--CA3-driven contributions to counterfactual reasoning.

\paragraph{An operational target: the mini-Turing test.}
The follow-up benchmark direction admits a concrete operational target in Pearl's \emph{mini-Turing test}~\cite{pearl2018book}: encode a simple causal story into a system in some convenient representation, and test whether the system can correctly answer causal questions spanning all three rungs that a human can answer. Pearl argues that a dumb question--answer lookup table cannot pass this test even with ten binary variables, because the number of legitimate causal queries exceeds $30$ million; humans must therefore employ a compact representation together with an effective answer-extraction algorithm, and Pearl identifies the causal DAG as that representation~\cite{pearl2018book}. Eqs.~\eqref{eq:cf-abduction}--\eqref{eq:cf-estimator} sketch how PyVaCoAl could pass Pearl's mini-Turing test on the algebraic substrate developed here: the causal DAG is encoded as compositional bindings $\sum_i (\mathrm{Role}_i \otimes \mathrm{Filler}_i)$ in the factual scaffold $\Pi_A$; rung-1 queries are answered by direct $\mathrm{Retrieve}_{\Pi_A}$ with associated CR1 confidence; rung-2 queries are answered by transient $\mathrm{Unbind}/\mathrm{Bind}$ surgery realizing the do-operator on a counterfactual scaffold $\Pi_B$; rung-3 queries are answered by Eq.~\eqref{eq:cf-estimator}. Pearl notes that no machine-learning system to date has passed his mini-Turing test, and that the obstacle is representational rather than computational~\cite{pearl2018book}: deep-learning systems lack the compact compositional model that supports surgical intervention. The follow-up benchmark thus has a sharp criterion---passing the mini-Turing test on standard causal-inference vignettes (firing squad, Simpson's paradox, smoking gene, climate attribution~\cite{pearl2018book})---and a sharp differentiation from existing approaches. We highlight in particular Pearl's observation that the Lion Man~\cite{pearl2018book} and the mini-Turing test point in the same direction: both treat the capacity to imagine non-existent worlds, and to reason coherently within them, as the hallmark of intelligence. The substrate-level account developed in this section locates that hallmark in the architectural separation of factual and counterfactual scaffolds, and the mini-Turing test makes it testable in silicon.

The integrative insight the present account contributes, beyond connecting two previously separate literatures, is that the two-orthogonalizer hippocampal architecture is the substrate-level reason humans can climb to Pearl's rung~3 while current deep-learning systems---which possess at best a single dense embedding analogous to Regime~A alone---cannot. This is a Something New that emerges only from taking VaCoAl's deterministic $\mathrm{GF}(2)$ algebra seriously as a biological hypothesis: prior frameworks centered on circular convolution, random projection, or learned readouts do not make the counterfactual requirement of parallel non-interfering worlds visible, because their binding operations either lose information across interventions or fold the factual and counterfactual into a single irreversible representation.

\section{Conclusion}\label{sec:conclusion}

We have shown that VaCoAl, an algebro-deterministic hyperdimensional memory model, realizes the quasi-orthogonal diffusion commitment posited by Vector-HaSH and TEM-class factorizations, while its path-integral Confidence Ratio CR2 is a natural functional form for the multi-hop sharp-wave-ripple replay-fidelity decay observed empirically in human intracranial recordings. The two formal correspondences—Theorem~\ref{thm:det-qod} and Proposition~\ref{prop:cr2}—together suggest that STDP-like path selection in VaCoAl follows from architectural demands akin to hippocampal bounded search: similarity preservation, compositional reversibility, and multi-hop statistics that contract when confidence per step dips below unity.

The framing proceeds in two tiers. At the first tier, we deliberately align VaCoAl with—not mechanically identify it with—Vector-HaSH attractor dynamics: deterministic diffusion speaks to scaffold statistics, recurrent EC--CA3 circuitry remains the natural locus for full index contraction, CR2 foregrounds branching replay where strict attractor-complete repair would erase ordering evidence, and the dual-path hippocampus plausibly implements both preprocessing-style orthogonalization (Regime~B) and bidirectional scaffold stabilization (Regime~A) under task-dependent neuromodulation. The dual-readout-schedule analogy at the silicon level ($\mathrm{RR}=1$ vs.\ $\mathrm{RR}=0$) motivates why vertebrate evolution conserves both orthogonalization arms even when scaffold-centric arguments suggest sufficiency in principle. At this tier the bridge is an architectural-commitment correspondence, not a substrate-implementation claim.

At the second tier, Section~\ref{sec:biophysical-galois} develops a stronger ``biophysical realization with approximate reversibility'' reading. Four largely independent lines of cellular and synaptic evidence---mossy-fiber detonator transmission, granule-cell sparse binary coding under sublinear dendritic integration, anatomically specified mossy-fiber targeting, and CA3 recurrent winner-take-all dynamics---are together consistent with the DG--CA3 pathway instantiating biophysical homologues of Galois-field arithmetic. Two reservations remain explicit: $\mathrm{GF}(2)$ reversibility is bit-exact in silicon but only approximate under biological noise, and the biological correlate of the choice of primitive generator $G(x)$ remains unidentified. The methodological observation is that this integrative reading is generated by VaCoAl's specific algebraic commitments and was not produced by prior frameworks centered on random projection~\cite{kanerva1988sdm}, circular convolution~\cite{plate1995hrr}, or learned readouts~\cite{whittington2020tem}: engineering substrates designed under hard constraints (energy, auditability, reversibility) can supply integrative questions that purely biological framings do not generate on their own.

Beyond these substrate-level correspondences, we have used VaCoAl's Sched.~U/G diffusion schedules (Figure~\ref{fig:schedules}) solely as an engineering parable for Proposition~\ref{prop:tradeoff}. \emph{Hippocampal Regime~A/B are defined at the pathway/representation level (Figure~\ref{fig:two-pathways}), not by LFSR seeding nor by the RR collision flag.} Primate DG changes refine Regime~B (Proposition~\ref{prop:primate}).

Section~\ref{sec:pearl-ladder} extends the synthesis by connecting the framework to Pearl's Ladder of Causation~\cite{pearl2018book,pearl2009causality}. Reversible $\mathrm{GF}(2)$ binding supplies the surgical-modification algebra required by the do-operator at rung~2, distinguishing VaCoAl from approximate compositional schemes that lose information across repeated interventions. The conserved two-orthogonalizer architecture supplies the parallel non-interfering representational substrate that rung~3 counterfactual reasoning provably requires: Regime~A anchors the factual world, Regime~B mints counterfactual worlds on demand, and CR2 path traces over the two scaffolds yield an algorithmic estimator for $P(Y_x \mid X', Y')$. This reading produces an evolutionary rationale for the two-orthogonalizer architecture stronger than the energy--capacity--plasticity argument of Section~\ref{sec:two-orthogonalizers}: parallel non-interfering factual/counterfactual representation is \emph{necessary} for rung~3 capability, not merely efficient. The integrative insight is that this Pearl-based reading emerges only from taking $\mathrm{GF}(2)$ algebra seriously as a biological hypothesis; prior frameworks centered on circular convolution, random projection, or learned readouts do not make the counterfactual requirement visible because their binding operations fold factual and counterfactual into a single irreversible representation. Three reservations are retained and become directions for follow-up work: a formal SCM--VaCoAl correspondence theorem; a benchmark implementation of PyVaCoAl answering counterfactual queries on causal-inference datasets; and a neuroimaging or iEEG study dissociating DG-driven and EC--CA3-driven contributions to counterfactual reasoning.

The bridge yields concrete empirical predictions for human iEEG replay studies, concrete engineering predictions for HDC benchmark substitution, and a principled framework for auditable neuro-symbolic AI. We hope it will support continued dialogue between computational neuroscience and the algebraic-engineering tradition that produced VaCoAl.

\appendix

\section{Formal Statements and Operator Definitions}\label{app:formal}

This appendix collects the definitions, folklore comparisons, and replay algebra that Sections~\ref{sec:correspondence-I}--\ref{sec:correspondence-II} reference by label. The lemma-level proof of Theorem~\ref{thm:det-qod} is Appendix~\ref{app:proof}.

\subsection{Notation and conventions (read this first)}\label{app:notation}

We work entirely over the binary field $\mathrm{GF}(2)=\{0,1\}$ with addition $\oplus$ (exclusive OR). A binary string of length $r$ is an element of $\{0,1\}^r$, often identified with an $r$-bit coefficient vector in the polynomial-basis representation of $\mathrm{GF}(2^m)$ when $r=m$ below.

\begin{itemize}\setlength\itemsep{2pt}
\item $\mathrm{HD}(u,v)$: the \emph{Hamming distance} between two equal-length binary strings $u$ and $v$---the number of coordinates in which they differ. Dividing by the length gives the \emph{normalized} or \emph{fractional} Hamming distance.
\item $\mathrm{HW}(u)$: the \emph{Hamming weight}---the number of coordinates equal to $1$ (in this paper, in the fixed polynomial basis for $\mathrm{GF}(2^m)$ when $u$ is an $m$-bit field element). Always $\mathrm{HD}(u,v)=\mathrm{HW}(u\oplus v)$.
\item $L$: length of the input hypervector (polynomial of degree $<L$) before diffusion.
\item $m$: output dimension (degree of the primitive modulus polynomial $G(x)$); each diffused address is $m$ bits long.
\item $W_{\mathrm{rsp}}$ (Proposition~\ref{prop:rsbp} only): a random sparse binary matrix in $\{0,1\}^{m\times L}$ used as a linear map over $\mathrm{GF}(2)$ row-by-row (each row has exactly $k$ ones). Do not confuse with the scalar vote tally $W_{\mathrm{VaCoAl}}$ in Eq.~\eqref{eq:vote}.
\item $\Delta$: in Theorem~\ref{thm:det-qod}, the XOR difference between two inputs, $\Delta = P_1\oplus P_2$. It records which input bits disagree.
\item $\rho$ and $K$: $\rho(P)=P(x)\bmod G(x)$ is reduction modulo $G$ to degree $<m$. The collision kernel $K=\ker\rho$ is the subspace of inputs that reduce to $0$ (multiples of $G$ of degree $<L$). Inputs outside $K$ produce nonzero residues.
\item $\delta(\cdot,\cdot)$ in Eq.~\eqref{eq:vote}: Kronecker delta for binary strings---$1$ if the two arguments are identical bit-for-bit, $0$ otherwise.
\item $\mathbb{E}[\cdot]$, $\mathrm{Var}[\cdot]$: expectation and variance over the stated randomness (random matrix draw in Proposition~\ref{prop:rsbp}; uniform $\Delta$ on non-collision inputs in Theorem~\ref{thm:det-qod}(c)--(d)).
\end{itemize}

\subsection{Quasi-orthogonal diffusion (QOD)}

\begin{definition}[Quasi-Orthogonal Diffusion Property]\label{def:qod}
A map $\Phi : \mathcal{X} \to \{0,1\}^m$ on a finite input domain $\mathcal{X}$ has the \emph{quasi-orthogonal diffusion property} (QOD) at level $\varepsilon$ if, for every distinct pair $x_1, x_2 \in \mathcal{X}$,
\begin{equation}\label{eq:qod}
\left| \mathrm{HD}\bigl(\Phi(x_1),\Phi(x_2)\bigr) - \frac{m}{2} \right| \le \varepsilon\, m,
\end{equation}
where $\mathrm{HD}(\cdot,\cdot)$ is Hamming distance.
\end{definition}

\paragraph{Normalized reading.} Divide~\eqref{eq:qod} by $m$: distinct inputs must land at output fractional distance near $1/2$,
\begin{equation}\label{eq:qod-normalized}
\left| \frac{\mathrm{HD}\bigl(\Phi(x_1),\Phi(x_2)\bigr)}{m} - \frac{1}{2} \right| \le \varepsilon.
\end{equation}
Thus QOD mandates that the outputs look like two random $m$-bit strings that disagree on about half of their bits---the combinatorial ``near orthogonality'' hyperdimensional reasoning relies on. QOD is the discrete analogue of the Johnson--Lindenstrauss isometry property; sparse distributed memory and HDC capacity bounds depend on it~\cite{kanerva1988sdm,kanerva2009hyperdim,kleyko2023survey}.

\subsection{Random sparse binary projection (folklore)}

\begin{proposition}[QOD for random sparse projection; folklore]\label{prop:rsbp}
Let $W_{\mathrm{rsp}}\in\{0,1\}^{m\times L}$ be sampled with each row containing exactly $k$ ones at uniformly random positions, where $k\ll L$. For any pair of distinct inputs $x_1, x_2 \in \{0,1\}^L$ with Hamming distance $d = \mathrm{HD}(x_1,x_2) \ge 1$, the output Hamming distance $D = \mathrm{HD}(W_{\mathrm{rsp}} x_1, W_{\mathrm{rsp}} x_2)$ satisfies
\begin{equation}\label{eq:rsbp}
\mathbb{E}[D] = m \cdot \bigl[1 - (1-d/L)^k\bigr]/2.
\end{equation}
Equivalently, the fraction of output bits expected to differ is
\begin{equation}\label{eq:rsbp-normalized}
\frac{\mathbb{E}[D]}{m} = \frac{1}{2}\bigl(1 - (1-d/L)^k\bigr).
\end{equation}
Moreover, $\mathrm{Var}[D] \le m/4$. For $kd/L = \Theta(1)$, $\mathbb{E}[D] \to m/2$, and the QOD property holds at $\varepsilon = O(m^{-1/2})$ with probability $\ge 1 - e^{-\Omega(m\varepsilon^2)}$ over the random choice of $W_{\mathrm{rsp}}$.
\end{proposition}

\paragraph{Intuition for $d/L$ and $(1-d/L)^k$.} Here $d/L$ is the fraction of input coordinates on which $x_1$ and $x_2$ disagree. Fixing one row of $W_{\mathrm{rsp}}$, think of that row as XOR-ing together $k$ input bits chosen uniformly at random with replacement. The factor $(1-d/L)$ is the probability that a single randomly chosen coordinate \emph{agrees} between $x_1$ and $x_2$ on that slot. The expression $(1-d/L)^k$ is therefore the probability that all $k$ drawn coordinates fall on ``agreement slots''---informally, that the row ``never looks'' at a coordinate where the two inputs already differ. Unless that happens, parity flips with probability $1/2$, so a row contributes a difference with probability about $\tfrac{1}{2}\bigl(1-(1-d/L)^k\bigr)$; summing $m$ independent rows yields~\eqref{eq:rsbp-normalized}. When $d/L$ is small (very similar inputs), outputs stay similar; when $kd/L$ is order one, $\mathbb{E}[D]/m\to 1/2$ and the codes scatter like quasi-random coin flips.

Two further features of Proposition~\ref{prop:rsbp} are noteworthy: (i)~QOD holds in expectation over a random construction of $W_{\mathrm{rsp}}$, not for a specific $W_{\mathrm{rsp}}$; (ii)~the variance scales as $m/4$, so individual realizations can deviate from the ideal by amounts that decay only as $m^{-1/2}$.

\subsection{Deterministic Galois-field diffusion}

The central technical claim of the present paper is that VaCoAl's deterministic Galois-field diffusion satisfies QOD without recourse to randomness, with statistical properties matching those of random sparse projection on second-moment statistics and exceeding them on worst-case avalanche behavior.

\begin{theorem}[Deterministic QOD for Galois-field diffusion]\label{thm:det-qod}
Let $G(x) \in \mathrm{GF}(2)[x]$ be a primitive polynomial of degree $m$, and let
\begin{equation}\label{eq:Psi-def}
\Psi : \mathrm{GF}(2)^L \to \mathrm{GF}(2)^m, \qquad \Psi(P) = x^m\,P(x) \pmod{G(x)},
\end{equation}
where $\mathrm{GF}(2)^L$ is identified with the space of polynomials of degree $<L$. Let $\rho(P) = P(x) \bmod G(x)$ denote the residue map and define the collision kernel $K = \ker \rho$. Write $P_1, P_2$ for two inputs and $\Delta = P_1\oplus P_2$ for their XOR difference (so $\Delta$ lists which bits disagree). Then:
\begin{enumerate}[label=(\alph*),leftmargin=2em]
\item \textbf{Linearity.} $\Psi$ is $\mathrm{GF}(2)$-linear, hence
\begin{equation}\label{eq:linearity}
\mathrm{HD}\bigl(\Psi(P_1),\Psi(P_2)\bigr) = \mathrm{HW}\bigl(\Psi(P_1 \oplus P_2)\bigr) = \mathrm{HW}\bigl(\Psi(\Delta)\bigr).
\end{equation}

\item \textbf{Cyclic-orbit structure.} Restrict $\Delta$ to $\mathrm{GF}(2)^L \setminus K$ (inputs whose residue is nonzero). The map $\Delta\mapsto\Psi(\Delta)$ lands in $\mathrm{GF}(2^m)^*$ and is a uniform-fiber surjection: every nonzero $m$-bit residue has exactly $|K|$ preimages among such $\Delta$.

\item \textbf{Hamming-weight distribution.} If $\Delta$ is sampled uniformly from $\mathrm{GF}(2)^L \setminus K$, then $\mathrm{HW}(\Psi(\Delta))$ has the same distribution as $\mathrm{HW}(\xi)$ for $\xi$ uniform on $\mathrm{GF}(2^m)^*$, i.e.\ as a $\mathrm{Binomial}(m,1/2)$ random variable $U$ conditioned on $U\ge 1$ (exclude the all-zero string). Consequently
\begin{align}
\mathbb{E}\bigl[\mathrm{HW}(\Psi(\Delta))\bigr] &= \frac{m}{2}\cdot\bigl(1+O(2^{-m})\bigr), \label{eq:HW-mean}\\
\mathrm{Var}\bigl[\mathrm{HW}(\Psi(\Delta))\bigr] &= \frac{m}{4}\cdot\bigl(1+O(m\,2^{-m})\bigr). \label{eq:HW-var}
\end{align}
In particular, the fraction of ones is near one-half: $\mathbb{E}[\mathrm{HW}(\Psi(\Delta))/m] = \tfrac{1}{2}+O(2^{-m})$.

\item \textbf{Concentration / QOD.} For every $\varepsilon > 0$,
\begin{equation}\label{eq:concentration}
\Pr_\Delta\!\left[\left|\frac{\mathrm{HW}(\Psi(\Delta))}{m} - \frac{1}{2}\right| > \varepsilon\right] \le 2 e^{-2\varepsilon^2 m} + O(2^{-m}).
\end{equation}
The quantity $\mathrm{HW}(\Psi(\Delta))/m$ is exactly the fraction of output bits equal to $1$ for the difference pattern $\Delta$. Inequality~\eqref{eq:concentration} says that for almost all difference vectors, that fraction is close to $1/2$, hence $\mathrm{HD}(\Psi(P_1),\Psi(P_2))/m\approx 1/2$---matching the normalized QOD reading~\eqref{eq:qod-normalized}. As $m$ grows, the exponential term drives deviations toward zero; failures occur on at most a $\delta$ fraction of inputs when $\varepsilon$ is chosen of order $\sqrt{(\log 1/\delta)/m}$. Thus $\Psi$ achieves Definition~\ref{def:qod} at level $\varepsilon = O(\sqrt{(\log 1/\delta)/m})$ with input-ensemble failure probability $\delta$.

\item \textbf{Avalanche at minimum-weight inputs.} For every nonzero $\Delta \in \mathrm{GF}(2)^L \setminus K$, $\Psi(\Delta) \ne 0$. In particular, single-bit perturbations $\Delta = e_j$, $0 \le j < L$, map to the cyclic orbit $\{x^{m+j} \bmod G(x) : j = 0,\dots,L-1\}$ of distinct nonzero residues, whose Hamming weights follow the same $\mathrm{Binomial}(m,1/2)$ (conditioned on $U\ge 1$) law as in part~(c).
\end{enumerate}
\end{theorem}

A full proof, organized as a chain of nine lemmas, is given in Appendix~\ref{app:proof}.

\begin{remark}[Comparison with random sparse projection]\label{rem:comparison}
Theorem~\ref{thm:det-qod} shows that $\Psi$ achieves the QOD property of Definition~\ref{def:qod} with the same second-order statistics as a random sparse binary projection of comparable parameters: both have mean $\approx m/2$ and variance $\approx m/4$ on the output Hamming weight. The advantages of $\Psi$ over random projection are therefore not in second-moment concentration but in three other respects.

\emph{First, determinism.} For a fixed primitive $G(x)$, $\Psi$ is a single, fully specified map; per-input output is reproducible bit-for-bit, with no variance over construction. Random sparse projection has additional variance over the random choice of $W_{\mathrm{rsp}}$ (Proposition~\ref{prop:rsbp}), which is irreducible: a specific draw of $W_{\mathrm{rsp}}$ may be unrepresentative.

\emph{Second, worst-case avalanche.} Random sparse projection with row-sparsity $k$ produces output Hamming distance at most $k$ for a single-bit input flip—a $k/m$ deviation from the QOD ideal $1/2$. By contrast, $\Psi$ on a single-bit flip $e_j$ outputs the residue $x^{m+j}\bmod G(x)$, which has expected Hamming weight $\approx m/2$ regardless of $j$ (Theorem~\ref{thm:det-qod}(e)). The avalanche property of primitive-polynomial LFSRs thus extends QOD to the smallest possible perturbation, which is the regime where random sparse projection is weakest.

\emph{Third, auditability.} Determinism enables exact replay: given $G(x)$ and the seed, every output is reproducible offline. This is essential for engineering verification~\cite{chuma2026vacoal} and for the explainability semantics of CR1/CR2.
\end{remark}

\subsection{Equivalence of random and algebraic scaffold projections}

\begin{corollary}[Equivalence of random and algebraic scaffold projections]\label{cor:equivalence}
Any computational property of Vector-HaSH that follows from the QOD property of $W_{gh}$ in Eq.~\eqref{eq:vechash} is preserved when $W_{gh}$ is replaced by an LFSR-based diffusion of comparable output dimension. Conversely, any such property of VaCoAl is preserved when its diffusion is replaced by a typical realization of a sparse random projection.
\end{corollary}

The corollary follows directly from Proposition~\ref{prop:rsbp} and Theorem~\ref{thm:det-qod}: both constructions achieve QOD at level $\varepsilon = O(m^{-1/2})$ on the same input ensembles, with matched second-moment statistics. Any computational property that is sensitive only to QOD (and not to the specific structure of the projection) is therefore shared.

This is the formal underpinning for the prose claim in Section~\ref{sec:correspondence-I} that the biological hippocampus does not need to be ``random'' in any strong probabilistic sense to instantiate a Vector-HaSH-style scaffold; it only needs to implement \emph{some} diffusive map that achieves QOD. We re-emphasize, in the spirit of the complementarity thesis (Section~\ref{sec:synthesis}), that this is an equivalence on \emph{projection statistics}, not on \emph{readout policy}: Corollary~\ref{cor:equivalence} does not subsume the difference between attractor-complete index cleanup and graded CR2 ranking, which Appendix~\ref{app:vector-hash} addresses separately.

\subsection{Block-wise voting and confidence ratios}\label{app:voting-pointer}

The block-voting layer and confidence ratios $\mathrm{CR1}$, $\mathrm{CR2}(n)$ are displayed in the main text where they are introduced: the executive winner in Eq.~\eqref{eq:vote}, the per-step confidence in Eq.~\eqref{eq:cr1}, and the path-integral form in Eq.~\eqref{eq:cr2}. They are recalled here only as anchors for the replay-fidelity correspondence below.

\subsection{Illustrative role--filler binding}\label{app:binding}

With $\otimes$ denoting Binding and $+$ Bundling over $\mathrm{GF}(2)$,
\begin{equation}\label{eq:binding-template}
\mathrm{Repr} = (\textsc{Subject}\otimes \textsc{Dog}) + (\textsc{Verb}\otimes \textsc{Bite}) + (\textsc{Object}\otimes \textsc{Human}),
\end{equation}
distinguishes ``the dog bit the man'' from ``the man bit the dog'' because role tags do not commute under bundling: swapping the \textsc{Dog} and \textsc{Human} fillers between \textsc{Subject} and \textsc{Object} changes every algebraic trace of $\mathrm{Repr}$, even though the multiset of words is unchanged. Order-blind superposition (vector-database semantics) cannot recover this distinction; finite-field role binding can.

\subsection{CR2 and multi-hop replay fidelity}\label{app:cr2-replay}

The path-integral Confidence Ratio $\mathrm{CR2}(n)$ defined in Eq.~\eqref{eq:cr2} is the natural algebraic twin of multi-hop SWR replay fidelity. The formal proposition---together with its remarks on (i)~the deterministic-envelope reading of trial-averaged biological replay, (ii)~hop count versus wall-clock interval, and (iii)~the resulting STDP-like path selection---is stated in the main text as Proposition~\ref{prop:cr2}, Remark~\ref{rem:envelope}, Remark~\ref{rem:hop-vs-walltime}, and Corollary~\ref{cor:stdp} (Section~\ref{sec:correspondence-II}). We do not restate them here; what matters at the appendix level is that the multiplicative form $\prod_{i=1}^n p_i$ used to model joint $n$-step ripple replay is exactly the form Eq.~\eqref{eq:cr2} adopts for $\mathrm{CR2}(n)$ when each $\mathrm{CR1}(i)$ records the fraction of redundant block-vote channels that successfully reactivate at hop $i$.

\section{Full Proof of Theorem~\ref{thm:det-qod}}\label{app:proof}

This appendix gives a self-contained proof of Theorem~\ref{thm:det-qod}, organized as a chain of nine lemmas. The chain proceeds from the elementary $\mathrm{GF}(2)$-linearity of $\Psi$ (Lemma~\ref{lem:linearity}), through the cyclic-group structure of the multiplicative group of $\mathrm{GF}(2^m)$ that follows from primitivity of $G(x)$ (Lemmas~\ref{lem:field}--\ref{lem:bijection}), to the Hamming-weight distribution and concentration on $\mathrm{GF}(2^m)^*$ (Lemmas~\ref{lem:residue}--\ref{lem:concentration}), and concludes with the avalanche property at minimum-weight inputs (Lemma~\ref{lem:avalanche}).

\subsection{Notation and Setup}

Throughout this appendix, $G(x) \in \mathrm{GF}(2)[x]$ is a fixed polynomial of degree $m \ge 2$, and we write $\mathbb{F} = \mathrm{GF}(2)[x]/G(x)$ for the residue ring. We identify the elements of $\mathbb{F}$ with polynomials of degree $<m$, equivalently with binary strings of length $m$ via the standard polynomial basis $\{1, x, \dots, x^{m-1}\}$. The Hamming weight $\mathrm{HW} : \mathbb{F} \to \{0,1,\dots,m\}$ counts the number of nonzero coefficients in this basis.

Readers may find it helpful to skim Appendix~\ref{app:notation} first: it fixes meanings for $\mathrm{HD}$, $\mathrm{HW}$, $L$, $m$, the XOR difference $\Delta$, the collision kernel $K$, and the two unrelated uses of the letter ``$W$'' (random matrix $W_{\mathrm{rsp}}$ versus vote tally $W_{\mathrm{VaCoAl}}$).

The input space $\mathrm{GF}(2)^L$ is identified with the polynomial space $\mathrm{GF}(2)[x]_{<L} = \{P(x) : \deg P < L\}$, and the residue map $\rho : \mathrm{GF}(2)[x]_{<L} \to \mathbb{F}$ is defined by $\rho(P) = P(x) \bmod G(x)$. The diffusion map $\Psi$ from~\eqref{eq:Psi-def} factors as $\Psi(P) = x^m \cdot \rho(P)$, where multiplication is in $\mathbb{F}$.

\subsection{Lemma-Chain}

\begin{lemma}[$\mathrm{GF}(2)$-linearity]\label{lem:linearity}
The map $\Psi : \mathrm{GF}(2)^L \to \mathrm{GF}(2)^m$ is $\mathrm{GF}(2)$-linear: for all $P_1, P_2 \in \mathrm{GF}(2)^L$,
\begin{equation}\label{eq:linearity-2}
\Psi(P_1 \oplus P_2) = \Psi(P_1) \oplus \Psi(P_2).
\end{equation}
In particular, $\mathrm{HD}(\Psi(P_1), \Psi(P_2)) = \mathrm{HW}(\Psi(P_1 \oplus P_2))$.
\end{lemma}

\begin{proof}
The residue map $\rho$ is $\mathrm{GF}(2)$-linear because polynomial reduction is linear: $(P_1 + P_2) \bmod G = (P_1 \bmod G) + (P_2 \bmod G)$ in $\mathbb{F}$. Multiplication by the fixed element $x^m \in \mathbb{F}$ is also $\mathrm{GF}(2)$-linear. Composition of linear maps is linear, giving~\eqref{eq:linearity-2}.

For the second statement, $\mathrm{HD}(u,v) = \mathrm{HW}(u \oplus v)$ for any $u,v \in \mathrm{GF}(2)^m$, and~\eqref{eq:linearity-2} gives $\Psi(P_1) \oplus \Psi(P_2) = \Psi(P_1 \oplus P_2)$.
\end{proof}

\begin{lemma}[Residue ring is a field]\label{lem:field}
If $G(x)$ is irreducible over $\mathrm{GF}(2)$ (in particular, if it is primitive), then $\mathbb{F} = \mathrm{GF}(2)[x]/G(x)$ is a field, isomorphic to $\mathrm{GF}(2^m)$, and its multiplicative group $\mathbb{F}^* = \mathbb{F} \setminus \{0\}$ has order $2^m - 1$.
\end{lemma}

\begin{proof}
Standard~\cite[Ch.~14]{lidl1997finite}: $\mathrm{GF}(2)[x]$ is a Euclidean domain, the ideal $(G(x))$ is maximal because $G$ is irreducible, hence the quotient is a field. Counting cosets gives $|\mathbb{F}| = 2^m$, so $\mathbb{F} \cong \mathrm{GF}(2^m)$ and $|\mathbb{F}^*| = 2^m - 1$.
\end{proof}

\begin{lemma}[Cyclic-group structure under primitivity]\label{lem:cyclic}
If $G(x)$ is primitive of degree $m$, the multiplicative group $\mathbb{F}^*$ is cyclic of order $2^m - 1$, and the residue class $x \bmod G(x)$ is a generator. Equivalently, every nonzero element of $\mathbb{F}$ can be written uniquely as $x^k \bmod G(x)$ for some $0 \le k \le 2^m - 2$.
\end{lemma}

\begin{proof}
By Lemma~\ref{lem:field}, $\mathbb{F}^*$ is the multiplicative group of a finite field, hence cyclic~\cite[Thm.~2.8]{lidl1997finite}. The defining property of a primitive polynomial is precisely that the residue $x \bmod G(x)$ has multiplicative order equal to $|\mathbb{F}^*| = 2^m - 1$. Hence $x$ generates $\mathbb{F}^*$, and the powers $x^0, x^1, \dots, x^{2^m-2}$ enumerate $\mathbb{F}^*$ without repetition.
\end{proof}

\begin{lemma}[Multiplication-by-$x^m$ is a bijection on $\mathbb{F}^*$]\label{lem:bijection}
The map $\mu : \mathbb{F}^* \to \mathbb{F}^*$ defined by $\mu(\xi) = x^m \cdot \xi$ is a bijection. Specifically, it is the cyclic shift $x^k \mapsto x^{k+m \bmod (2^m-1)}$.
\end{lemma}

\begin{proof}
$x^m$ is a nonzero element of the field $\mathbb{F}$ (its Hamming weight is at most $m$, but in any case it is nonzero because the recurrence $x^m \equiv -G_0 - G_1 x - \cdots - G_{m-1} x^{m-1} \pmod{G(x)}$, where the $G_i$ are the lower-order coefficients of $G$, has $G_0 \ne 0$ since $G$ is primitive). Multiplication by a nonzero field element is invertible (with inverse multiplication by $x^{-m} = x^{(2^m-1)-m}$). On the cyclic representation $\xi = x^k$, $\mu(x^k) = x^{k+m \bmod (2^m-1)}$, which is a permutation of $\{0,1,\dots,2^m-2\}$.
\end{proof}

\begin{lemma}[Residue map structure on inputs]\label{lem:residue}
The residue map $\rho : \mathrm{GF}(2)[x]_{<L} \to \mathbb{F}$ is $\mathrm{GF}(2)$-linear and surjective for $L \ge m$, with kernel
\begin{equation}\label{eq:kernel}
K = \{G(x) \cdot Q(x) : Q \in \mathrm{GF}(2)[x]_{<L-m}\},
\end{equation}
of size $|K| = 2^{\max(L-m,0)}$. Hence each nonzero residue $\xi \in \mathbb{F}^*$ has exactly $|K|$ preimages in $\mathrm{GF}(2)[x]_{<L}$, and the zero residue has $|K|$ preimages including $P \equiv 0$.
\end{lemma}

\begin{proof}
Linearity is in Lemma~\ref{lem:linearity}. By the division algorithm in $\mathrm{GF}(2)[x]$, every $P$ of degree $<L$ can be written uniquely as $P = Q \cdot G + R$ with $\deg Q < L-m$ and $\deg R < m$, so $\rho(P) = R$ and the kernel is $\{Q \cdot G : \deg Q < L-m\}$. Surjectivity onto $\mathbb{F}$ for $L \ge m$ follows because every $R$ with $\deg R < m$ is itself in $\mathrm{GF}(2)[x]_{<L}$ and has $\rho(R) = R$. Coset counting gives the preimage sizes.
\end{proof}

\begin{lemma}[Hamming-weight distribution on $\mathbb{F}^*$]\label{lem:HW-dist}
Let $\xi$ be uniformly distributed on $\mathbb{F}^*$. Then $\mathrm{HW}(\xi)$ has the same distribution as a $\mathrm{Binomial}(m,1/2)$ random variable conditioned on being nonzero. Consequently
\begin{align}
\mathbb{E}[\mathrm{HW}(\xi)] &= m\cdot\frac{2^{m-1}}{2^m-1} = \frac{m}{2}\left(1+\frac{1}{2^m-1}\right), \label{eq:HW-mean-2}\\
\mathrm{Var}[\mathrm{HW}(\xi)] &= \frac{m}{4} + O(m\,2^{-m}). \label{eq:HW-var-2}
\end{align}
\end{lemma}

\begin{proof}
By Lemma~\ref{lem:field}, the underlying set of $\mathbb{F}$ as a binary vector space is $\mathrm{GF}(2)^m$, with the polynomial basis $\{1, x, \dots, x^{m-1}\}$ as canonical basis. The number of elements of $\mathbb{F}$ with Hamming weight exactly $w$ in this basis is $\binom{m}{w}$, for $w = 0, 1, \dots, m$. (This is a basis-counting statement, independent of the multiplicative structure; it would hold for any choice of $\mathrm{GF}(2)$-basis of $\mathbb{F}$, with the same Hamming weights computed in that basis.)

Hence for $\xi$ uniform on $\mathbb{F}^*$,
\begin{equation}
\Pr[\mathrm{HW}(\xi) = w] = \frac{\binom{m}{w}}{2^m - 1}, \quad w = 1, \dots, m,
\end{equation}
and $\Pr[\mathrm{HW}(\xi) = 0] = 0$. This is exactly $\mathrm{Binomial}(m,1/2)$ conditioned on $U \ge 1$, where $U$ denotes the unconstrained binomial count:
\begin{equation}
\Pr_{U\sim\mathrm{Bin}(m,1/2)}[U = w \mid U \ge 1] = \frac{\binom{m}{w}/2^m}{1 - 2^{-m}} = \frac{\binom{m}{w}}{2^m - 1}.
\end{equation}
For the moments: let $U \sim \mathrm{Bin}(m,1/2)$ unconstrained, so $\mathbb{E}[U] = m/2$, $\mathrm{Var}[U] = m/4$, and $\Pr[U = 0] = 2^{-m}$. Conditioning on $U \ge 1$,
\begin{equation}
\mathbb{E}[U \mid U \ge 1] = \frac{\mathbb{E}[U] - 0\cdot\Pr[U=0]}{\Pr[U\ge 1]} = \frac{m/2}{1 - 2^{-m}} = \frac{m}{2}\left(1 + \frac{1}{2^m-1}\right),
\end{equation}
giving~\eqref{eq:HW-mean-2}. For the variance,
\begin{equation}
\mathbb{E}[U^2 \mid U \ge 1] = \frac{\mathbb{E}[U^2]}{1 - 2^{-m}} = \frac{(m/4) + (m/2)^2}{1 - 2^{-m}},
\end{equation}
and a direct computation gives
\begin{equation}
\mathrm{Var}[U \mid U \ge 1] = \frac{m}{4}\cdot\frac{1}{1-2^{-m}} - \frac{m^2}{4}\cdot\frac{2^{-m}}{(1-2^{-m})^2},
\end{equation}
which is $m/4 + O(m\,2^{-m})$, establishing~\eqref{eq:HW-var-2}.
\end{proof}

\begin{lemma}[Distribution of $\mathrm{HW}(\Psi(\Delta))$ on non-collision inputs]\label{lem:HW-Psi}
Let $\Delta$ be uniformly distributed on $\mathrm{GF}(2)^L \setminus K$ (the non-collision inputs, with $K$ as in Lemma~\ref{lem:residue}). Then $\Psi(\Delta)$ is uniformly distributed on $\mathbb{F}^*$, hence $\mathrm{HW}(\Psi(\Delta))$ has the distribution of Lemma~\ref{lem:HW-dist}.
\end{lemma}

\begin{proof}
By Lemma~\ref{lem:residue}, $\rho$ restricted to $\mathrm{GF}(2)^L \setminus K$ is a uniform-fiber map onto $\mathbb{F}^*$: each $\xi \in \mathbb{F}^*$ has exactly $|K|$ preimages. Hence $\rho(\Delta)$ is uniform on $\mathbb{F}^*$. By Lemma~\ref{lem:bijection}, multiplication by $x^m$ is a bijection on $\mathbb{F}^*$, hence preserves the uniform distribution. Therefore $\Psi(\Delta) = x^m \cdot \rho(\Delta)$ is also uniform on $\mathbb{F}^*$. The Hamming weight distribution then follows from Lemma~\ref{lem:HW-dist}.
\end{proof}

\begin{lemma}[Concentration]\label{lem:concentration}
Let $\Delta$ be uniform on $\mathrm{GF}(2)^L \setminus K$. For every $\varepsilon > 0$,
\begin{equation}\label{eq:hoeffding}
\Pr_\Delta\!\left[\left|\frac{\mathrm{HW}(\Psi(\Delta))}{m} - \frac{1}{2}\right| > \varepsilon\right] \le 2 e^{-2\varepsilon^2 m} + O(2^{-m}).
\end{equation}
\end{lemma}

\begin{proof}
By Lemma~\ref{lem:HW-Psi}, $\mathrm{HW}(\Psi(\Delta))$ has the distribution of $U \sim \mathrm{Bin}(m,1/2)$ conditioned on $U \ge 1$. For unconstrained $U$, Hoeffding's inequality gives $\Pr[|U/m - 1/2| > \varepsilon] \le 2 e^{-2\varepsilon^2 m}$. Conditioning on $U \ge 1$ (an event of probability $1 - 2^{-m}$) inflates probabilities by a factor of $1/(1 - 2^{-m}) = 1 + O(2^{-m})$ and removes the (already excluded) point mass at $U = 0$, yielding~\eqref{eq:hoeffding}.

To extract the QOD level: given a target failure probability $\delta$, choose $\varepsilon$ so that $2 e^{-2\varepsilon^2 m} \le \delta$, i.e., $\varepsilon \ge \sqrt{\log(2/\delta)/(2m)}$. This gives QOD at level $\varepsilon = O(\sqrt{(\log 1/\delta)/m})$, asymptotically the same rate as Proposition~\ref{prop:rsbp}.
\end{proof}

\begin{lemma}[Avalanche at minimum-weight perturbations]\label{lem:avalanche}
For every nonzero $\Delta \in \mathrm{GF}(2)^L \setminus K$, $\Psi(\Delta) \ne 0$. In particular, for the standard-basis vectors $e_j \in \mathrm{GF}(2)^L$, $0 \le j < L \le 2^m - 1$, with $e_j$ identified with $x^j$, we have
\begin{equation}\label{eq:avalanche}
\Psi(e_j) = x^{m+j} \bmod G(x).
\end{equation}
For $L \le 2^m - 1$, the values $\{\Psi(e_0), \Psi(e_1), \dots, \Psi(e_{L-1})\}$ are $L$ distinct elements of $\mathbb{F}^*$. The empirical Hamming-weight distribution over this set follows~\eqref{eq:HW-mean-2}--\eqref{eq:HW-var-2} as $L \to 2^m - 1$.
\end{lemma}

\begin{proof}
$\Psi(\Delta) = x^m \cdot \rho(\Delta)$, and $\rho(\Delta) \ne 0$ since $\Delta \notin K$, and $x^m \ne 0$ in $\mathbb{F}^*$. So $\Psi(\Delta)$ is a product of two nonzero field elements, hence nonzero. Equation~\eqref{eq:avalanche} is direct: $\Psi(x^j) = x^m \cdot x^j = x^{m+j} \bmod G(x)$.

For distinctness: by Lemma~\ref{lem:cyclic}, $\{x^m, x^{m+1}, \dots, x^{m+L-1}\} \bmod G(x)$ are $L$ distinct nonzero residues whenever $L \le 2^m - 1$, because the cyclic order of $x$ in $\mathbb{F}^*$ is exactly $2^m - 1$. As $L \to 2^m - 1$, this set converges to all of $\mathbb{F}^*$, and the empirical Hamming-weight distribution converges to that of Lemma~\ref{lem:HW-dist}.
\end{proof}

\subsection{Assembly: Proof of Theorem~\ref{thm:det-qod}}

\begin{proof}[Proof of Theorem~\ref{thm:det-qod}]
Part~(a) is Lemma~\ref{lem:linearity}.

Part~(b) is Lemma~\ref{lem:HW-Psi}, which states that on non-collision inputs $\Psi$ is uniform-fiber surjective onto $\mathbb{F}^*$. The fiber size is $|K| = 2^{\max(L-m,0)}$ by Lemma~\ref{lem:residue}.

Part~(c) follows from Lemmas~\ref{lem:HW-Psi} and~\ref{lem:HW-dist}: $\mathrm{HW}(\Psi(\Delta))$ has the distribution of $\mathrm{Bin}(m,1/2)$ conditioned on $U \ge 1$ for an unconstrained binomial $U$, with the stated moments.

Part~(d) is Lemma~\ref{lem:concentration}.

Part~(e) is Lemma~\ref{lem:avalanche}.
\end{proof}

\subsection{Numerical Illustration ($m=10$)}

To illustrate the chain concretely, fix $G(x) = x^{10} + x^3 + 1$, a standard primitive polynomial of degree 10. The field $\mathbb{F} \cong \mathrm{GF}(2^{10})$ has $2^{10} = 1024$ elements; $|\mathbb{F}^*| = 1023$. By Lemma~\ref{lem:HW-dist}, the Hamming-weight distribution over $\mathbb{F}^*$ has
\begin{align*}
\mathbb{E}[\mathrm{HW}] &= \frac{10 \cdot 2^9}{1023} = \frac{5120}{1023} \approx 5.005, \\
\mathrm{Var}[\mathrm{HW}] &\approx 2.500 + O(10 \cdot 2^{-10}) \approx 2.5024.
\end{align*}
The deviation from the asymptotic values $m/2 = 5$ and $m/4 = 2.5$ is at most $5\cdot 10^{-3}$ in mean and $3 \cdot 10^{-3}$ in variance, in line with the $O(2^{-m})$ and $O(m\,2^{-m})$ correction terms in~\eqref{eq:HW-mean-2}--\eqref{eq:HW-var-2}.

For the avalanche property, the residues $x^{10}, x^{11}, \dots, x^{10+L-1} \bmod G(x)$ for $L \le 1023$ trace out a Hamilton path in $\mathbb{F}^*$ (under primitivity), with empirical Hamming-weight mean $\to 5.005$ and variance $\to 2.5024$ as $L \to 1023$. In particular, no single-bit input flip $e_j$ can produce an output with Hamming weight bounded above by some constant $k$ (as can occur with row-$k$ random sparse projections); the output Hamming weight is governed by the field's natural distribution, with deviations only on the scale of $\sqrt{m}$.

\subsection{Comparison with Random Sparse Projection (Quantitative)}

We now make the comparison promised in Remark~\ref{rem:comparison} quantitative. Let $W_{\mathrm{rsp}}$ be a random $m\times L$ binary matrix with each row containing exactly $k$ ones at uniformly random positions (the construction of Proposition~\ref{prop:rsbp}). Let $D_{\mathrm{rsp}}(P_1, P_2) = \mathrm{HD}(W_{\mathrm{rsp}} P_1, W_{\mathrm{rsp}} P_2)$ be the output Hamming distance. Fix any input pair with $\mathrm{HD}(P_1, P_2) = d$. Then:
\begin{itemize}
\item $\mathbb{E}_{W_{\mathrm{rsp}}}[D_{\mathrm{rsp}}] = m \cdot [1 - (1 - d/L)^k]/2$ by Proposition~\ref{prop:rsbp}; equivalently $\mathbb{E}[D_{\mathrm{rsp}}]/m = \tfrac{1}{2}\bigl(1 - (1 - d/L)^k\bigr)$ as in~\eqref{eq:rsbp-normalized};
\item $\mathrm{Var}_{W_{\mathrm{rsp}}}[D_{\mathrm{rsp}}] \le m/4$;
\item For $d = 1$ (single-bit perturbation), $\mathbb{E}_{W_{\mathrm{rsp}}}[D_{\mathrm{rsp}}] = m \cdot [1 - (1 - 1/L)^k]/2 \approx mk/(2L)$ for $k \ll L$, which is far below $m/2$ when $k \ll L$.
\end{itemize}
By contrast, for $\Psi$:
\begin{itemize}
\item For each single-bit perturbation $e_j$, $\mathrm{HW}(\Psi(e_j))$ takes a specific value with no variance over the construction of $\Psi$;
\item The mean of $\mathrm{HW}(\Psi(e_j))$ over $j = 0, \dots, L-1$ approaches $m/2$ as $L \to 2^m - 1$;
\item No specific $j$ can yield $\mathrm{HW}(\Psi(e_j)) \le k$ uniformly in $j$, because $\Psi$ traces a Hamilton path through $\mathbb{F}^*$.
\end{itemize}
Hence the dominant difference between deterministic Galois-field diffusion and random sparse projection is in the worst-case behavior at single-bit (or otherwise minimum-weight) perturbations: random sparse projection with row-sparsity $k$ has output Hamming distance bounded above by $k$ for single-bit input flips, and is therefore far from QOD for $k \ll m/2$. Galois-field diffusion has no such failure mode, regardless of the structure of $G(x)$ (provided it is primitive).

The second-moment statistics over the input ensemble are the same for both constructions, as shown in Lemma~\ref{lem:HW-dist}. The advantages of the deterministic construction are therefore: (i)~reproducibility (no construction variance); (ii)~worst-case avalanche at minimum-weight perturbations; and (iii)~auditability via the explicit cyclic structure traced in Lemma~\ref{lem:avalanche}. These three together justify the use of $\Psi$ as a substitute for $W_{gh}$ in Vector-HaSH (Corollary~\ref{cor:equivalence}).

\section{Relationship to Vector-HaSH: Regimes of Error Correction and Branching Inference}\label{app:vector-hash}

This appendix carries the narrative commitment of Section~\ref{sec:synthesis} into purely computational bookkeeping; it dovetails Section~\ref{sec:two-orthogonalizers}'s evolutionary motive (silicon $\mathrm{RR}\in\{1,0\}$ regimes as analogues for mixed Regime~A/B demands) without equating RR with synaptic morphology. It complements Sections~\ref{sec:correspondence-I}--\ref{sec:correspondence-II}: Corollary~\ref{cor:equivalence} licenses substituting Galois-field diffusion for Vector-HaSH's random scaffold map whenever only QOD statistics matter, yet two further layers remain logically independent. First, Vector-HaSH foregrounds recurrent entorhinal--hippocampal dynamics that retract noisy states toward scaffold fixed points (Section~\ref{sec:bg-vectorhash}); VaCoAl's silicon majority vote is an analogy for CA3 completion, not a literal implementation of continuous-time attractors. Second, when $\mathrm{CR1} < 1$ on branched workloads, VaCoAl's CR2 supplies an explicit graded ranking over multi-hop continuations—a bookkeeping device that aligns with branching replay statistics (Section~\ref{sec:correspondence-II}) but does not coincide, without further assumptions, with Vector-HaSH's externally steered sequential advance~\cite{chandra2025episodic}. The short analytical answer once Corollary~\ref{cor:equivalence} is admitted is therefore: \emph{equivalence on scaffold diffusion; non-equivalence a priori on (i)~attractor-complete contraction versus abstaining reads, and (ii)~branch-aware path scoring versus velocity-supervised hopping}. Expanded engineering benchmarks belong in a separate companion manuscript; here we tighten vocabulary only.

\subsection{Vector-HaSH in Brief}

Chandra et al.~\cite{chandra2025episodic} posit that hippocampal CA3 supports a grid-cell-driven scaffold: entorhinal states project into CA3 via a fixed random matrix $W_{gh}$ (Eq.~\eqref{eq:vechash} in Section~\ref{sec:bg-vectorhash}), yielding a vast family of stable attractors that act as an address space for episodic memory. New content binds to scaffold locations through plastic outward weights; sequential recall is organized by integrating low-dimensional velocity signals that advance the scaffold, analogously to path integration.

Two features of this normative picture matter for appendix-level alignment.

\paragraph{Scaffold/content factorization.} Interference at overload primarily degrades associative reconstruction of particulars while sparing the scaffold-side index bookkeeping that Vector-HaSH associates with graceful degradation rather than a Hopfield cliff~\cite{chandra2025episodic,hopfield1982neural}.

\paragraph{Retrieval-phase correction.} The attractor cleanup that restores scaffold indices naturally reads as repeated exchange on an entorhinal--hippocampal bidirectional scaffold channel; purely feedforward edge-processing without that closed loop lacks the iterated contraction motif (cf.\ Section~\ref{sec:synthesis} and Section~\ref{sec:bg-vectorhash}). The published minimal circuit emphasizes learned orthogonalisation without modeling an explicit DG stage; anatomical completions belong to Sections~\ref{sec:two-orthogonalizers}, \ref{sec:mapping-AB}, and Appendix~\ref{sec:anatomical-implication} below, not to block-level identification with VaCoAl.

\subsection{Structural Correspondence with PyVaCoAl}

Four elements invite an element-for-element comparison (cf.\ Corollary~\ref{cor:equivalence}):
\begin{itemize}
\item \textbf{Fixed mixing substrate.} Vector-HaSH's fixed $W_{gh}$ pairs with VaCoAl's fixed primitive $G(x)$ and deterministic map $\Psi$ in Eq.~\eqref{eq:diffusion-2}. Both supply QOD-class address variability without fresh randomness at query time.
\item \textbf{Attractor-complete cleanup $\leftrightarrow$ Rescue mode.} Vector-HaSH-style complete error correction corresponds to operating PyVaCoAl in Rescue mode, where residual collisions are repaired (exact-match / auxiliary-table lookup) so that, on benchmark traces, $\mathrm{CR1} = \mathrm{CR2} = 1$ and downstream ranks are degenerate unless an external ordering is imposed~\cite{chuma2026vacoal}.
\item \textbf{Plastic hetero-associations $\leftrightarrow$ entry-address storage.} Vector-HaSH's learned hippocampus-to-content weights parallel VaCoAl's writes of Entry Addresses into block-indexed stores at addresses computed by $\Psi$.
\item \textbf{Branch scoring $\leftrightarrow$ Don't Care CR2 trajectories (non-overlap).} When competing continuations coexist and no external supervisor supplies a temporal ordering analogous to trajectory integration~\cite{chandra2025episodic}, VaCoAl's $\mathrm{CR1} < 1$ regime induces multiplicative CR2 separation among paths (Sections~\ref{sec:bg-vacoal}, \ref{sec:correspondence-II}); this graded ranking is orthogonal to claiming identity with Vector-HaSH's attractor map and is deliberately highlighted in silicon because iEEG replay decay is naturally multiplicative (Proposition~\ref{prop:cr2}).
\end{itemize}

The correspondence is structural on the scaffold side, not a claim that silicon blocks ``are'' CA3 pyramids. The bioRxiv narrative leans on the Don't Care configuration ($\mathrm{CR1} < 1$ along nontrivial hops), which Vector-HaSH does not emphasize, precisely because CR2 contraction is where algebraic benchmarks and ripple replay statistics meet (Section~\ref{sec:correspondence-II}, Proposition~\ref{prop:cr2}) without implying that biological EC$\leftrightarrow$CA3 implements Galois XOR.

\subsection{Rescue / CR1 and Effective Branching}

Write $\mathrm{RR}$ for PyVaCoAl's rescue/collision-resolution setting with the convention $\mathrm{RR}=1$ for Rescue (forced repair to collision-free reads on the genealogy benchmark) and $\mathrm{RR}=0$ for Don't Care abstention. Under Rescue, every surviving hop satisfies $\mathrm{CR1} = 1$ and hence $\mathrm{CR2}(n) = 1$ in Proposition~\ref{prop:cr2}'s terminology. Within a bounded Frontier of competing continuations, rankings that depend on CR2 alone therefore collapse; tie-breaks revert to implementation details (e.g., lexicographic order on entry tags) unless an external supervisor supplies direction—Vector-HaSH's velocity controller is one such supervisor~\cite{chandra2025episodic}. On a directed graph with true branching $b > 1$, absent such external structure the effective object being traversed can be a $b$-regular outward tree in hardware yet a single-threaded walk in the ranking semantics of the memory layer. We summarize this bookkeeping as
\begin{equation}\label{eq:RR1}
\mathrm{RR} = 1 \iff \mathrm{CR1}(n) \equiv 1 \implies b_{\text{effective}} = 1
\end{equation}
in the sense that intrinsic sequential scoring from CR2 vanishes. Equation~\eqref{eq:RR1} is not a critique of Vector-HaSH for spatial sequences where $b \approx 1$ is appropriate; it marks where VaCoAl-style DAG benchmarks and attractor-complete regimes diverge.

\subsection{CR2 as Contraction When CR1 $<$ 1}

In Don't Care mode at sufficient depth, observed CR1 stabilizes slightly below unity~\cite{chuma2026vacoal}. With $\mathrm{CR2}(n) = \mathrm{CR2}(n-1) \cdot \mathrm{CR1}(n-1)$ and $\mathrm{CR1} < 1$, the map $x \mapsto \mathrm{CR1} \cdot x$ on $[0,1]$ is a strict contraction; iterating it drives CR2 toward the unique fixed point $0$ (Banach theorem). Short, high-confidence lineages therefore dominate long, leaky ones—the ``Occam'' bias in path ranking. Vector-HaSH instead projects states onto attractors hop-by-hop; the contraction story here is tied to the absence of hop-wise projection that pins CR1 to 1. Main-text claims about STDP-like pruning should be read against this algebraic backdrop.

\subsection{Pointer: CR2 Trajectories Near Classical Capacity Ratios}

Fix total SRAM/DRAM budget $N \cdot 2^m$ and sweep $m$ on the Wikidata mentor--student DAG benchmark~\cite{chuma2026vacoal}. Empirically, shallow $m$ raises collision pressure, lowers CR1, and can drive terminal CR2 toward neighborhoods of the classical Hopfield capacity ratio $\alpha_c \approx 0.138$ where mean-field retrieval breaks down~\cite{hopfield1982neural}. We state this only as motivation: a full depth sweep with significance tests belongs in an engineering companion, not in the present hippocampal-focused argument.

\subsection{Anatomical Implication (Hypothesis-Level)}\label{sec:anatomical-implication}

Vector-HaSH's published minimal circuit omits an explicit DG stage. Read together with Section~\ref{sec:synthesis}: (i)~Regime~A (EC$\leftrightarrow$CA3), as developed in Sections~\ref{sec:two-orthogonalizers}, \ref{sec:mapping-AB}, is the natural anatomical correlate of scaffold attractor dynamics that enact index-like cleanup through recurrent exchange~\cite{chandra2025episodic,rolls2023brain}; (ii)~Regime~B (EC$\to$DG$\to$mossy-fiber$\to$CA3) behaves as a novelty-gated, largely feedforward pattern-separation shaft that mints asymmetric CA3 ingress patterns without, by itself, substituting for the closed-loop relaxation used in the scaffold story.

VaCoAl's engineering stack parallels this split loosely: Galois-field diffusion (plus block addressing) parallels the quasi-orthogonal map embodied in Vector-HaSH's $W_{gh}$ equivalence class (Corollary~\ref{cor:equivalence}); majority voting parallels CA3 recurrence as a coherence filter; auxiliary Rescue/Don't Care bookkeeping splits exact cleanup from abstaining reads. DG is hypothesized—not asserted—to tighten the analogous ``second layer'' biology provides before scaffold recurrence settles an index~\cite{marr1971simple,treves1994computational}: whether mammalian DG is necessary for relational inference where symbolic branching $b > 1$ dominates behavior remains an empirical target beyond this appendix.

We note only the logical choreography suggested by juxtaposition: preprocessing-style orthogonalisation (Regime~B; Don't Care / $\mathrm{CR1} < 1$ statistical regimes here) biases which CA3 motifs become eligible before bidirectional scaffold contraction (Regime~A; Rescue-complete / attractor-complete side) declares a discrete scaffold label. Galois substitution does not remove the recurrent channel from that choreography; Corollary~\ref{cor:equivalence} addresses the projection statistic only.

\subsection{Summary}

Vector-HaSH and VaCoAl converge on the sufficient QOD scaffold (Corollary~\ref{cor:equivalence}) yet remain logically distinct wherever Section~\ref{sec:synthesis} invokes dynamical index cleanup versus graded path scoring. Rescue-complete traces align with attractor-complete error correction; Don't Care traces supply the multiplicative CR2 effects emphasized in Sections~\ref{sec:correspondence-II}--\ref{sec:predictions}. Neither mode is universally canonical: hippocampal sequences with effective $b \approx 1$ favor near-complete scaffold repair, whereas large-branching symbolic workloads expose CR2 contractions unless an external supervisor (e.g., velocity steering in Vector-HaSH~\cite{chandra2025episodic}) breaks ties. Section~\ref{sec:anatomical-implication} summarizes the corresponding Regime~A/B hypothesis compactly; expanded proofs and benchmarks remain for a companion engineering manuscript.

\end{document}